\documentclass{article} %
\usepackage{iclr2022_conference,times}

\usepackage{times}
\usepackage{epsfig}
\usepackage{graphicx}
\usepackage{amsmath}
\usepackage{amssymb}
\usepackage{bm}
\usepackage{amsthm}
\usepackage{algorithm}
\usepackage{algorithmicx}
\usepackage{algpseudocode}
\usepackage[dvipsnames]{xcolor}
\usepackage{caption}
\usepackage{subcaption}
\usepackage[export]{adjustbox}
\usepackage{float}
\usepackage{stfloats}
\usepackage{mathtools}
\usepackage{tabularx}
\usepackage{booktabs}
\usepackage{float}
\usepackage{placeins}
\usepackage{thm-restate}
\usepackage{thmtools}
\usepackage{wrapfig}

\usepackage[breaklinks=true,bookmarks=false]{hyperref}
\definecolor{citecolor}{HTML}{21618C}%
\definecolor{urlcolor}{HTML}{1D8348}
\definecolor{linkcolor}{HTML}{1D8348}
\hypersetup{
    colorlinks=true,
    linkcolor=linkcolor,
    filecolor=magenta,      
    urlcolor=urlcolor,
    citecolor=citecolor,
    }
    
\usepackage[capitalise]{cleveref}
\usepackage{multirow}

\usepackage{etoolbox}
\newcommand{\algorithmicdoinparallel}{\textbf{do in parallel}}
\makeatletter
\AtBeginEnvironment{algorithmic}{%
  \newcommand{\FORP}[2][default]{\ALC@it\algorithmicfor\ #2\ %
    \algorithmicdoinparallel\ALC@com{#1}\begin{ALC@for}}
}
\makeatother

\newcommand{\myparagraph}[1]{\vspace{0pt}\paragraph{#1}}

\newcommand{\mbb}[1]{\mathbb{#1}}
\newcommand{\ud}{\mathrm{d}}

\newcommand{\mcal}{\mathcal}
\newcommand{\norm}[1]{\left\lVert#1\right\rVert}

\usepackage{mdframed}

\newcommand{\be}{\begin{equation}}
\newcommand{\ee}{\end{equation}}
\definecolor{Gray}{gray}{0.85}
\definecolor{LightCyan}{rgb}{0.88,1,1}

\makeatletter
\usepackage{xspace} 
\def\@onedot{\ifx\@let@token.\else.\null\fi\xspace}
\DeclareRobustCommand\onedot{\futurelet\@let@token\@onedot}

\definecolor{blue1}{RGB}{0,128,255}
\definecolor{blue3}{RGB}{0,0,128}
\definecolor{darkpastelgreen}{rgb}{0.01, 0.75, 0.24}
\definecolor{cerulean}{rgb}{0.0, 0.48, 0.65}

\def\eg{\emph{e.g}\onedot}

\def\ie{\emph{i.e}\onedot}

\def\vs{\emph{vs}\onedot}

\def\eqref#1{equation~\ref{#1}}
\def\Eqref#1{Equation~\ref{#1}}

\def\1{\bm{1}}

\def\rvw{{\mathbf{w}}}
\def\rvx{{\mathbf{x}}}
\def\rvy{{\mathbf{y}}}
\def\rvz{{\mathbf{z}}}

\def\vtheta{{\bm{\theta}}}

\def\vs{{\bm{s}}}

\DeclareMathAlphabet{\mathsfit}{\encodingdefault}{\sfdefault}{m}{sl}
\SetMathAlphabet{\mathsfit}{bold}{\encodingdefault}{\sfdefault}{bx}{n}

\def\gN{{\mathcal{N}}}

\def\gX{{\mathcal{X}}}

\newcommand{\x}{\ensuremath{\mathbf{x}}\xspace}

\newcommand{\z}{\ensuremath{\mathbf{z}}\xspace}

\newcommand{\modelfull}{Stochastic Differential Editing\xspace}
\newcommand{\model}{SDEdit\xspace}
\iclrfinalcopy
\begin{document}

\title{SDEdit: Guided Image Synthesis and Editing with Stochastic Differential Equations}

\author{Chenlin Meng$^1$ \qquad Yutong He$^1$  \qquad Yang Song$^1$ \qquad Jiaming Song$^1$ \\ \textbf{Jiajun Wu$^1$ \qquad  Jun-Yan Zhu$^2$ \qquad Stefano Ermon$^1$} \\
$^1$Stanford University \qquad $^2$Carnegie Mellon University}

\maketitle

\begin{abstract}
Guided image synthesis 
enables everyday users to create and edit photo-realistic images with minimum effort. 
The key challenge is balancing \emph{faithfulness} to the user inputs (\eg, hand-drawn colored strokes) and \emph{realism} of the synthesized images. 
Existing GAN-based methods attempt to achieve such balance using either conditional GANs %
or
GAN inversions,  %
which are challenging and often require additional training data or loss functions for individual applications. To address these issues, 
we introduce a new image synthesis and editing method, Stochastic Differential Editing (SDEdit), based on a 
diffusion model generative prior, which synthesizes realistic images by iteratively denoising through a stochastic differential equation (SDE). 
Given an input image with user guide in a form of manipulating RGB pixels,
SDEdit first adds noise to the input,
then subsequently denoises the resulting image through the SDE prior to increase its realism. 
 SDEdit does not require task-specific training or inversions and can naturally achieve the balance between realism and faithfulness. SDEdit outperforms state-of-the-art GAN-based methods by up to  $98.09\%$ on realism and 
 $91.72\%$ on overall satisfaction
scores, according to a human perception study, on multiple tasks, including stroke-based image synthesis and editing as well as image compositing.

\end{abstract}
\begin{figure}[H]
    \centering
    \includegraphics[width=\linewidth]{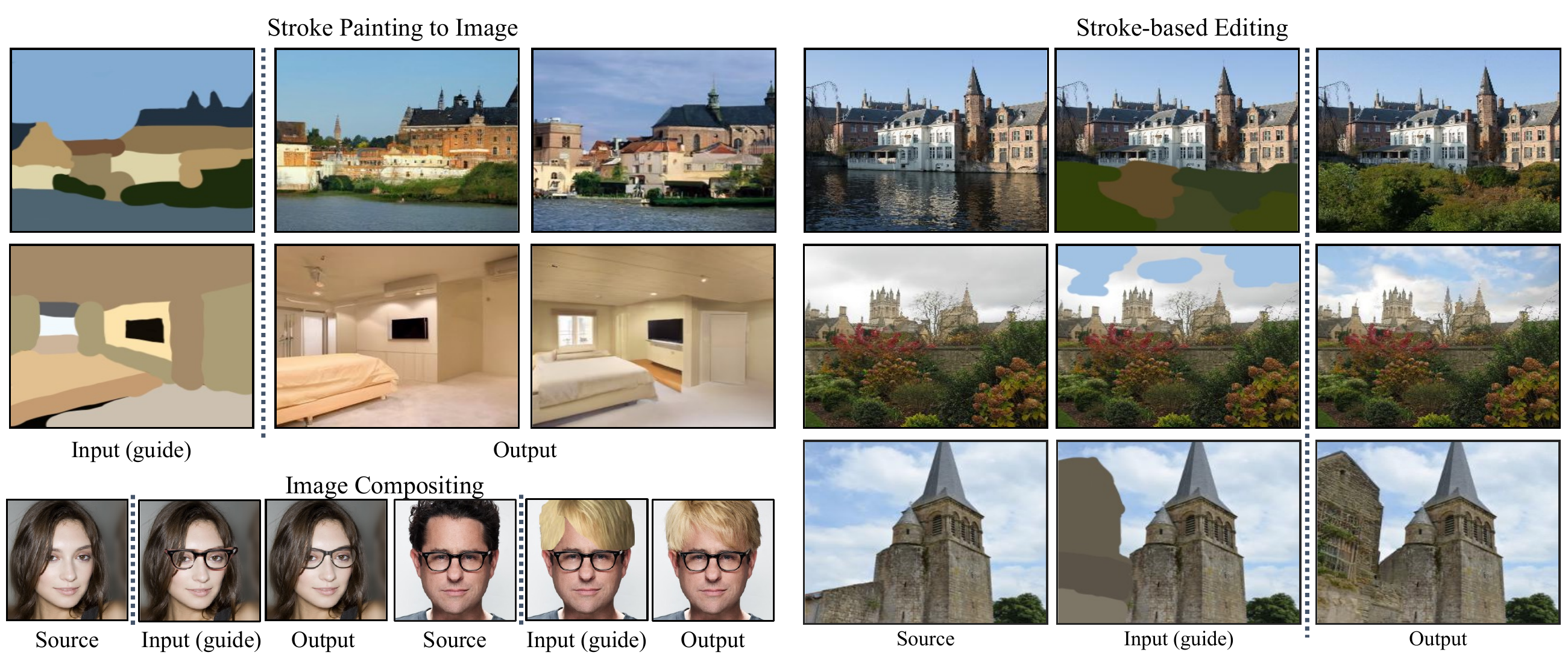}
    \captionof{figure}{\modelfull (\model) is a {\textbf{unified}} image synthesis and editing framework based on stochastic differential equations. \model allows stroke painting to image, image compositing, and stroke-based editing
    {\textbf{without}} task-specific model training and loss functions.
    }
    \label{fig:teaser}
\end{figure}
\section{Introduction}

Modern generative models can create photo-realistic images from random noise~\citep{karras2019style,song2021scorebased}, serving as an important tool for visual content creation. Of particular interest is guided image synthesis and editing, where a user specifies a general guide (such as coarse colored strokes) and the generative model learns to fill in the details (see \cref{fig:teaser}). 
There are two natural desiderata for guided image synthesis: the synthesized image should appear \textit{realistic} as well as be \textit{faithful} to the user-guided input, thus enabling people with or without artistic expertise to produce photo-realistic images from different levels of details.

Existing methods often attempt to achieve such balance via two approaches.
The first category leverages conditional GANs~\citep{isola2017image,zhu2017unpaired}, which learn a direct mapping from original images to edited ones.
Unfortunately, for each new editing task, these methods require data collection and model re-training, both of which could be expensive and time-consuming.  
The second category leverages GAN inversions~\citep{zhu2016generative,brock2017neural,abdal2019image2stylegan,gu2020image,wu2021stylespace,abdal2020image2stylegan++}, where a pre-trained GAN is used to invert an input image to a latent representation, which is subsequently modified
to generate the edited image. This procedure involves manually designing loss functions and optimization procedures for different image editing tasks. Besides, it may sometimes fail to find a latent code that faithfully represents the input~\citep{bau2019seeing}.

To balance \emph{realism} and \emph{faithfulness} while avoiding the previously mentioned challenges,
we introduce SDEdit, a guided image synthesis and editing framework leveraging generative stochastic differential equations \citep[SDEs;][]{song2021scorebased}. Similar to the closely related  diffusion models~\citep{sohl2015deep,ho2020denoising}, SDE-based generative models smoothly convert an initial Gaussian noise vector to a realistic image sample through iterative denoising, and have achieved unconditional image synthesis performance comparable to or better than that of GANs~\citep{dhariwal2021diffusion}. 
The key intuition of SDEdit is to ``hijack'' the generative process of SDE-based generative models, as illustrated in \cref{fig:sde_stroke}. Given an input image with user 
guidance input,
such as a stroke painting or an image with stroke edits, we can add a suitable amount of noise to smooth out undesirable artifacts and distortions (\eg, unnatural details at stroke pixels), while still preserving the overall structure of the input user guide. We then initialize the SDE with this noisy input, and progressively remove the noise to obtain a denoised result that is both realistic and faithful to the user guidance input (see \cref{fig:sde_stroke}).

Unlike conditional GANs, SDEdit does not require collecting training images or user annotations for each new task; %
unlike GAN inversions, SDEdit does not require the design of additional training or task-specific loss functions. 
SDEdit only uses  a single pretrained SDE-based generative model trained on unlabeled data: given a user guide in a form of manipulating RGB pixels,
SDEdit adds Gaussian noise to the guide and then run the reverse SDE to synthesize images.
SDEdit naturally finds a trade-off between realism and faithfulness: when we add more Gaussian noise and run the SDE for longer, the synthesized images are more realistic but less faithful. We can use this observation to find the right balance between realism and faithfulness.

We demonstrate SDEdit on three applications: stroke-based image synthesis, stroke-based image editing, and image compositing. We show that SDEdit can produce \emph{realistic} and \emph{faithful} images from guides with various levels of fidelity. On stroke-based image synthesis experiments, \model outperforms
state-of-the-art GAN-based approaches by up to $98.09\%$ on realism score and $91.72\%$ on overall satisfaction score (measuring both realism and faithfulness) according to human judgements. On image compositing experiments, \model achieves a better faithfulness score and outperforms the baselines by up to $83.73\%$ on overall satisfaction score in user studies. Our code and models will be available upon publication.

\section{Background: Image Synthesis with Stochastic Differential Equations (SDEs)}
Stochastic differential equations (SDEs) generalize ordinary differential equations (ODEs) by injecting random noise into the dynamics.
The solution of an SDE is a time-varying random variable (\ie, stochastic process), which we denote as $\rvx(t) \in \mathbb{R}^d$, where $t \in [0, 1]$ indexes time. 
In image synthesis~\citep{song2021scorebased}, we suppose that $\rvx(0) \sim p_{0}=p_{\text{data}}$ represents a sample from the data distribution and that a forward SDE produces $\rvx(t)$ for $t \in (0, 1]$ via a Gaussian diffusion. 
Given $\rvx(0)$, $\rvx(t)$ is distributed as a Gaussian distribution:
\begin{align}
     \rvx(t) = \alpha(t) \rvx(0) + \sigma(t) \rvz, \quad \rvz \sim \gN(\bm{0}, \bm{I}), \label{eq:rvxt}
\end{align}
where $\sigma(t): [0, 1] \to [0, \infty)$ is a scalar
function that describes the magnitude of the noise $\rvz$, and $\alpha(t): [0, 1] \to [0, 1]$ is a scalar
function that denotes the magnitude of the data $\rvx(0)$. 
The probability density function of $\rvx(t)$ is denoted as $p_t$.

Two types of SDE are usually considered: the Variance Exploding SDE (VE-SDE) has $\alpha(t) = 1$ for all $t$ and $\sigma(1)$ being a large constant so that $p_{1}$ is close to $\gN(\bf{0}, \sigma^2(1)\bf{I})$;
whereas the Variance Preserving (VP) SDE satisfies $\alpha^2(t) + \sigma^2(t) = 1$ for all $t$ with $\alpha(t)\to0$ as $t\to 1$ so that $p_{1}$ equals to $\gN(\bf{0}, \bf{I})$.
Both VE and VP SDE transform the data distribution to random Gaussian noise as $t$ goes from $0$ to $1$.
For brevity, we discuss the details based on VE-SDE for the remainder of the main text, and discuss the VP-SDE procedure in \cref{app:sdedit}. Though possessing slightly different forms and performing differently depending on the image domain, they share the same mathematical intuition.

\paragraph{Image synthesis with VE-SDE.}
Under these definitions, we can pose the image synthesis problem as gradually removing noise from a noisy observation $\rvx(t)$ to recover $\rvx(0)$. This can be performed via a reverse SDE~\citep{anderson1982reverse,song2021scorebased} that travels from $t = 1$ to $t = 0$, based on the knowledge about the noise-perturbed score function $\nabla_\rvx \log p_t(\rvx)$. For example, the sampling procedure for VE-SDE is defined by the following (reverse) SDE:
\begin{align}
    \mathrm{d} \rvx(t) = \left[ - \frac{\mathrm{d} [\sigma^2(t)]}{\mathrm{d} t} \nabla_\rvx \log p_t(\rvx)\right] \mathrm{d}t + \sqrt{\frac{\mathrm{d} [\sigma^2(t)]}{\mathrm{d} t}} \mathrm{d} \bar{\mathbf{w}}, \label{eq:reverse-sde-ve}
\end{align}
where $\bar{\mathbf{w}}$ is a Wiener process when time flows backwards from $t = 1$ to $t = 0$. 
If we set the initial conditions $\rvx(1) \sim p_1 = \gN(\bf{0}, \sigma^2(1)\bf{I})$,
then the solution to $\rvx(0)$ will be distributed as $p_{\mathrm{data}}$.
In practice, the noise-perturbed score function can be learned through denoising score matching~\citep{vincent2011connection}. Denote the learned score model as $\vs_\vtheta(\rvx(t), t)$, the learning objective for time $t$ is:
\begin{align}
    L_t = \mathbb{E}_{\rvx(0) \sim p_{\mathrm{data}}, \rvz \sim \gN(\bf{0}, \bf{I})}[\lVert \sigma_t \vs_\vtheta(\rvx(t), t) - \rvz \rVert_2^2],
\end{align}
where $p_{\mathrm{data}}$ is the data distribution and $\rvx(t)$ is defined as in \Eqref{eq:rvxt}. The overall training objective is a weighted sum over $t$ %
of each individual learning objective $L_t$, and various weighting procedures have been discussed in \cite{ho2020denoising,song2020denoising,song2021scorebased}. 

With a parametrized score model $\vs_\vtheta(\rvx(t), t)$ to approximate %
$\nabla_\rvx \log p_t(\rvx)$, the SDE solution can be approximated with the Euler-Maruyama method; an update rule from $(t + \Delta t)$ to $t$ is
\begin{align}
    \rvx(t) = \rvx(t + \Delta t) + (\sigma^2(t) - \sigma^2(t + \Delta t)) \vs_\vtheta(\rvx(t), t) + \sqrt{\sigma^2(t) - \sigma^2(t + \Delta t)} \rvz. \label{eq:update-1-step}
\end{align}
where $\rvz \sim \gN(\bf{0}, \bf{I})$. 
We can select a particular discretization of the time interval from $1$ to $0$, initialize $\rvx(0) \sim \gN(\bf{0}, \sigma^2(1)\bf{I})$ and iterate via \Eqref{eq:update-1-step} to produce an image $\rvx(0)$.

\begin{figure}
\vspace{-20pt}
\centering
\includegraphics[width=\linewidth]{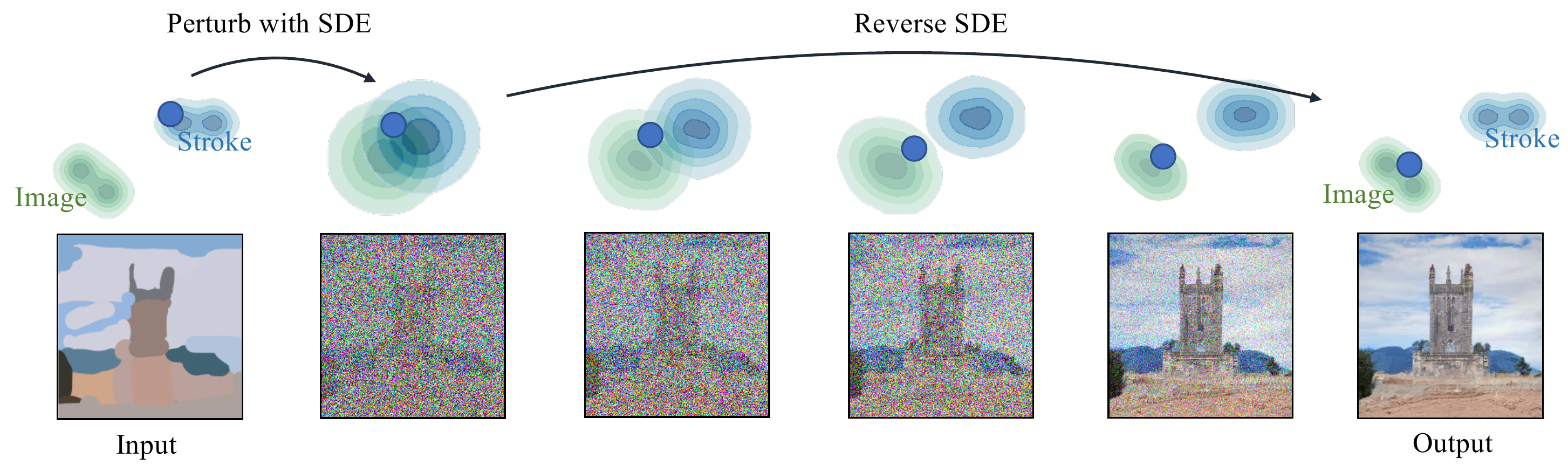}
\caption{Synthesizing images from strokes with \model. The blue dots illustrate the editing process of our method. The green and blue contour plots represent the distributions of images and stroke paintings, respectively. Given a stroke painting, we first perturb it with Gaussian noise 
and progressively remove the noise by simulating the reverse SDE. This process gradually projects an unrealistic stroke painting to the manifold of natural images.
}
\vspace{-15pt}
\label{fig:sde_stroke}
\end{figure}

\section{Guided Image Synthesis and Editing with \model}
\label{sec:method}
In this section, we introduce \model and describe how we can perform guided image synthesis and editing through an SDE model pretrained on unlabeled images. 

\paragraph{Setup.} The user provides a full resolution image $\rvx^{(g)}$ in a form of manipulating RGB pixels, which we call a ``\textit{guide}''. 
The guide may contain different levels of details; a high-level guide contains only coarse colored strokes, a mid-level guide contains colored strokes on a real image, and a low-level guide contains image patches on a target image.
We illustrate these guides in \cref{fig:teaser},
which can be easily provided by non-experts. Our goal is to produce full resolution images with two desiderata:
\begin{description}
    \item[Realism.] The image should appear realistic (\textit{e.g.}, measured by humans or neural networks).
    \item[Faithfulness.] The image should be similar to the guide $\rvx^{(g)}$ (\textit{e.g.}, measured by $L_2$ distance).
\end{description}
We note that realism and faithfulness are not positively correlated, since there can be realistic images that are not faithful (\textit{e.g.}, a random realistic image) and faithful images that are not realistic (\textit{e.g.}, the guide itself). 
Unlike regular inverse problems, we do not assume knowledge about the measurement function (\textit{i.e.}, the mapping from real images to user-created guides in RBG pixels is unknown), so techniques for solving inverse problems with score-based models~\citep{dhariwal2021diffusion,kawar2021snips}
and methods requiring paired datasets~\citep{isola2017image,zhu2017unpaired}
do not apply here. %

 \begin{figure}
 \vspace{-20pt}
     \begin{subfigure}[h]{0.28\linewidth}
         \centering
    \includegraphics[width=\linewidth]{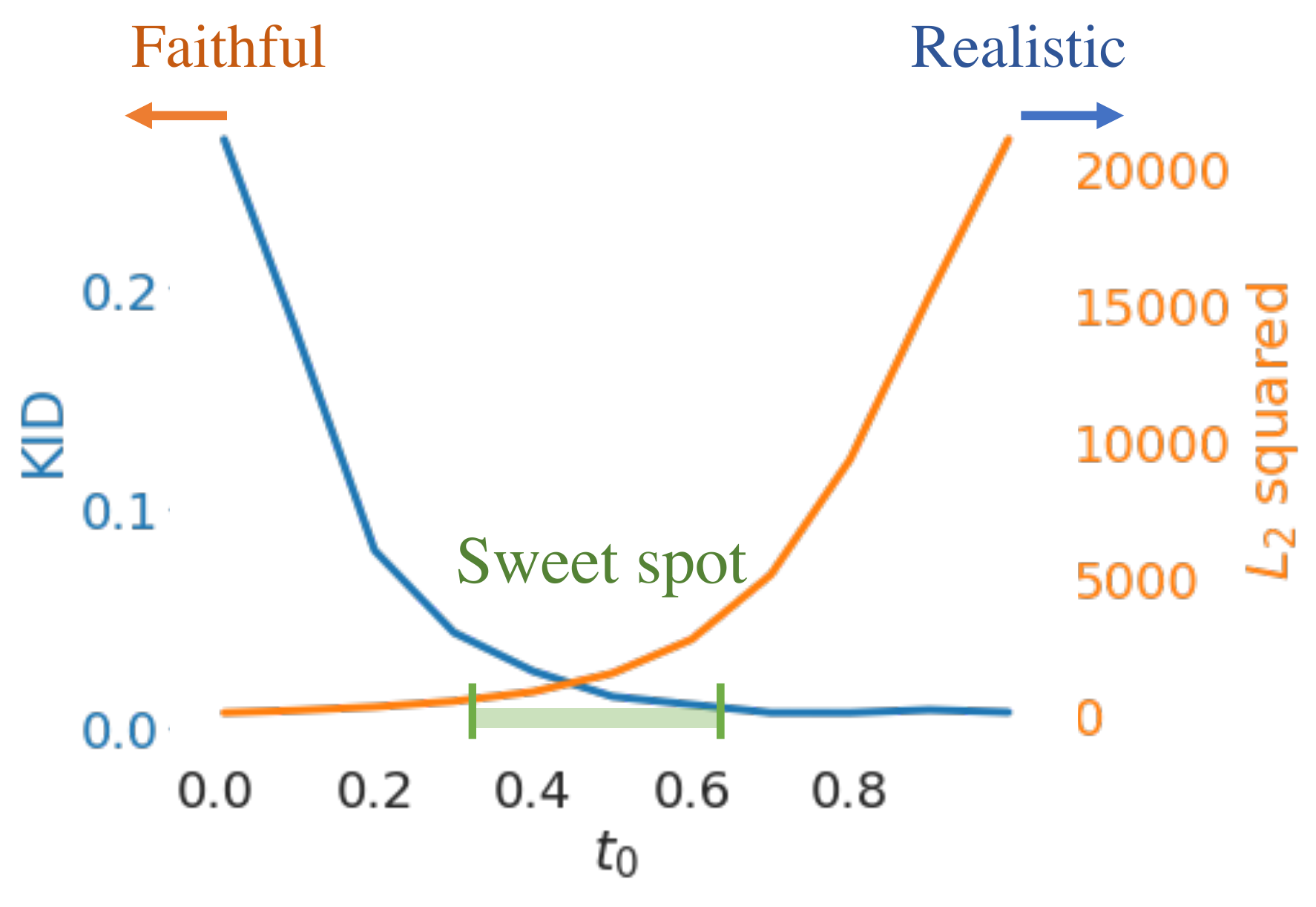}
\caption{KID and $L_2$ norm \textbf{squared} plot with respect to $t_0$.
}
\label{fig:kid_l2}
     \end{subfigure}
      \hfill
 \begin{subfigure}[h]{0.7\linewidth}
     \centering
     \includegraphics[width=\linewidth]{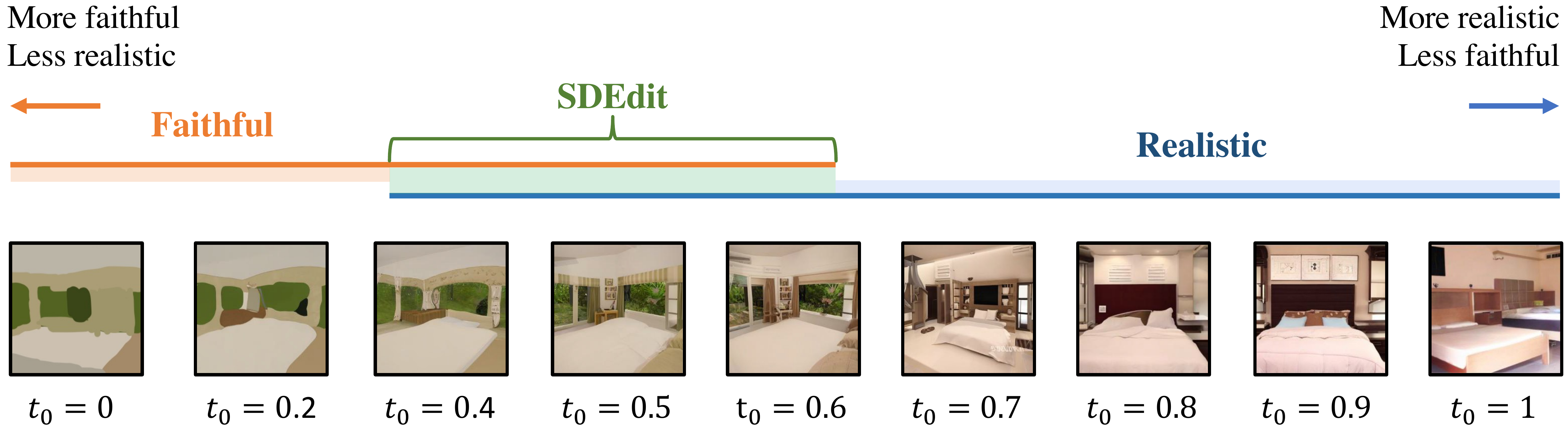}
\caption{
We illustrate synthesized images of SDEdit with various $t_0$ initializations. $t_0 = 0$ indicates the guide itself, whereas $t_0 = 1$ indicates a random sample.
}
     \end{subfigure}
     
\caption{Trade-off between faithfulness and realism for stroke-based generation on LSUN. As $t_0$ increases, the generated images become {\textbf{more realistic}}
while {\textbf{less faithful}}. Given an input, \model aims at generating an image that is both faithful and realistic, which means that we should choose $t_0$ appropriately ($t_0 \in [0.3, 0.6]$ in this example). 
}
\label{fig:sde_tradeoff}
\end{figure}

\paragraph{Procedure.} 
Our method, SDEdit, uses the fact that the reverse SDE can be solved not only from $t_0 = 1$, but also from any intermediate time $t_0 \in (0, 1)$ -- an approach not studied by previous SDE-based generative models.
We need to find a proper initialization from our guides from which we can solve the reverse SDE to obtain desirable, realistic, and faithful images.
For any given guide $\rvx^{(g)}$, we define the SDEdit procedure as follows:
\begin{quote}
    Sample $\rvx^{(g)}(t_0) \sim \gN(\rvx^{(g)}; \sigma^2(t_0) \mathbf{I})$, then produce $\rvx(0)$ by iterating \Eqref{eq:update-1-step}.
\end{quote}
We use $\mathrm{SDEdit}(\rvx^{(g)}; t_0, \theta)$ to denote the above procedure.
Essentially, SDEdit selects a particular time $t_0$, add Gaussian noise of standard deviation $\sigma^2(t_0)$ to the guide $\rvx^{(g)}$ and then solves the corresponding reverse SDE at $t = 0$ to produce the synthesized $\rvx(0)$. 

Apart from the discretization steps taken by the SDE solver, %
the key hyperparameter for SDEdit is $t_0$, the time from which we begin the image synthesis procedure in the reverse SDE. In the following, we describe a realism-faithfulness trade-off that allows us to select reasonable values of $t_0$.

 \begin{figure}
  \vspace{-20pt}
     \centering
     \includegraphics[width=0.85\linewidth]{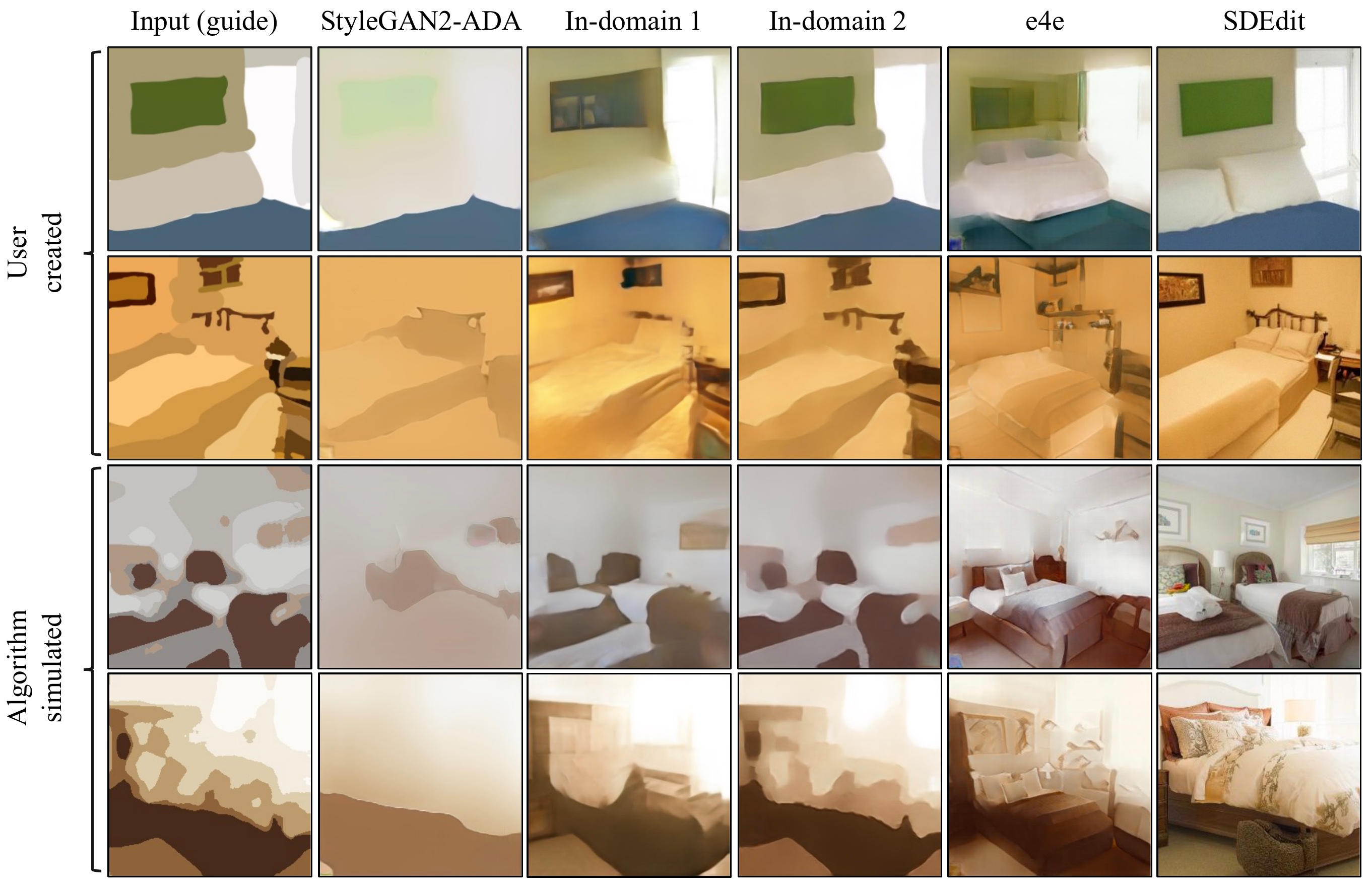}
     \caption{ 
     {{\model generates more realistic and faithful images}} than state-of-the-art GAN-based models on stroke-based generation (LSUN bedroom).
     The guide in the first two rows are created by human and the ones in the last two rows are simulated by algorithm.
}
\label{fig:sde_stroke_generation_baseline}
\vspace{-10pt}
\end{figure}

\paragraph{Realism-faithfulness trade-off.} We note that for properly trained SDE models, there is a realism-faithfulness trade-off when choosing different values of $t_0$. 
To illustrate this, we focus on the LSUN dataset,
and use 
high-level stroke paintings 
as guides to perform stroke-based image generation.
We provide experimental details in \cref{app:sec:generating_stroke}.
We consider different choices of $t_0\in[0,1]$ for the same input. To quantify realism, we adopt neural methods for comparing image distributions, such as the Kernel Inception Score \citep[KID;][]{binkowski2018demystifying}.
If the KID between synthesized images
and real images are low, then the synthesized images are realistic. %
For faithfulness, we measure the squared $L_2$ distance between the synthesized images and the guides $\x^{(g)}$. 
From \cref{fig:sde_tradeoff},
we observe increased realism but decreased faithfulness as $t_0$ increases.

The realism-faithfulness trade-off can be interpreted from another angle. If the guide is far from any realistic images, then we must tolerate at least a certain level of deviation from the guide (non-faithfulness) in order to produce a realistic image. 
This is illustrated in the following proposition.
\begin{restatable}{proposition}{propbound}
Assume that $\norm{s_\theta(\rvx, t)}^2_2 \leq C$ for all $\rvx \in \gX$ and $t \in [0, 1]$. Then for all $\delta \in (0, 1)$ with probability at least $(1 - \delta)$,
\begin{align}
    \norm{\rvx^{(g)} - \mathrm{SDEdit}(\rvx^{(g)}; t_0, \theta)}^2_2 \leq \sigma^2(t_0) (C \sigma^2(t_0) + d + 2\sqrt{- d \cdot \log \delta} - 2 \log \delta)
\end{align}
where $d$ is the number of dimensions of $\rvx^{(g)}$.
\end{restatable}

We provide the proof in \cref{app:sec:proofs}. On a high-level, the difference from the guides and the synthesized images can be decomposed into the outputs of the score and random Gaussian noise; both would increase as $t_0$ increases, and thus the difference becomes greater. 
The above proposition suggests that for the image to be realistic with high probability, we must have a large enough $t_0$. On the flip side, if $t_0$ is too large, then the faithfulness to the guide deteriorates, and $\mathrm{SDEdit}$ will produce random realistic images (with the extreme case being unconditional image synthesis). 

\paragraph{Choice of $t_0$.} We note that the quality of the guide may affect the overall quality of the synthesized image. 
For reasonable guides, we find that $t_0  \in [0.3, 0.6]$ works well.
However, if the guide is an image with only white pixels, then even the closest ``realistic'' samples from the model distribution can be quite far, and we must sacrifice faithfulness for better realism by choosing a large $t_0$. 
In interactive settings (where user draws a sketch-based guide), we can initialize $t_0 \in [0.3, 0.6]$, synthesize a candidate with SDEdit, and ask the user whether the sample should be more faithful or more realistic; from the responses, we can obtain a reasonable $t_0$ via binary search.
In large-scale non-interactive settings (where we are given a set of produced guides), 
we can perform a similar binary search on a randomly selected image to obtain $t_0$ and subsequently fix $t_0$ for all guides in the same task.
Although different guides could potentially have different optimal $t_0$, we empirically observe that the shared $t_0$ works well for all reasonable guides in the same task.

\paragraph{Detailed algorithm and extensions.} 
We present the general algorithm for VE-SDE in \cref{alg:synthesis_ve}.
Due to space limit, we describe our detailed algorithm for VP-SDE in \cref{app:sdedit}. Essentially, the algorithm is an Euler-Maruyama method for solving $\mathrm{SDEdit}(\rvx^{(g)}; t_0, \theta)$.
For cases where we wish to keep certain parts of the synthesized images to be identical to that of the guides, we can also introduce an additional channel that masks out parts of the image we do not wish to edit. This is a slight modification to the SDEdit procedure mentioned in the main text, and we discuss the details in \cref{app:sec:algorithm}.

\begin{algorithm}
  \caption{Guided image synthesis and editing with SDEdit (VE-SDE)}\label{alg:synthesis_ve}
  \begin{algorithmic}
  \Require{$\x^{(g)}$ (guide), $t_0$ (SDE hyper-parameter), $N$ (total denoising steps)
  }
  \State{$\Delta t \gets \frac{t_0}{N}$}
    \State{$\rvz \sim \mcal{N}(\bm{0}, \bm{I})$}
    \State{$\rvx \gets \rvx + \sigma(t_0) \rvz$}
    \For{$n \gets N$ \textbf{to} $1$}
        \State{$t \gets t_0\frac{n}{N}$}
        \State{$\rvz \sim \mcal{N}(\bm{0}, \bm{I})$} 
        \State{$\epsilon \gets \sqrt{\sigma^2(t) - \sigma^2(t - \Delta t)}$}
        \State{$\rvx \gets \rvx + \epsilon^2 \vs_\vtheta(\rvx, t) + \epsilon \rvz$}
    \EndFor
  \State{\textbf{Return} $\rvx$}
  \end{algorithmic}
\end{algorithm}

\section{Related Work}

\myparagraph{Conditional GANs.} Conditional GANs for image editing~\citep{isola2017image,zhu2017unpaired,jo2019sc,liu2021deflocnet} learn to directly generate an image based on a user input, and have demonstrated success on a variety of tasks including image synthesis and editing~\citep{Faceshop,chen2017photographic,dekel2018sparse,wang2018pix2pixHD,park2019SPADE,zhu2020sean,jo2019sc,liu2021deflocnet}, 
inpainting~\citep{pathak2016context,iizuka2017globally,yang2017high,liu2018image}, photo colorization~\citep{zhang2016colorful,larsson2016learning,zhang2017real,he2018deep}, semantic image texture and geometry synthesis~\citep{TexSyn18,TerrainSyn17,xian2017texturegan}. 
They have also achieved strong performance on image editing using user sketch or color~\citep{jo2019sc,liu2021deflocnet,sangkloy2016scribbler}.
However, conditional models have to be trained on both original and edited images, thus requiring data collection and model re-training for new editing tasks. Thus, applying such methods to on-the-fly image manipulation 
is still challenging since a new model needs to be trained for each new application. Unlike conditional GANs, SDEdit only requires training on the original image. As such, it can be directly applied to various editing tasks at test time as illustrated in \cref{fig:teaser}.

\myparagraph{GANs inversion and editing.} Another mainstream approach to image editing involves GAN inversion~\citep{zhu2016generative,brock2017neural}, where the input is first projected into the latent space of an unconditional GAN before synthesizing a new image from the modified latent code.
Several methods have been proposed in this direction, including fine-tuning network weights for each image~\citep{bau2019semantic,pan2020exploiting,roich2021pivotal}, choosing better or multiple layers to project and edit~\citep{abdal2019image2stylegan,abdal2020image2stylegan++,gu2020image,wu2021stylespace}, designing better encoders~\citep{richardson2021encoding,tov2021designing}, modeling image corruption and transformations~\citep{anirudh2020mimicgan,huh2020transforming}, and discovering meaningful latent directions~\citep{shen2020interpreting,goetschalckx2019ganalyze,jahanian2019steerability,harkonen2020ganspace}. 
However, these methods need to define different loss functions for different tasks. They also require GAN inversion, which can be inefficient and inaccurate for various datasets~\citep{huh2020transforming,karras2020analyzing,bau2019seeing,xu2021generative}.

\myparagraph{Other generative models.}
Recent advances in training non-normalized probabilistic models, such as score-based generative models~\citep{song2019generative,song2020improved,song2021scorebased, ho2020denoising, song2020denoising, jolicoeur-martineau2021adversarial} and energy-based models~\citep{ackley1985learning,gao2017learning,du2019implicit,xie2018cooperative, xie2016theory, song2021train}, have 
achieved comparable image sample quality as GANs. However, most of the prior works in this direction have focused on unconditional image generation and density estimation, and state-of-the-art techniques for image editing and synthesis are still dominated by GAN-based methods. In this work, we focus on the recently emerged generative modeling with stochastic differential equations (SDE), and study its application to controllable image editing and synthesis tasks. 
A concurrent work~\citep{choi2021ilvr} performs conditional image synthesis with diffusion models, where the conditions can be represented as the known function of the underlying true image.

\section{Experiments}
\begin{figure}
\vspace{-10pt}
\centering
\includegraphics[width=\linewidth]{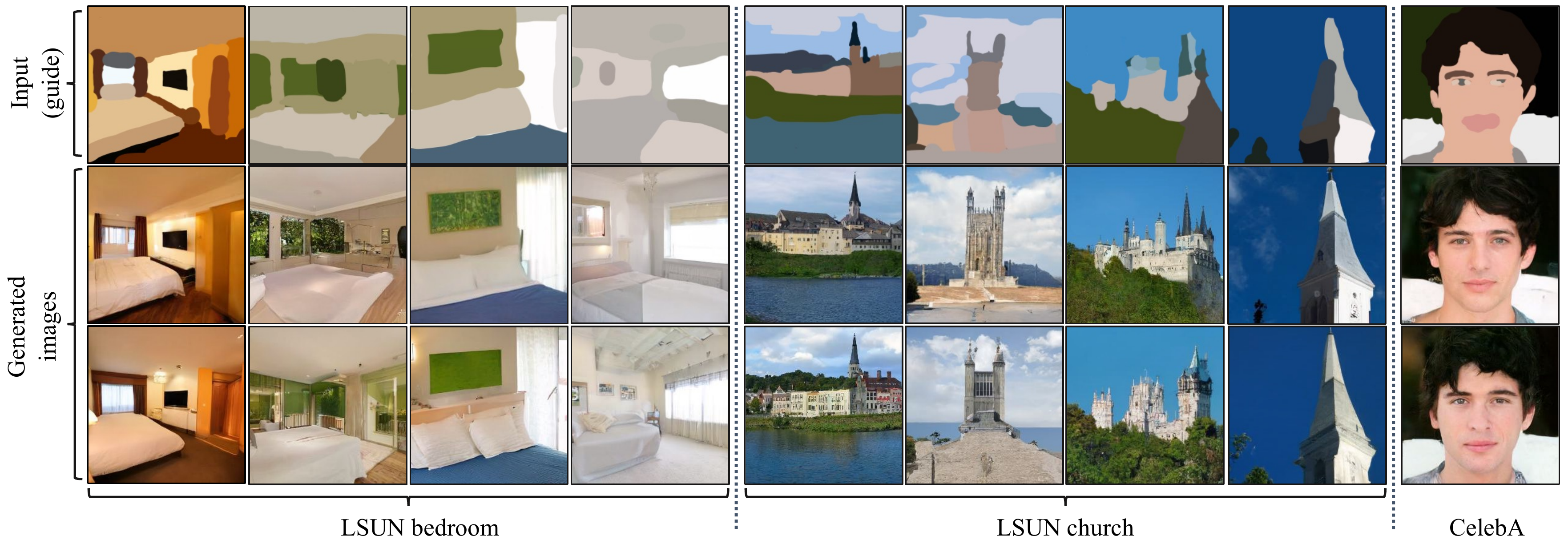}
\caption{
{{\model can generate realistic, faithful and diverse images} for a given stroke input drawn by human.}
}
\label{fig:sde_stroke_generation}
\vspace{-10pt}
\end{figure}

In this section, we show that \model is able to outperform state-of-the-art GAN-based models on stroke-based image synthesis and editing as well as image compositing.
Both \model
and the baselines use publicly available pre-trained checkpoints. Based on the availability of open-sourced SDE checkpoints, we use VP-SDE for experiments on LSUN datasets, and VE-SDE for experiments on CelebA-HQ.

\begin{table}[H]
\vspace{-10pt}
\newcommand{\mround}[1]{\round{#1}{2}}
    \newcommand{\nastar}{\multicolumn{1}{c}{\hspace{15pt}--~*}}
    \newcommand{\na}{\multicolumn{1}{c}{--}}
    \centering %
    {
        \centering %

        \begin{adjustbox}{max width=0.9\linewidth}
        \begin{tabular}{l @{\extracolsep{5pt}}  c c c c @{}}
        \\
        \toprule
        {Baselines}
        & Faithfulness score ($L_2$) $\downarrow$
        & \shortstack{\model is more realistic (MTurk) $\uparrow$}
        & \shortstack{\model is more satisfactory (Mturk) $\uparrow$}
        \\ 
        \midrule
        In-domain GAN-1 %
        & 101.18 %
        & 94.96\% %
        & 89.48\% %
        \\
        In-domain GAN-2%
        & 57.11 %
        & 97.87\% %
        & 89.51\% %
        
        \\
        StyleGAN2-ADA%
        & 68.12 %
        & 98.09\% %
        & 91.72\% %
        
        \\
        e4e%
        & 53.76 %
         & 80.34\% %
        & 75.43\%
       
        \\
        \model
        & \textbf{32.55} %
        & \na
        & \na
        \\
    \bottomrule
        \end{tabular}
        \end{adjustbox}
}
\caption{
{{\model outperforms all the
    GAN baselines}} on stroke-based generation on LSUN (bedroom).
The input strokes are created by human users.
The rightmost two columns stand for the percentage of MTurk workers that prefer \model to the baseline for pairwise comparison.}
\label{tab:stroke_lsun}
\vspace{-10pt}
\end{table}

\paragraph{Evaluation metrics.}
We evaluate the editing results based on \emph{realism} and \emph{faithfulness}. To quantify \emph{realism}, we use Kernel Inception Score (KID) between the generated images and the target realistic image dataset (details in \cref{app:sec:generating_stroke}), and pairwise human evaluation between different approaches with Amazon Mechanical Turk (MTurk).
To quantify \emph{faithfulness}, we report the $L_2$ distance summed over all pixels between the guide and the edited output image normalized to [0,1]. We also consider LPIPS~\citep{zhang2018context} and MTurk human evaluation for certain experiments. 
To quantify the overall human satisfaction score (\emph{realism} + \emph{faithfulness}), we leverage MTurk human evaluation to perform pairwise comparsion between the baselines and \model (see \cref{app:sec:mturk}).

\subsection{Stroke-Based Image Synthesis}

Given an input stroke painting, our goal is to generate a \emph{realistic} and \emph{faithful} image 
\emph{when no paired data is available}. We consider stroke painting guides created by human users (see \cref{fig:sde_stroke_generation}). At the same time, we also propose an algorithm to automatically simulate user stroke paintings based on a source image (see \cref{fig:sde_stroke_generation_baseline}), allowing us to perform large scale quantitative evaluations for \model. We provide more details in \cref{app:sec:generating_stroke}.

\begin{figure}
\centering
\includegraphics[width=\linewidth]{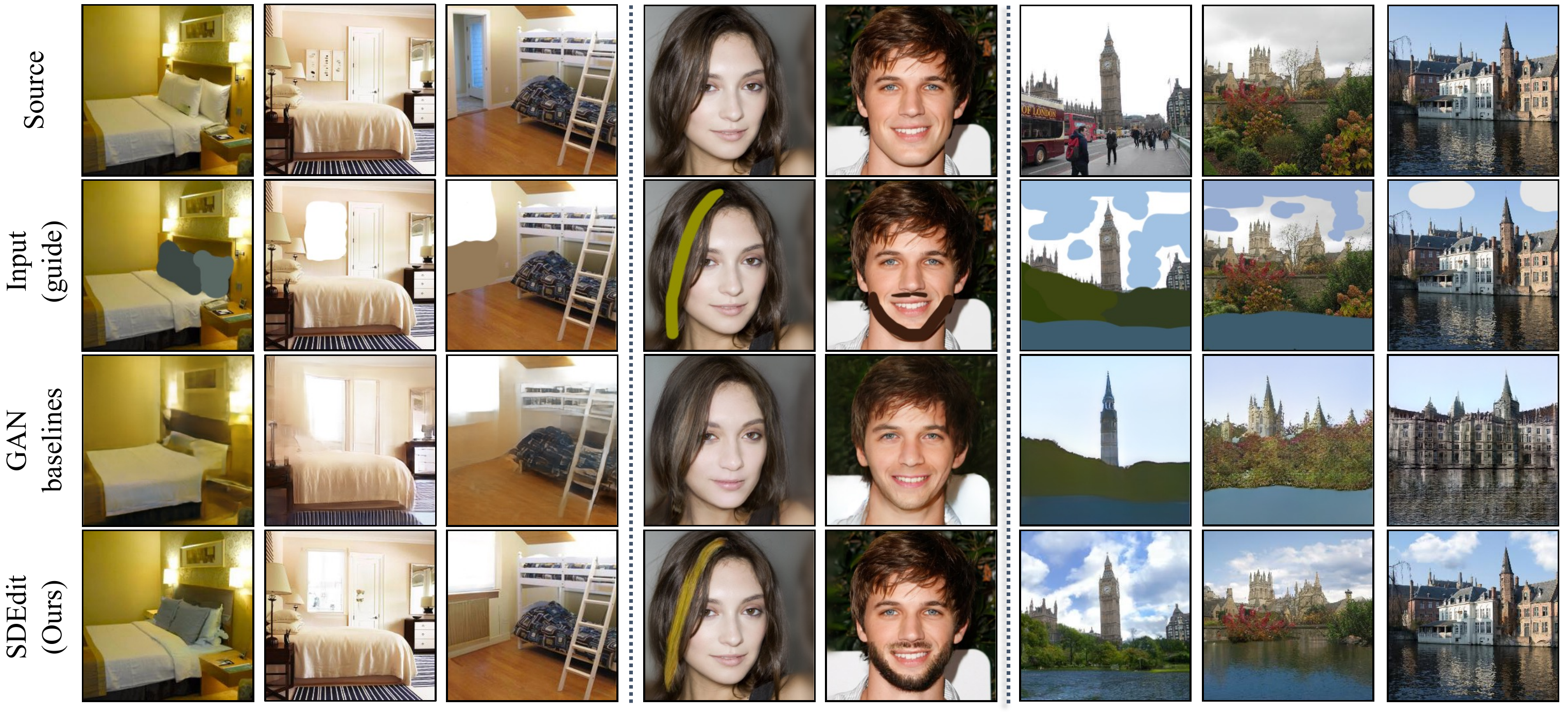}
\caption{Stroke-based image editing with \model
on LSUN bedroom, CelebA-HQ, and LSUN church datasets. For comparison, we show the results of GAN baselines,
where results for LSUN bedroom and CelebA-HQ are obtained by in-domain GAN (the leftmost 5 panels), and results for LSUN church are from StyleGAN2-ADA (the rightmost 3 panels). We observe that {{\model is able to produce more faithful and realistic editing compared to the baselines}}.
}
\label{fig:sde_stroke_edit}
\vspace{-5pt}
\end{figure}

\paragraph{Baselines.} For comparison, we choose three state-of-the-art
GAN-based 
image editing and synthesis methods as our baselines. 
Our first baseline is the image projection method used in StyleGAN2-ADA\footnote{\url{https://github.com/NVlabs/stylegan2-ada}}~\citep{karras2020training}, where inversion is done in the $W^+$ space of StyleGANs by minimizing the perceptual loss. Our second baseline is in-domain GAN\footnote{\url{https://github.com/genforce/idinvert_pytorch}}~\citep{zhu2020domain}, where inversion is accomplished by running optimization steps on top of an encoder. Specifically, we consider two versions of the in-domain GAN inversion techniques: the first one (denoted as In-domain GAN-1) only uses the encoder to maximize the inversion speed, whereas the second (denoted as In-domain GAN-2) runs additional optimization steps to maximize the inversion accuracy. Our third baseline is e4e\footnote{\url{https://github.com/omertov/encoder4editing}}~\citep{tov2021designing}, whose encoder objective is explicitly designed to balance between perceptual quality and editability by encouraging to invert images close to $W$ space of a pretrained StyleGAN model.

\begin{wrapfigure}{r}{0.5\textwidth}
 \newcommand{\mround}[1]{\round{#1}{2}}
    \newcommand{\nastar}{\multicolumn{1}{c}{\hspace{15pt}--~*}}
    \newcommand{\na}{\multicolumn{1}{c}{--}}
    \centering %
    {
        \centering %
        \begin{adjustbox}{max width=0.99\linewidth}
  \begin{tabular}{@{}lllll@{}}
\toprule
\multirow{2}{*}{Methods} & \multicolumn{2}{c}{LSUN Bedroom} & \multicolumn{2}{c}{LSUN Church} \\
\cline{2-5}
& $L_2$ $\downarrow$    & KID $\downarrow$  & $L_2$ $\downarrow$           & KID $\downarrow$         \\ 
\midrule
In-domain GAN-1 & 105.23       & 0.1147      & -             &  -       \\
In-domain GAN-2 & 76.11       & 0.2070      & -             &  -       \\
StyleGAN2-ADA   & 74.03       & 0.1750      & 72.41      & 0.1544      \\
e4e             & 52.40       & 0.0464      & 68.53      & 0.0354     \\
SDEdit (ours)   & \textbf{36.76}       & \textbf{0.0030}      & \textbf{37.67}      & \textbf{0.0156}      \\ \bottomrule
\end{tabular}
        \end{adjustbox}
}
\captionof{table}{{{\model outperforms {all} the
    GAN baselines on both faithfulness and realism}} for stroke-based image generation.
The input strokes are generated with the stroke-simulation algorithm.
KID is computed using the generated images and the corresponding validation sets (see \cref{app:sec:generating_stroke}). 
}
\vspace{-5pt}
\label{tab:stroke_lsun_synthetic}
\end{wrapfigure}

\paragraph{Results.}
We present qualitative comparison results in
\cref{fig:sde_stroke_generation_baseline}.
We observe that all baselines struggle to generate realistic images based on stroke painting inputs whereas \model successfully generates realistic images that preserve semantics of the input stroke painting. 
As shown in \cref{fig:sde_stroke_generation}, \model can also synthesize diverse images for the same input. 
We present quantitative comparison results using user-created stroke guides in \cref{tab:stroke_lsun} and algorithm-simulated stroke guides in \cref{tab:stroke_lsun_synthetic}.
We report the $L_2$ distance for faithfulness comparison, and leverage MTurk (see \cref{app:sec:mturk})
or KID scores for realism comparison.
To quantify the overall human satisfaction score (faithfulness + realism),
we ask a different set of MTurk workers to perform another 3000 pairwise comparisons between the baselines and \model based on \emph{both faithfulness} and \emph{realism}.
We observe that {{\model outperforms GAN baselines on {all} the evaluation metrics}}, beating the baselines by more than \textbf{80\%} on realism scores and \textbf{75\%} on overall satisfaction scores. 
We provide more experimental details in \cref{app:sdedit} and more results in \cref{app:results}.

\subsection{Flexible Image Editing}
In this section, we show that \model is able to outperform existing GAN-based models on image editing tasks.
We focus on LSUN (bedroom, church) and CelebA-HQ datasets, and provide more details on the experimental setup in the \cref{app:experiment}. %

\paragraph{Stroke-based image editing.} 
Given an image with stroke edits, we want to generate a realistic and faithful image based on the user edit.
We consider the same GAN-based baselines \citep{zhu2020domain,karras2020training,tov2021designing} as our previous experiment.
As shown in \cref{fig:sde_stroke_edit}, results generated by the baselines tend to introduce undesired modifications, occasionally making the region outside the stroke blurry. In contrast, \model is able to generate image edits that are {\emph{both realistic}} and {\emph{faithful}} to the input,
while avoiding making undesired modifications. We provide extra results in \cref{app:results}.

\paragraph{Image compositing.}
\begin{figure}
\vspace{-5pt}
\centering
\includegraphics[width=\linewidth]{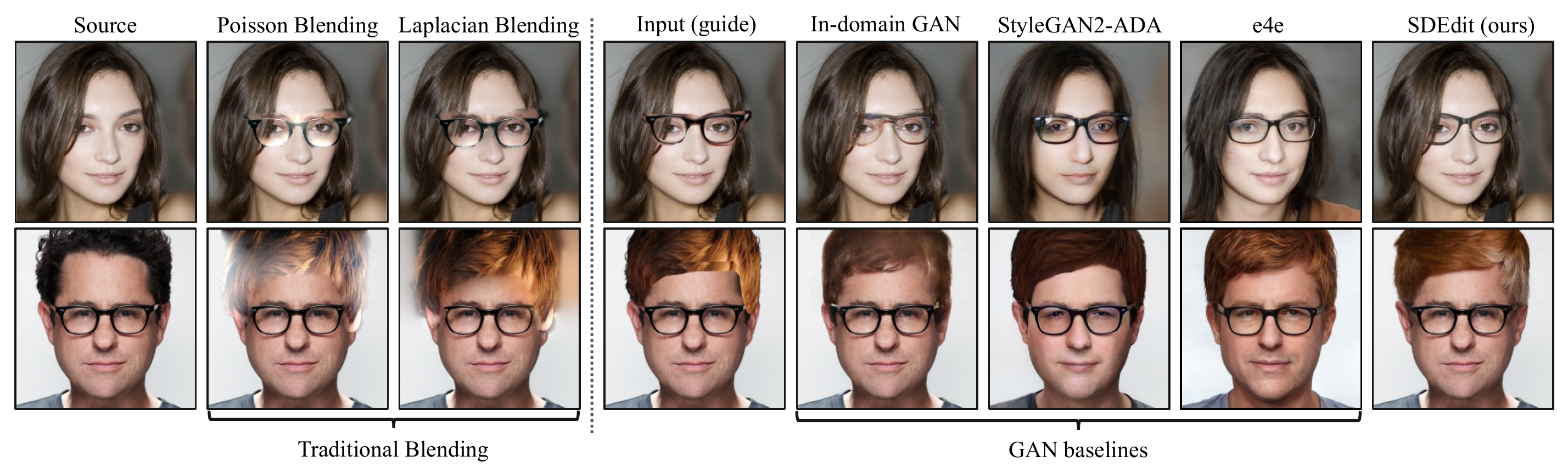}
\caption{
{{\model is able to achieve realistic while more faithful editing results}} compared to traditional blending and recent GAN-based approaches for image compositing on CelebA-HQ. 
Quantitative results are reported in \cref{tab:image_edit_celeba}.
}
\label{fig:celeba_edit}
\vspace{-5pt}
\end{figure}
We focus on compositing images on the CelebA-HQ dataset~\citep{karras2017progressive}. 
Given an image randomly sampled from the dataset, we ask users to specify how they want the edited image to look like using pixel patches copied from other reference images as well as the pixels they want to perform modifications (see \cref{fig:celeba_edit}). 
We compare our method with traditional blending algorithms \citep{burt1987laplacian,perez2003poisson} and the same GAN baselines considered previously. We perform qualitative comparison in \cref{fig:celeba_edit}.
For quantitative comparison, we report the $L_2$ distance to quantify faithfulness. 
To quantify realism, we ask MTurk workers to perform 1500 pairwise comparisons between the baselines and \model.
To quantify user satisfaction score (faithfulness + realism), 
we ask different workers to perform another 1500 pairwise comparisons against \model.
To quantify undesired changes (\eg change of identity), we follow \cite{bau2020rewriting} to compute masked LPIPS~\citep{zhang2018context}. As evidenced in \cref{tab:image_edit_celeba}, we observe that \model is able to generate both faithful and realistic images with {{much better LPIPS scores}} than the baselines, outperforming the baselines by up to \textbf{83.73\%} on overall satisfaction score and \textbf{75.60\%} on realism. Although our realism score is marginally lower than e4e, images generated by \model are more faithful and more satisfying overall.
We present more experiment details in \cref{app:experiment}. %

\begin{table}%
    \newcommand{\mround}[1]{\round{#1}{2}}
    \newcommand{\nastar}{\multicolumn{1}{c}{\hspace{15pt}--~*}}
    \newcommand{\na}{\multicolumn{1}{c}{--}}
    \centering \small
    {
        \centering \small
        
        \resizebox{
          \ifdim\width>\textwidth
            \textwidth
          \else
            0.92\width
          \fi
        }{!}{%
        \begin{tabular}{l @{\extracolsep{5pt}}  c c c c @{}}
        \\
        \toprule
        \multirow{2}{*}{Methods}
        & $L_2$ $\downarrow$
        & \shortstack{\model  more realistic}
        & \shortstack{\model  more satisfactory}
        & \shortstack{LPIPS}
        \\ 
        & (faithfulness) 
        & (Mturk) $\uparrow$
        & (Mturk) $\uparrow$
        & (masked) $\downarrow$
        \\ 
        \midrule
        Laplacian Blending%
        & 68.45 %
        & 75.27\% %
        & 83.73\% %
        & 0.09 %
        \\
        Poisson Blending%
        & 63.04 %
        & 75.60\% %
        & 82.18\% %
        & 0.05 %
        \\
        In-domain GAN%
        & 36.67 %
        & 53.08\% %
        & 73.53\% %
        & 0.23 %
        \\
        StyleGAN2-ADA%
        & 69.38 %
        & 74.12\% %
        & 83.43\%  %
        & 0.21 %
        \\
        e4e%
        & 53.90 %
        &43.67\%
        &66.00\%
        & 0.33%
        \\
        \model (ours)
        &\textbf{21.70} %
        &\na
        &\na
        &\textbf{0.03}\\
    \bottomrule
        \end{tabular}
}
    }
    \vspace{-5pt}
    \caption{
    Image compositing experiments on CelebA-HQ.
    The middle two columns indicate the percentage of MTurk workers that {{prefer \model}}.
    We also report the masked LPIPS distance between edited and unchanged images to quantify undesired changes outside the masks. We observe that \model is able to achieve realistic editing while being more faithful than the baselines, {beating the baseline by up to \textbf{83.73\%}} on overall satisfaction score by human evaluators.}
    \vspace{-10pt}
    \label{tab:image_edit_celeba}%
\end{table}

\section{Conclusion}
We propose \modelfull (\model), a guided image editing and synthesis method via generative modeling of images with stochastic differential equations (SDEs) allowing for balanced realism and faithfulness. Unlike image editing techniques via GAN inversion, our method does not require task-specific optimization algorithms for reconstructing inputs, and is particularly suitable for datasets or tasks where GAN inversion losses are hard to design or optimize.
Unlike conditional GANs, our method does not  require collecting new datasets for the ``guide" images or re-training models, both of which could be expensive or time-consuming.
We demonstrate that \model outperforms existing GAN-based methods on stroke-based image synthesis,  stroke-based image editing and image compositing
without task-specific training.

\myparagraph{Acknowledgments.}
The authors want to thank Kristy Choi for proofreading. This research
was supported by NSF (\#1651565, \#1522054, \#1733686), ONR (N00014-19-1-2145), AFOSR
(FA9550-19-1-0024), ARO, Autodesk, Stanford HAI, Amazon ARA, and Amazon AWS. Yang Song is supported by the Apple PhD Fellowship in AI/ML. J.-Y. Zhu is partly supported by Naver Corporation. 
\section*{Ethics Statement}
In this work, we propose \model, which is a new image synthesis and editing methods based on generative stochastic differential equations (SDEs). In our experiments, all the considered datasets are open-sourced and publicly available, being used under permission. 
Similar to commonly seen deep-learning based image synthesis and editing algorithms, our method has both positive and negative societal impacts depending on the applications and usages. 
On the positive side, \model enables everyday users with or without artistic expertise to create and edit photo-realistic images with minimum effort, lowering the barrier to entry for visual content creation. 
On the other hand, \model can be used to generate high-quality edited images that are hard to be distinguished from real ones by humans, which could be used in malicious ways to deceive humans and spread misinformation. 
Similar to commonly seen deep-learning models (such as GAN-based methods for face-editing), SDEdit might be exploited by malicious users with potential negative impacts. In our code release, we will explicitly specify allowable uses of our system with appropriate licenses.

We also notice that forensic methods for detecting fake machine-generated images mostly focus on distinguishing samples generated by GAN-based approaches. Due to the different underlying nature between GANs and generative SDEs, we observe that state-of-the-art approaches for detecting fake images generated by GANs~\citep{wang2019cnngenerated} struggle to distinguish fake samples generated by SDE-based models. For instance, on the LSUN bedroom dataset, it only successfully detects less than $3\%$ of \model-generated images whereas being able to distinguish up to $93\%$ on GAN-based generation.
Based on these observations, we believe developing forensic methods for SDE-based models is also critical as SDE-based methods become more prevalent.

For human evaluation experiments, we leveraged Amazon Mechanical Turk (MTurk). 
For each worker, the evaluation HIT 
contains 15 pairwise comparison questions for comparing edited images. The reward per task is kept as 0.2\$.  Since each task takes around 1 minute, the wage is around 12\$ per hour. 
We provide more details on Human evaluation experiments in \cref{app:sec:mturk}.
We also note that the bias of human evaluators (MTurk workers) and the bias of users (through the input ``guidance") could potentially 
affect the evaluation metrics and results used to track the progress towards guided image synthesis and editing.

\section*{Reproducibility Statement}
\begin{itemize}
    \item [1.] Our code will be released upon publication. 
    \item [2.] We use open source datasets and SDE checkpoints on the corresponding datasets. We did not train any SDE models. 
    \item [3.] Proofs are provided in \cref{app:sec:proofs}.
    \item [4.] Extra details on \model and pseudocode are provided in \cref{app:sdedit}.
    \item [5.] Details on experimental settings are provided in \cref{app:experiment}.
    \item [6.] Extra experimental results are provided in \cref{app:results}.
    \item [7.] Details on human evaluation are provided in \cref{app:sec:mturk}.
\end{itemize}

\newpage
{\small
\bibliographystyle{iclr2022_conference}
\bibliography{iclr2022_conference}

\begin{thebibliography}{71}
\providecommand{\natexlab}[1]{#1}
\providecommand{\url}[1]{\texttt{#1}}
\expandafter\ifx\csname urlstyle\endcsname\relax
  \providecommand{\doi}[1]{doi: #1}\else
  \providecommand{\doi}{doi: \begingroup \urlstyle{rm}\Url}\fi

\bibitem[Abdal et~al.(2019)Abdal, Qin, and Wonka]{abdal2019image2stylegan}
Rameen Abdal, Yipeng Qin, and Peter Wonka.
\newblock Image2stylegan: How to embed images into the stylegan latent space?
\newblock In \emph{Proceedings of the IEEE/CVF International Conference on
  Computer Vision}, pp.\  4432--4441, 2019.

\bibitem[Abdal et~al.(2020)Abdal, Qin, and Wonka]{abdal2020image2stylegan++}
Rameen Abdal, Yipeng Qin, and Peter Wonka.
\newblock Image2stylegan++: How to edit the embedded images?
\newblock In \emph{Proceedings of the IEEE/CVF Conference on Computer Vision
  and Pattern Recognition}, pp.\  8296--8305, 2020.

\bibitem[Ackley et~al.(1985)Ackley, Hinton, and Sejnowski]{ackley1985learning}
David~H Ackley, Geoffrey~E Hinton, and Terrence~J Sejnowski.
\newblock A learning algorithm for boltzmann machines.
\newblock \emph{Cognitive science}, 9\penalty0 (1):\penalty0 147--169, 1985.

\bibitem[Anderson(1982)]{anderson1982reverse}
Brian~DO Anderson.
\newblock Reverse-time diffusion equation models.
\newblock \emph{Stochastic Processes and their Applications}, 12\penalty0
  (3):\penalty0 313--326, 1982.

\bibitem[Anirudh et~al.(2020)Anirudh, Thiagarajan, Kailkhura, and
  Bremer]{anirudh2020mimicgan}
Rushil Anirudh, Jayaraman~J Thiagarajan, Bhavya Kailkhura, and Peer-Timo
  Bremer.
\newblock Mimicgan: Robust projection onto image manifolds with corruption
  mimicking.
\newblock \emph{International Journal of Computer Vision}, pp.\  1--19, 2020.

\bibitem[Bau et~al.(2019{\natexlab{a}})Bau, Strobelt, Peebles, Wulff, Zhou,
  Zhu, and Torralba]{bau2019semantic}
David Bau, Hendrik Strobelt, William Peebles, Jonas Wulff, Bolei Zhou, Jun-Yan
  Zhu, and Antonio Torralba.
\newblock Semantic photo manipulation with a generative image prior.
\newblock \emph{ACM SIGGRAPH}, 38\penalty0 (4):\penalty0 1--11,
  2019{\natexlab{a}}.

\bibitem[Bau et~al.(2019{\natexlab{b}})Bau, Zhu, Wulff, Peebles, Strobelt,
  Zhou, and Torralba]{bau2019seeing}
David Bau, Jun-Yan Zhu, Jonas Wulff, William Peebles, Hendrik Strobelt, Bolei
  Zhou, and Antonio Torralba.
\newblock Seeing what a gan cannot generate.
\newblock In \emph{Proceedings of the IEEE/CVF International Conference on
  Computer Vision}, pp.\  4502--4511, 2019{\natexlab{b}}.

\bibitem[Bau et~al.(2020)Bau, Liu, Wang, Zhu, and Torralba]{bau2020rewriting}
David Bau, Steven Liu, Tongzhou Wang, Jun-Yan Zhu, and Antonio Torralba.
\newblock Rewriting a deep generative model.
\newblock In \emph{European Conference on Computer Vision (ECCV)}, 2020.

\bibitem[Bi{\'n}kowski et~al.(2018)Bi{\'n}kowski, Sutherland, Arbel, and
  Gretton]{binkowski2018demystifying}
Miko{\l}aj Bi{\'n}kowski, Danica~J Sutherland, Michael Arbel, and Arthur
  Gretton.
\newblock Demystifying mmd gans.
\newblock \emph{arXiv preprint arXiv:1801.01401}, 2018.

\bibitem[Brock et~al.(2017)Brock, Lim, Ritchie, and Weston]{brock2017neural}
Andrew Brock, Theodore Lim, James~M Ritchie, and Nick Weston.
\newblock Neural photo editing with introspective adversarial networks.
\newblock In \emph{International Conference on Learning Representations
  (ICLR)}, 2017.

\bibitem[Burt \& Adelson(1987)Burt and Adelson]{burt1987laplacian}
Peter~J Burt and Edward~H Adelson.
\newblock The laplacian pyramid as a compact image code.
\newblock In \emph{Readings in computer vision}, pp.\  671--679. Elsevier,
  1987.

\bibitem[Chen \& Koltun(2017)Chen and Koltun]{chen2017photographic}
Qifeng Chen and Vladlen Koltun.
\newblock Photographic image synthesis with cascaded refinement networks.
\newblock In \emph{IEEE International Conference on Computer Vision (ICCV)},
  2017.

\bibitem[Choi et~al.(2021)Choi, Kim, Jeong, Gwon, and Yoon]{choi2021ilvr}
Jooyoung Choi, Sungwon Kim, Yonghyun Jeong, Youngjune Gwon, and Sungroh Yoon.
\newblock Ilvr: Conditioning method for denoising diffusion probabilistic
  models.
\newblock \emph{arXiv preprint arXiv:2108.02938}, 2021.

\bibitem[Dekel et~al.(2018)Dekel, Gan, Krishnan, Liu, and
  Freeman]{dekel2018sparse}
Tali Dekel, Chuang Gan, Dilip Krishnan, Ce~Liu, and William~T Freeman.
\newblock Sparse, smart contours to represent and edit images.
\newblock In \emph{IEEE Conference on Computer Vision and Pattern Recognition
  (CVPR)}, 2018.

\bibitem[Dhariwal \& Nichol(2021)Dhariwal and Nichol]{dhariwal2021diffusion}
Prafulla Dhariwal and Alex Nichol.
\newblock Diffusion models beat gans on image synthesis.
\newblock \emph{arXiv preprint arXiv:2105.05233}, 2021.

\bibitem[Du \& Mordatch(2019)Du and Mordatch]{du2019implicit}
Yilun Du and Igor Mordatch.
\newblock Implicit generation and generalization in energy-based models.
\newblock \emph{arXiv preprint arXiv:1903.08689}, 2019.

\bibitem[Gao et~al.(2017)Gao, Lu, Zhou, Zhu, and Nian~Wu]{gao2017learning}
Ruiqi Gao, Yang Lu, Junpei Zhou, Song-Chun Zhu, and Ying Nian~Wu.
\newblock Learning energy-based models as generative convnets via multi-grid
  modeling and sampling.
\newblock \emph{arXiv e-prints}, pp.\  arXiv--1709, 2017.

\bibitem[Goetschalckx et~al.(2019)Goetschalckx, Andonian, Oliva, and
  Isola]{goetschalckx2019ganalyze}
Lore Goetschalckx, Alex Andonian, Aude Oliva, and Phillip Isola.
\newblock Ganalyze: Toward visual definitions of cognitive image properties.
\newblock In \emph{IEEE International Conference on Computer Vision (ICCV)},
  2019.

\bibitem[Gu et~al.(2020)Gu, Shen, and Zhou]{gu2020image}
Jinjin Gu, Yujun Shen, and Bolei Zhou.
\newblock Image processing using multi-code gan prior.
\newblock In \emph{IEEE Conference on Computer Vision and Pattern Recognition
  (CVPR)}, 2020.

\bibitem[Gu\'{e}rin et~al.(2017)Gu\'{e}rin, Digne, Galin, Peytavie, Wolf,
  Benes, and Martinez]{TerrainSyn17}
\'{E}ric Gu\'{e}rin, Julie Digne, \'{E}ric Galin, Adrien Peytavie, Christian
  Wolf, Bedrich Benes, and Benoundefinedt Martinez.
\newblock Interactive example-based terrain authoring with conditional
  generative adversarial networks.
\newblock \emph{ACM Transactions on Graphics (TOG)}, 36\penalty0 (6), 2017.

\bibitem[H{\"a}rk{\"o}nen et~al.(2020)H{\"a}rk{\"o}nen, Hertzmann, Lehtinen,
  and Paris]{harkonen2020ganspace}
Erik H{\"a}rk{\"o}nen, Aaron Hertzmann, Jaakko Lehtinen, and Sylvain Paris.
\newblock Ganspace: Discovering interpretable gan controls.
\newblock In \emph{Advances in Neural Information Processing Systems}, 2020.

\bibitem[He et~al.(2018)He, Chen, Liao, Sander, and Yuan]{he2018deep}
Mingming He, Dongdong Chen, Jing Liao, Pedro~V Sander, and Lu~Yuan.
\newblock Deep exemplar-based colorization.
\newblock \emph{ACM Transactions on Graphics (TOG)}, 37\penalty0 (4):\penalty0
  1--16, 2018.

\bibitem[Ho et~al.(2020)Ho, Jain, and Abbeel]{ho2020denoising}
Jonathan Ho, Ajay Jain, and Pieter Abbeel.
\newblock Denoising diffusion probabilistic models.
\newblock \emph{arXiv preprint arXiv:2006.11239}, 2020.

\bibitem[Huh et~al.(2020)Huh, Zhang, Zhu, Paris, and
  Hertzmann]{huh2020transforming}
Minyoung Huh, Richard Zhang, Jun-Yan Zhu, Sylvain Paris, and Aaron Hertzmann.
\newblock Transforming and projecting images into class-conditional generative
  networks.
\newblock In \emph{European Conference on Computer Vision (ECCV)}, 2020.

\bibitem[Iizuka et~al.(2017)Iizuka, Simo-Serra, and
  Ishikawa]{iizuka2017globally}
Satoshi Iizuka, Edgar Simo-Serra, and Hiroshi Ishikawa.
\newblock Globally and locally consistent image completion.
\newblock \emph{ACM Transactions on Graphics (TOG)}, 36\penalty0 (4):\penalty0
  107, 2017.

\bibitem[Isola et~al.(2017)Isola, Zhu, Zhou, and Efros]{isola2017image}
Phillip Isola, Jun-Yan Zhu, Tinghui Zhou, and Alexei~A Efros.
\newblock Image-to-image translation with conditional adversarial networks.
\newblock In \emph{IEEE Conference on Computer Vision and Pattern Recognition
  (CVPR)}, 2017.

\bibitem[Jahanian et~al.(2020)Jahanian, Chai, and
  Isola]{jahanian2019steerability}
Ali Jahanian, Lucy Chai, and Phillip Isola.
\newblock On the''steerability" of generative adversarial networks.
\newblock In \emph{International Conference on Learning Representations
  (ICLR)}, 2020.

\bibitem[Jo \& Park(2019)Jo and Park]{jo2019sc}
Youngjoo Jo and Jongyoul Park.
\newblock Sc-fegan: Face editing generative adversarial network with user's
  sketch and color.
\newblock In \emph{Proceedings of the IEEE/CVF International Conference on
  Computer Vision}, pp.\  1745--1753, 2019.

\bibitem[Jolicoeur-Martineau et~al.(2021)Jolicoeur-Martineau,
  Pich{\'e}-Taillefer, Mitliagkas, and des
  Combes]{jolicoeur-martineau2021adversarial}
Alexia Jolicoeur-Martineau, R{\'e}mi Pich{\'e}-Taillefer, Ioannis Mitliagkas,
  and Remi~Tachet des Combes.
\newblock Adversarial score matching and improved sampling for image
  generation.
\newblock In \emph{International Conference on Learning Representations}, 2021.

\bibitem[Karras et~al.(2017)Karras, Aila, Laine, and
  Lehtinen]{karras2017progressive}
Tero Karras, Timo Aila, Samuli Laine, and Jaakko Lehtinen.
\newblock Progressive growing of gans for improved quality, stability, and
  variation.
\newblock \emph{arXiv preprint arXiv:1710.10196}, 2017.

\bibitem[Karras et~al.(2019)Karras, Laine, and Aila]{karras2019style}
Tero Karras, Samuli Laine, and Timo Aila.
\newblock A style-based generator architecture for generative adversarial
  networks.
\newblock In \emph{IEEE Conference on Computer Vision and Pattern Recognition
  (CVPR)}, 2019.

\bibitem[Karras et~al.(2020{\natexlab{a}})Karras, Aittala, Hellsten, Laine,
  Lehtinen, and Aila]{karras2020training}
Tero Karras, Miika Aittala, Janne Hellsten, Samuli Laine, Jaakko Lehtinen, and
  Timo Aila.
\newblock Training generative adversarial networks with limited data.
\newblock \emph{arXiv preprint arXiv:2006.06676}, 2020{\natexlab{a}}.

\bibitem[Karras et~al.(2020{\natexlab{b}})Karras, Laine, Aittala, Hellsten,
  Lehtinen, and Aila]{karras2020analyzing}
Tero Karras, Samuli Laine, Miika Aittala, Janne Hellsten, Jaakko Lehtinen, and
  Timo Aila.
\newblock Analyzing and improving the image quality of stylegan.
\newblock In \emph{IEEE Conference on Computer Vision and Pattern Recognition
  (CVPR)}, 2020{\natexlab{b}}.

\bibitem[Kawar et~al.(2021)Kawar, Vaksman, and Elad]{kawar2021snips}
Bahjat Kawar, Gregory Vaksman, and Michael Elad.
\newblock Snips: Solving noisy inverse problems stochastically.
\newblock \emph{arXiv preprint arXiv:2105.14951}, 2021.

\bibitem[Larsson et~al.(2016)Larsson, Maire, and
  Shakhnarovich]{larsson2016learning}
Gustav Larsson, Michael Maire, and Gregory Shakhnarovich.
\newblock Learning representations for automatic colorization.
\newblock In \emph{European Conference on Computer Vision (ECCV)}, 2016.

\bibitem[Laurent \& Massart(2000)Laurent and Massart]{laurent2000adaptive}
Beatrice Laurent and Pascal Massart.
\newblock Adaptive estimation of a quadratic functional by model selection.
\newblock \emph{Annals of Statistics}, pp.\  1302--1338, 2000.

\bibitem[Liu et~al.(2018)Liu, Reda, Shih, Wang, Tao, and
  Catanzaro]{liu2018image}
Guilin Liu, Fitsum~A Reda, Kevin~J Shih, Ting-Chun Wang, Andrew Tao, and Bryan
  Catanzaro.
\newblock Image inpainting for irregular holes using partial convolutions.
\newblock In \emph{European Conference on Computer Vision (ECCV)}, 2018.

\bibitem[Liu et~al.(2021)Liu, Wan, Huang, Song, Han, Liao, Jiang, and
  Liu]{liu2021deflocnet}
Hongyu Liu, Ziyu Wan, Wei Huang, Yibing Song, Xintong Han, Jing Liao, Bin
  Jiang, and Wei Liu.
\newblock Deflocnet: Deep image editing via flexible low-level controls.
\newblock In \emph{Proceedings of the IEEE/CVF Conference on Computer Vision
  and Pattern Recognition}, pp.\  10765--10774, 2021.

\bibitem[Pan et~al.(2020)Pan, Zhan, Dai, Lin, Loy, and Luo]{pan2020exploiting}
Xingang Pan, Xiaohang Zhan, Bo~Dai, Dahua Lin, Chen~Change Loy, and Ping Luo.
\newblock Exploiting deep generative prior for versatile image restoration and
  manipulation.
\newblock In \emph{European Conference on Computer Vision}, 2020.

\bibitem[Park et~al.(2019)Park, Liu, Wang, and Zhu]{park2019SPADE}
Taesung Park, Ming-Yu Liu, Ting-Chun Wang, and Jun-Yan Zhu.
\newblock Semantic image synthesis with spatially-adaptive normalization.
\newblock In \emph{IEEE Conference on Computer Vision and Pattern Recognition
  (CVPR)}, 2019.

\bibitem[Pathak et~al.(2016)Pathak, Krahenbuhl, Donahue, Darrell, and
  Efros]{pathak2016context}
Deepak Pathak, Philipp Krahenbuhl, Jeff Donahue, Trevor Darrell, and Alexei~A
  Efros.
\newblock Context encoders: Feature learning by inpainting.
\newblock In \emph{IEEE Conference on Computer Vision and Pattern Recognition
  (CVPR)}, 2016.

\bibitem[P{\'e}rez et~al.(2003)P{\'e}rez, Gangnet, and Blake]{perez2003poisson}
Patrick P{\'e}rez, Michel Gangnet, and Andrew Blake.
\newblock Poisson image editing.
\newblock In \emph{ACM SIGGRAPH}, pp.\  313--318, 2003.

\bibitem[Portenier et~al.(2018)Portenier, Hu, Szab\'{o}, Bigdeli, Favaro, and
  Zwicker]{Faceshop}
Tiziano Portenier, Qiyang Hu, Attila Szab\'{o}, Siavash~Arjomand Bigdeli, Paolo
  Favaro, and Matthias Zwicker.
\newblock Faceshop: Deep sketch-based face image editing.
\newblock \emph{ACM Transactions on Graphics (TOG)}, 37\penalty0 (4), 2018.

\bibitem[Richardson et~al.(2021)Richardson, Alaluf, Patashnik, Nitzan, Azar,
  Shapiro, and Cohen-Or]{richardson2021encoding}
Elad Richardson, Yuval Alaluf, Or~Patashnik, Yotam Nitzan, Yaniv Azar, Stav
  Shapiro, and Daniel Cohen-Or.
\newblock Encoding in style: a stylegan encoder for image-to-image translation.
\newblock In \emph{Proceedings of the IEEE/CVF Conference on Computer Vision
  and Pattern Recognition}, 2021.

\bibitem[Roich et~al.(2021)Roich, Mokady, Bermano, and
  Cohen-Or]{roich2021pivotal}
Daniel Roich, Ron Mokady, Amit~H Bermano, and Daniel Cohen-Or.
\newblock Pivotal tuning for latent-based editing of real images.
\newblock \emph{arXiv preprint arXiv:2106.05744}, 2021.

\bibitem[Sangkloy et~al.(2017)Sangkloy, Lu, Fang, Yu, and
  Hays]{sangkloy2016scribbler}
Patsorn Sangkloy, Jingwan Lu, Chen Fang, Fisher Yu, and James Hays.
\newblock Scribbler: Controlling deep image synthesis with sketch and color.
\newblock In \emph{IEEE Conference on Computer Vision and Pattern Recognition
  (CVPR)}, 2017.

\bibitem[Shen et~al.(2020)Shen, Gu, Tang, and Zhou]{shen2020interpreting}
Yujun Shen, Jinjin Gu, Xiaoou Tang, and Bolei Zhou.
\newblock Interpreting the latent space of gans for semantic face editing.
\newblock In \emph{IEEE Conference on Computer Vision and Pattern Recognition
  (CVPR)}, 2020.

\bibitem[Sohl-Dickstein et~al.(2015)Sohl-Dickstein, Weiss, Maheswaranathan, and
  Ganguli]{sohl2015deep}
Jascha Sohl-Dickstein, Eric~A Weiss, Niru Maheswaranathan, and Surya Ganguli.
\newblock Deep unsupervised learning using nonequilibrium thermodynamics.
\newblock \emph{arXiv preprint arXiv:1503.03585}, 2015.

\bibitem[Song et~al.(2020)Song, Meng, and Ermon]{song2020denoising}
Jiaming Song, Chenlin Meng, and Stefano Ermon.
\newblock Denoising diffusion implicit models.
\newblock \emph{arXiv preprint arXiv:2010.02502}, 2020.

\bibitem[Song \& Ermon(2019)Song and Ermon]{song2019generative}
Yang Song and Stefano Ermon.
\newblock Generative modeling by estimating gradients of the data distribution.
\newblock In \emph{Advances in Neural Information Processing Systems
  (NeurIPS)}, 2019.

\bibitem[Song \& Ermon(2020)Song and Ermon]{song2020improved}
Yang Song and Stefano Ermon.
\newblock Improved techniques for training score-based generative models.
\newblock \emph{arXiv preprint arXiv:2006.09011}, 2020.

\bibitem[Song \& Kingma(2021)Song and Kingma]{song2021train}
Yang Song and Diederik~P Kingma.
\newblock How to train your energy-based models.
\newblock \emph{arXiv preprint arXiv:2101.03288}, 2021.

\bibitem[Song et~al.(2021)Song, Sohl-Dickstein, Kingma, Kumar, Ermon, and
  Poole]{song2021scorebased}
Yang Song, Jascha Sohl-Dickstein, Diederik~P Kingma, Abhishek Kumar, Stefano
  Ermon, and Ben Poole.
\newblock Score-based generative modeling through stochastic differential
  equations.
\newblock In \emph{International Conference on Learning Representations
  (ICLR)}, 2021.

\bibitem[Tov et~al.(2021)Tov, Alaluf, Nitzan, Patashnik, and
  Cohen-Or]{tov2021designing}
Omer Tov, Yuval Alaluf, Yotam Nitzan, Or~Patashnik, and Daniel Cohen-Or.
\newblock Designing an encoder for stylegan image manipulation.
\newblock \emph{ACM Transactions on Graphics (TOG)}, 40\penalty0 (4):\penalty0
  1--14, 2021.

\bibitem[Vincent(2011)]{vincent2011connection}
Pascal Vincent.
\newblock A connection between score matching and denoising autoencoders.
\newblock \emph{Neural computation}, 23\penalty0 (7):\penalty0 1661--1674,
  2011.

\bibitem[Wang et~al.(2020)Wang, Wang, Zhang, Owens, and
  Efros]{wang2019cnngenerated}
Sheng-Yu Wang, Oliver Wang, Richard Zhang, Andrew Owens, and Alexei~A Efros.
\newblock Cnn-generated images are surprisingly easy to spot...for now.
\newblock In \emph{CVPR}, 2020.

\bibitem[Wang et~al.(2018)Wang, Liu, Zhu, Tao, Kautz, and
  Catanzaro]{wang2018pix2pixHD}
Ting-Chun Wang, Ming-Yu Liu, Jun-Yan Zhu, Andrew Tao, Jan Kautz, and Bryan
  Catanzaro.
\newblock High-resolution image synthesis and semantic manipulation with
  conditional gans.
\newblock In \emph{IEEE Conference on Computer Vision and Pattern Recognition
  (CVPR)}, 2018.

\bibitem[Wu et~al.(2021)Wu, Lischinski, and Shechtman]{wu2021stylespace}
Zongze Wu, Dani Lischinski, and Eli Shechtman.
\newblock Stylespace analysis: Disentangled controls for stylegan image
  generation.
\newblock In \emph{Proceedings of the IEEE/CVF Conference on Computer Vision
  and Pattern Recognition}, 2021.

\bibitem[Xian et~al.(2018)Xian, Sangkloy, Agrawal, Raj, Lu, Fang, Yu, and
  Hays]{xian2017texturegan}
Wenqi Xian, Patsorn Sangkloy, Varun Agrawal, Amit Raj, Jingwan Lu, Chen Fang,
  Fisher Yu, and James Hays.
\newblock Texturegan: Controlling deep image synthesis with texture patches.
\newblock In \emph{IEEE Conference on Computer Vision and Pattern Recognition
  (CVPR)}, 2018.

\bibitem[Xie et~al.(2016)Xie, Lu, Zhu, and Wu]{xie2016theory}
Jianwen Xie, Yang Lu, Song-Chun Zhu, and Yingnian Wu.
\newblock A theory of generative convnet.
\newblock In \emph{International Conference on Machine Learning}, pp.\
  2635--2644. PMLR, 2016.

\bibitem[Xie et~al.(2018)Xie, Lu, Gao, and Wu]{xie2018cooperative}
Jianwen Xie, Yang Lu, Ruiqi Gao, and Ying~Nian Wu.
\newblock Cooperative learning of energy-based model and latent variable model
  via mcmc teaching.
\newblock In \emph{Proceedings of the AAAI Conference on Artificial
  Intelligence}, volume~32, 2018.

\bibitem[Xu et~al.(2021)Xu, Shen, Zhu, Yang, and Zhou]{xu2021generative}
Yinghao Xu, Yujun Shen, Jiapeng Zhu, Ceyuan Yang, and Bolei Zhou.
\newblock Generative hierarchical features from synthesizing images.
\newblock In \emph{Proceedings of the IEEE/CVF Conference on Computer Vision
  and Pattern Recognition}, pp.\  4432--4442, 2021.

\bibitem[Yang et~al.(2017)Yang, Lu, Lin, Shechtman, Wang, and Li]{yang2017high}
Chao Yang, Xin Lu, Zhe Lin, Eli Shechtman, Oliver Wang, and Hao Li.
\newblock High-resolution image inpainting using multi-scale neural patch
  synthesis.
\newblock In \emph{IEEE Conference on Computer Vision and Pattern Recognition
  (CVPR)}, 2017.

\bibitem[Zhang et~al.(2018)Zhang, Dana, Shi, Zhang, Wang, Tyagi, and
  Agrawal]{zhang2018context}
Hang Zhang, Kristin Dana, Jianping Shi, Zhongyue Zhang, Xiaogang Wang, Ambrish
  Tyagi, and Amit Agrawal.
\newblock Context encoding for semantic segmentation.
\newblock In \emph{Proceedings of the IEEE conference on Computer Vision and
  Pattern Recognition}, pp.\  7151--7160, 2018.

\bibitem[Zhang et~al.(2016)Zhang, Isola, and Efros]{zhang2016colorful}
Richard Zhang, Phillip Isola, and Alexei~A Efros.
\newblock Colorful image colorization.
\newblock In \emph{European Conference on Computer Vision (ECCV)}, 2016.

\bibitem[Zhang et~al.(2017)Zhang, Zhu, Isola, Geng, Lin, Yu, and
  Efros]{zhang2017real}
Richard Zhang, Jun-Yan Zhu, Phillip Isola, Xinyang Geng, Angela~S Lin, Tianhe
  Yu, and Alexei~A Efros.
\newblock Real-time user-guided image colorization with learned deep priors.
\newblock \emph{ACM Transactions on Graphics (TOG)}, 9\penalty0 (4), 2017.

\bibitem[Zhou et~al.(2018)Zhou, Zhu, Bai, Lischinski, Cohen-Or, and
  Huang]{TexSyn18}
Yang Zhou, Zhen Zhu, Xiang Bai, Dani Lischinski, Daniel Cohen-Or, and Hui
  Huang.
\newblock Non-stationary texture synthesis by adversarial expansion.
\newblock \emph{ACM Transactions on Graphics (TOG)}, 37\penalty0 (4), 2018.

\bibitem[Zhu et~al.(2020{\natexlab{a}})Zhu, Shen, Zhao, and
  Zhou]{zhu2020domain}
Jiapeng Zhu, Yujun Shen, Deli Zhao, and Bolei Zhou.
\newblock In-domain gan inversion for real image editing.
\newblock In \emph{European Conference on Computer Vision (ECCV)},
  2020{\natexlab{a}}.

\bibitem[Zhu et~al.(2016)Zhu, Kr{\"a}henb{\"u}hl, Shechtman, and
  Efros]{zhu2016generative}
Jun-Yan Zhu, Philipp Kr{\"a}henb{\"u}hl, Eli Shechtman, and Alexei~A Efros.
\newblock Generative visual manipulation on the natural image manifold.
\newblock In \emph{European Conference on Computer Vision (ECCV)}, 2016.

\bibitem[Zhu et~al.(2017)Zhu, Park, Isola, and Efros]{zhu2017unpaired}
Jun-Yan Zhu, Taesung Park, Phillip Isola, and Alexei~A Efros.
\newblock Unpaired image-to-image translation using cycle-consistent
  adversarial networks.
\newblock In \emph{IEEE International Conference on Computer Vision (ICCV)},
  2017.

\bibitem[Zhu et~al.(2020{\natexlab{b}})Zhu, Abdal, Qin, and Wonka]{zhu2020sean}
Peihao Zhu, Rameen Abdal, Yipeng Qin, and Peter Wonka.
\newblock Sean: Image synthesis with semantic region-adaptive normalization.
\newblock In \emph{IEEE Conference on Computer Vision and Pattern Recognition
  (CVPR)}, 2020{\natexlab{b}}.

\end{thebibliography}
}

\clearpage
\appendix

\section{Proofs}
\label{app:sec:proofs}
\propbound*
\begin{proof}
Denote $\rvx^{(g)}(0) = \mathrm{SDEdit}(\rvx^{(g)}; t, \theta)$, then
\begin{align}
\norm{\rvx^{(g)}(t_0) - \rvx^{(g)}(0)}_2^2 &= \norm{\int_{t_0}^{0} \frac{\mathrm{d} \rvx^{(g)}(t)}{\mathrm{d} t} \mathrm{d} t}_2^2 \\
&= \norm{\int_{t_0}^{0} \left[ - \frac{\mathrm{d} [\sigma^2(t)]}{\mathrm{d} t} s_\theta(\rvx, t; \theta)\right]  \mathrm{d} t + \sqrt{\frac{\mathrm{d} [\sigma^2(t)]}{\mathrm{d} t}}  \mathrm{d} \bar{\mathbf{w}}}_2^2\\
 & \leq \norm{\int_{t_0}^{0} \left[ - \frac{\mathrm{d} [\sigma^2(t)]}{\mathrm{d} t} s_\theta(\rvx, t; \theta)\right]  \mathrm{d} t}_2^2 + \norm{\int_{t_0}^{0} \sqrt{\frac{\mathrm{d} [\sigma^2(t)]}{\mathrm{d} t}}  \mathrm{d} \bar{\mathbf{w}}}_2^2
\end{align}
From the assumption over $s_\theta(\rvx, t; \theta)$, the first term is not greater than $$C \norm{ \int_{t_0}^{0} \left[- \frac{\mathrm{d} [\sigma^2(t)]}{\mathrm{d} t} \right] \mathrm{d} t }_2^2  = C \sigma^4(t_0),$$
where equality could only happen when each score output has a squared $L_2$ norm of $C$ and they are linearly dependent to one other.
The second term is independent to the first term as it only concerns random noise; this is equal to the squared $L_2$ norm of a random variable from a Wiener process at time $t = 0$, with marginal distribution being $\epsilon \sim \gN(\mathbf{0}, \sigma^2(t_0)\mathbf{I})$ (this marginal does not depend on the discretization steps in Euler-Maruyama).  
The squared $L_2$ norm of $\epsilon$ divided by $\sigma^2(t_0)$ is a $\chi^2$-distribution with $d$-degrees of freedom. From \citet{laurent2000adaptive}, Lemma 1, we have the following one-sided tail bound: 
\begin{align}
    \Pr(\norm{\epsilon}_2^2 / \sigma^2(t_0) \geq d + 2 \sqrt{d \cdot -\log \delta} - 2 \log \delta) \leq \exp(\log \delta) = \delta.
\end{align}
Therefore, with probability at least $(1 - \delta)$, we have that:
\begin{align}
    \norm{\rvx^{(g)}(t_0) - \rvx^{(g)}(0)}_2^2 \leq \sigma^2(t_0) (C \sigma^2(t_0) + d + 2\sqrt{- d \cdot \log \delta} - 2 \log \delta),
\end{align}
completing the proof.
\end{proof}

\section{Extra ablation studies}

\label{app:extra}
In this section, we perform extra ablation studies and analysis for \model{}.

\subsection{Analysis on the quality of user guide}

As discussed in \cref{sec:method}, if the guide is far from any realistic images (\eg, random noise or has an unreasonable composition) , then we must tolerate at least a certain level of deviation from the guide (non-faithfulness) in order to produce a realistic image.

For practical applications, we perform extra ablation studies on how the quality of guided stroke would affect the results in \cref{fig:new_weird_input}, \cref{fig:new_input_ablation} and \cref{app:tab:input_ablation}. Specifically, in \cref{fig:new_weird_input} we consider stroke input of 
1) a human face with limited detail for a CelebA-HQ model, 2) a human face with spikes for a CelebA-HQ model, 3) a building with limited detail for a LSUN-church model, 4) a horse for a LSUN-church model.
We observe that SDEdit is in general tolerant to different kinds of user inputs.
In \cref{app:tab:input_ablation}, we quantitatively analyze the effect of user guide quality using simulated stroke paintings as input. Described in \cref{app:sec:generating_stroke}, the human-stroke-simulation algorithm uses different numbers of colors to generate stroke guides with different levels of detail. We compare \model with baselines qualitatively in \cref{fig:new_input_ablation} and quantitatively in \cref{app:tab:input_ablation}.
Similarly, we observe 
that \model has a high tolerance to input guides and consistently outperforms the baselines across all setups in this experiment.

\begin{figure}[H]
    \centering
    \includegraphics[width=0.8\textwidth]{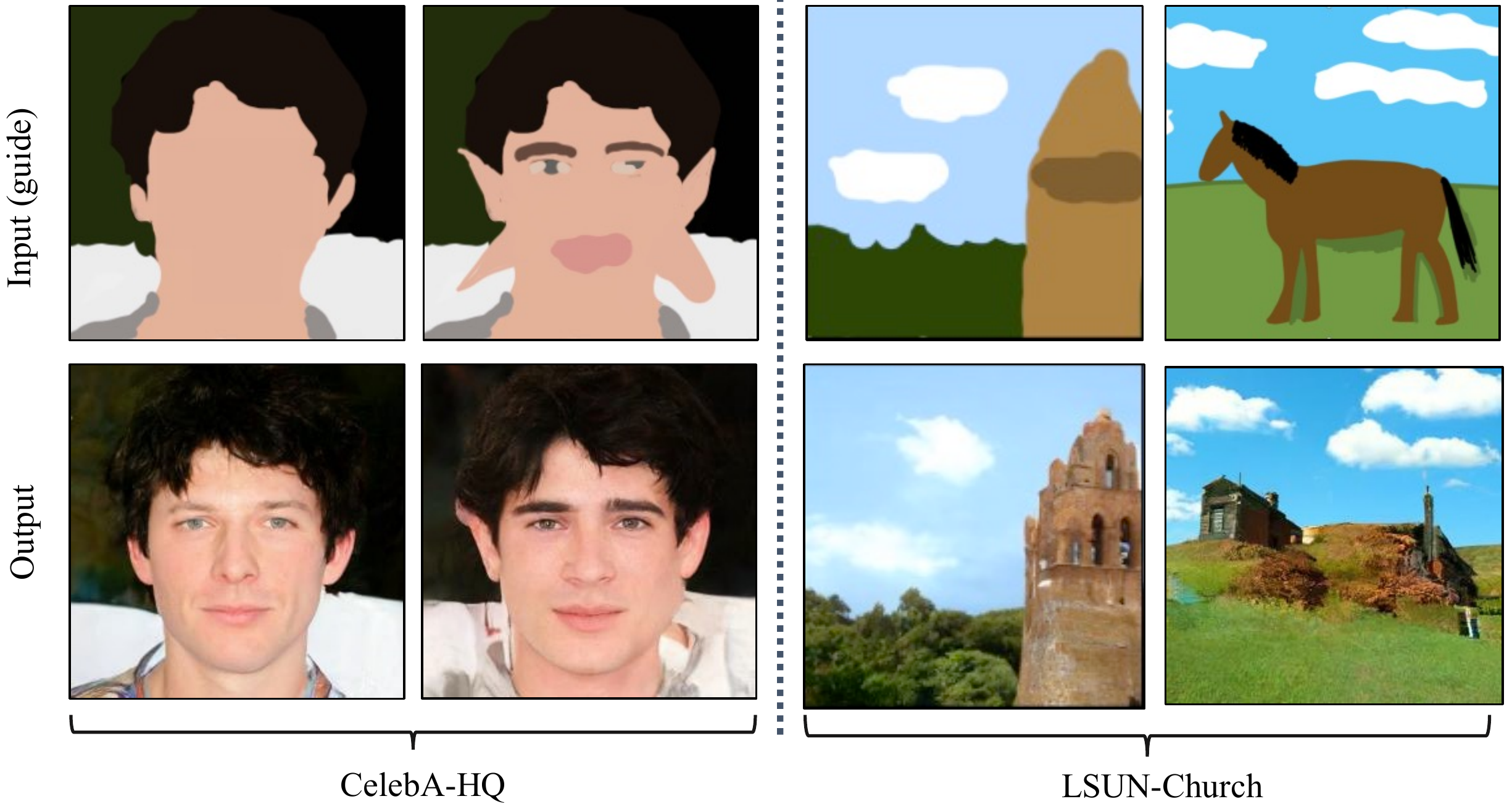}
    \caption{Analysis on the quality of user guide for stoke-based image synthesis. We observe that \model{} is in general tolerant to different kinds of user inputs.}
    \label{fig:new_weird_input}
\end{figure}

\begin{figure}[H]
    \centering
    \includegraphics[width=0.8\textwidth]{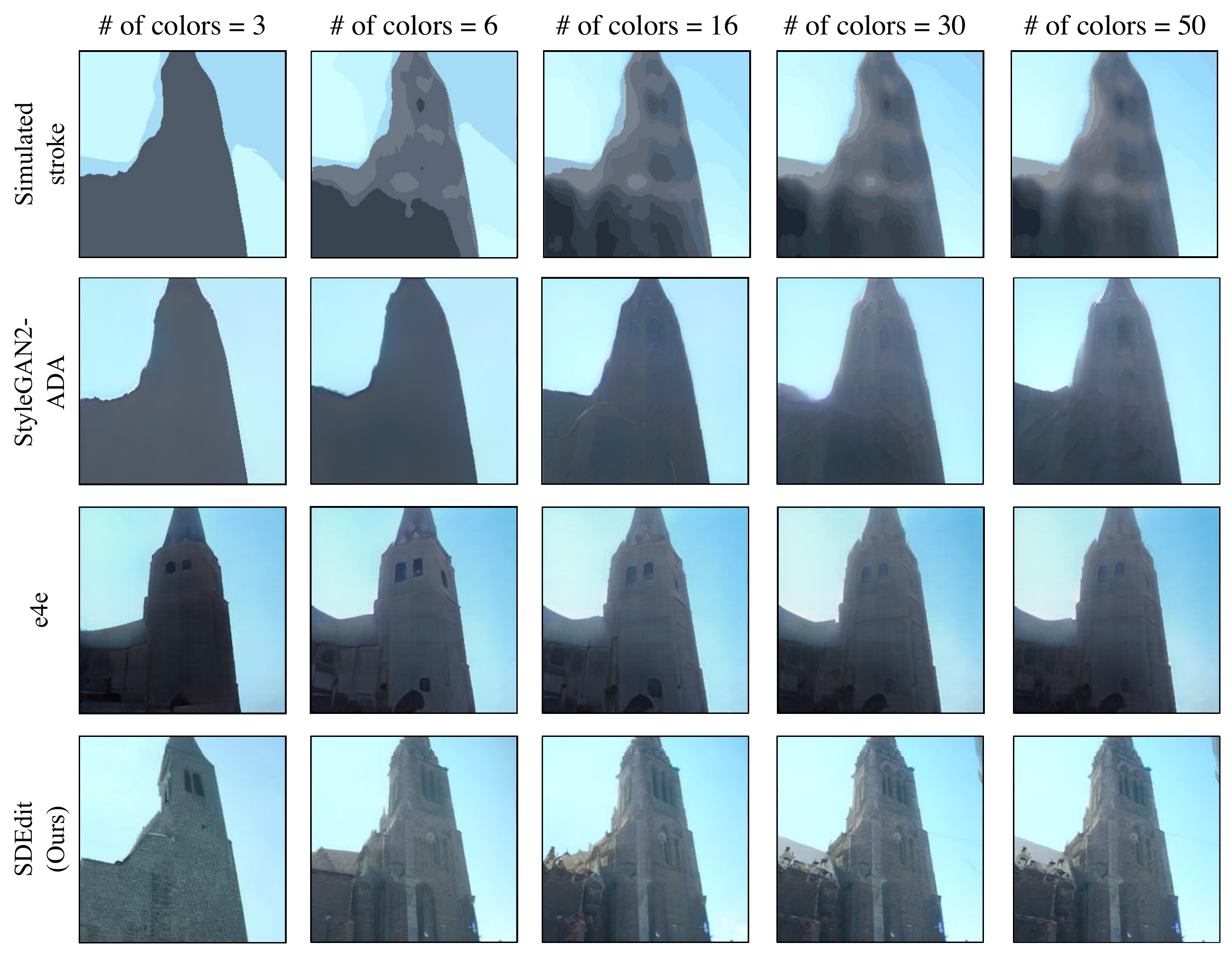}
    \caption{Analysis on the quality of user guide for stoke-based image synthesis. We observe that \model{} is in general tolerant to different kinds of user inputs.}
    \label{fig:new_input_ablation}
\end{figure}

\begin{figure}
 \newcommand{\mround}[1]{\round{#1}{2}}
    \newcommand{\nastar}{\multicolumn{1}{c}{\hspace{15pt}--~*}}
    \newcommand{\na}{\multicolumn{1}{c}{--}}
    \centering %
    {
        \centering %
        \begin{adjustbox}{max width=0.99\linewidth}
\begin{tabular}{@{}ccccccc@{}}
\toprule
\multirow{2}{*}{\# of colors} & \multicolumn{2}{c}{StyleGAN2-ADA} & \multicolumn{2}{c}{e4e} & \multicolumn{2}{c}{SDEdit (Ours)} \\ \cmidrule(l){2-7} 
& KID $\downarrow$  & $L_2$ $\downarrow$  & KID $\downarrow$  & $L_2$ $\downarrow$       & KID $\downarrow$  & $L_2$ $\downarrow$             \\ \midrule
3                             & 0.1588           & 67.22          & 0.0379      & 70.73     & \textbf{0.0233}           & \textbf{36.00}          \\
6                             & 0.1544           & 72.41          & 0.0354      & 68.53     & \textbf{0.0156}           & \textbf{37.67}          \\
16                            & 0.0923           & 69.52          & 0.0319      & 68.20     & \textbf{0.0135}           & \textbf{37.70}          \\
30                            & 0.0911           & 67.11          & 0.0304      & 68.66     & \textbf{0.0128}           & \textbf{37.42}          \\
50                            & 0.0922           & 65.28          & 0.0307      & 68.80     & \textbf{0.0126}           & \textbf{37.40}          \\ \bottomrule
\end{tabular}
        \end{adjustbox}
}
\captionof{table}{
We compare \model with baselines quantitatively on LSUN-church dataset on stroke-based generation. 
``\# of colors" denotes the number of colors used to generate the synthetic stroke paintings, with fewer colors corresponding to a less accurate and less detailed input guide (see \cref{fig:new_input_ablation}). We observe that \model consistently achieves more realistic and more faithful outputs and outperforms the baselines across all setups.
}
\label{app:tab:input_ablation}
\end{figure}

\FloatBarrier
\subsection{Flexible image editing with \model{}}

In this section, we perform extra image editing experiments including editing closing eyes~\cref{fig:new_flexible_editing_closing}, opening mouth, and changing lip color~\cref{fig:new_flexible_editing}. We observe that \model can still achieve reasonable editing results, which shows that \model{} is capable of flexible image editing tasks.

\begin{figure}[H]
    \centering
    \includegraphics[width=0.56\textwidth]{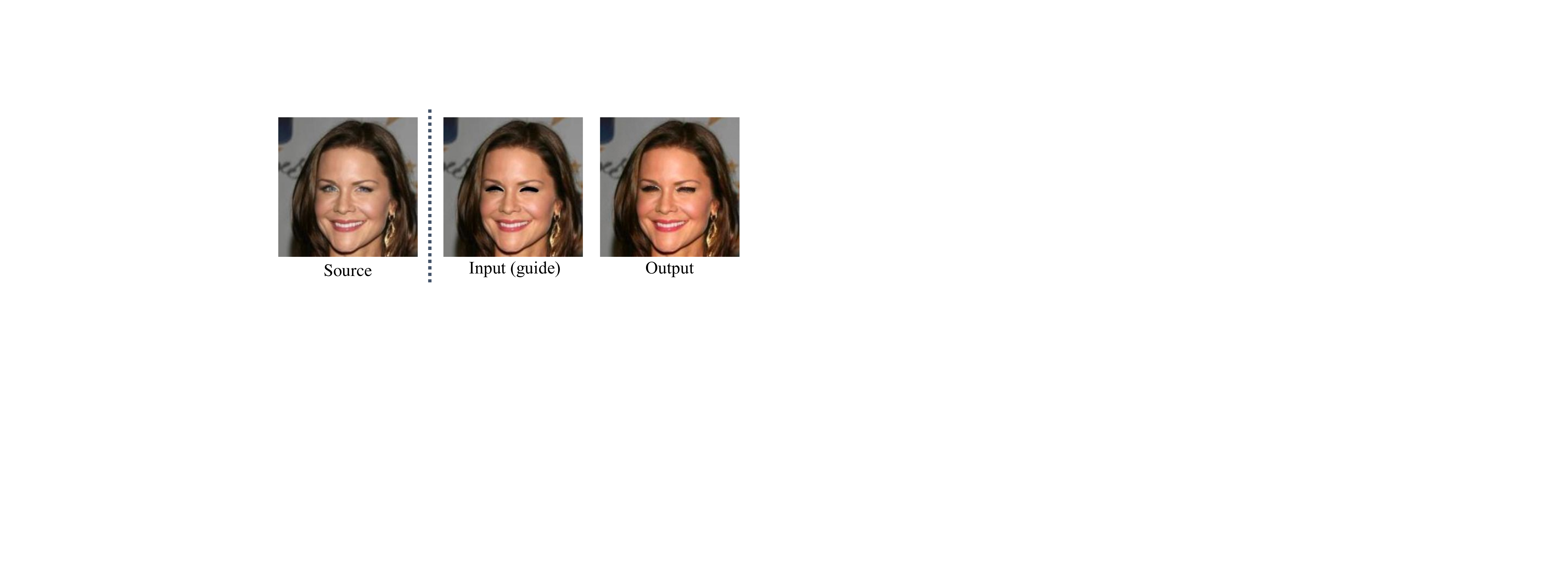}
    \caption{Flexible image editing on closing eyes with \model{}.}
    \label{fig:new_flexible_editing_closing}
\end{figure}

\begin{figure}[H]
    \centering
    \includegraphics[width=0.7\textwidth]{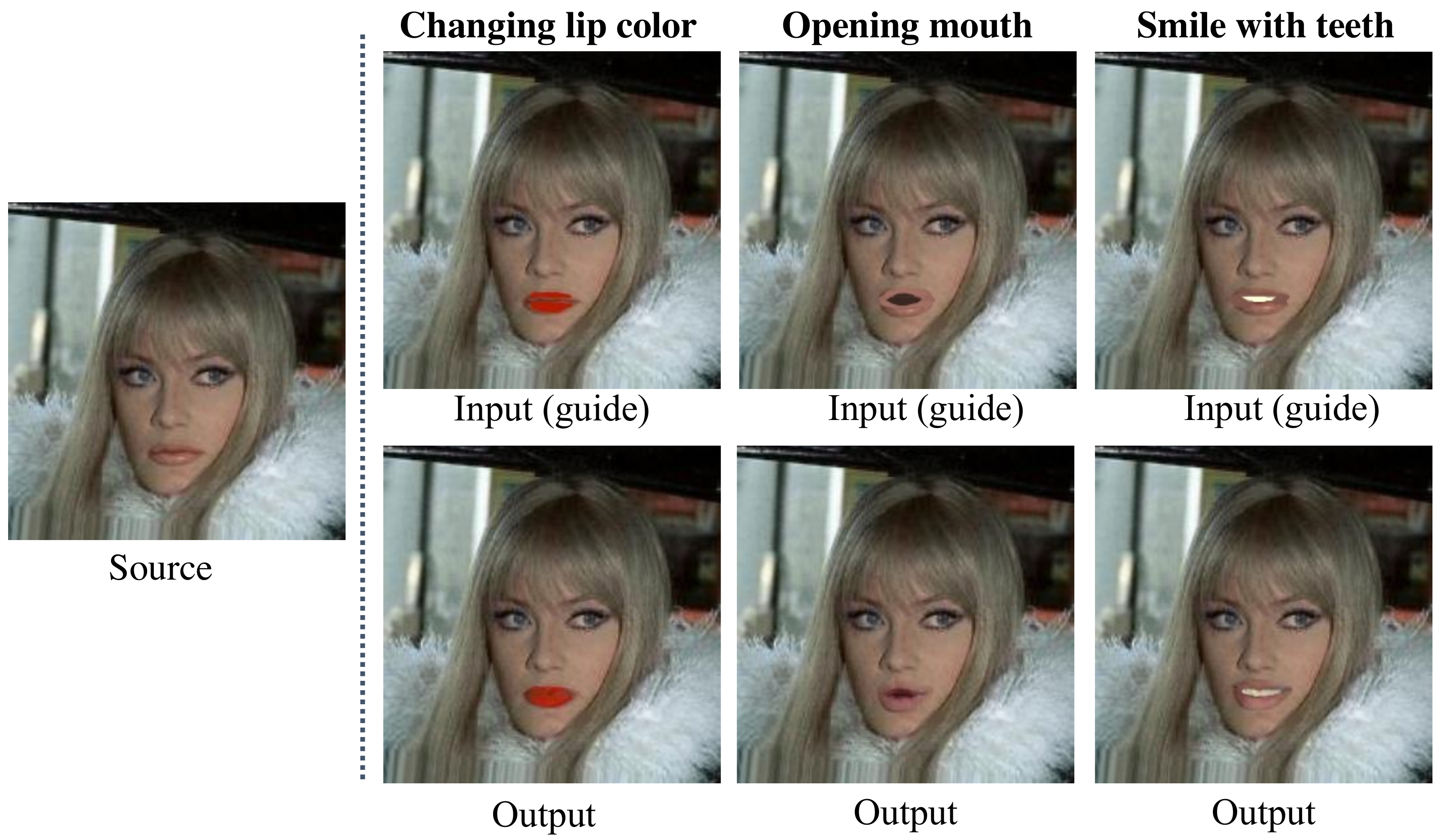}
    \caption{Flexible image editing on mouth with \model{}.}
    \label{fig:new_flexible_editing}
\end{figure}

\subsection{Analysis on $t_0$}
In this section, we provide extra analysis on the effect of $t_0$ (see \cref{fig:new_fig1_analysis}). 
As illustrated in \cref{fig:sde_tradeoff}, we can tune $t_0$ to tradeoff between faithfulness and realism---with a smaller $t_0$ corresponding to a more faithful but less realistic generated image. If we want to keep the brown stroke in \cref{fig:new_fig1_analysis}, we can reduce $t_0$ to increase its faithfulness which could potentially decrease its realism. Additional analysis can be found in \cref{app:sec:real-faith-trade-off}.

\begin{figure}[H]
    \centering
    \includegraphics[width=\textwidth]{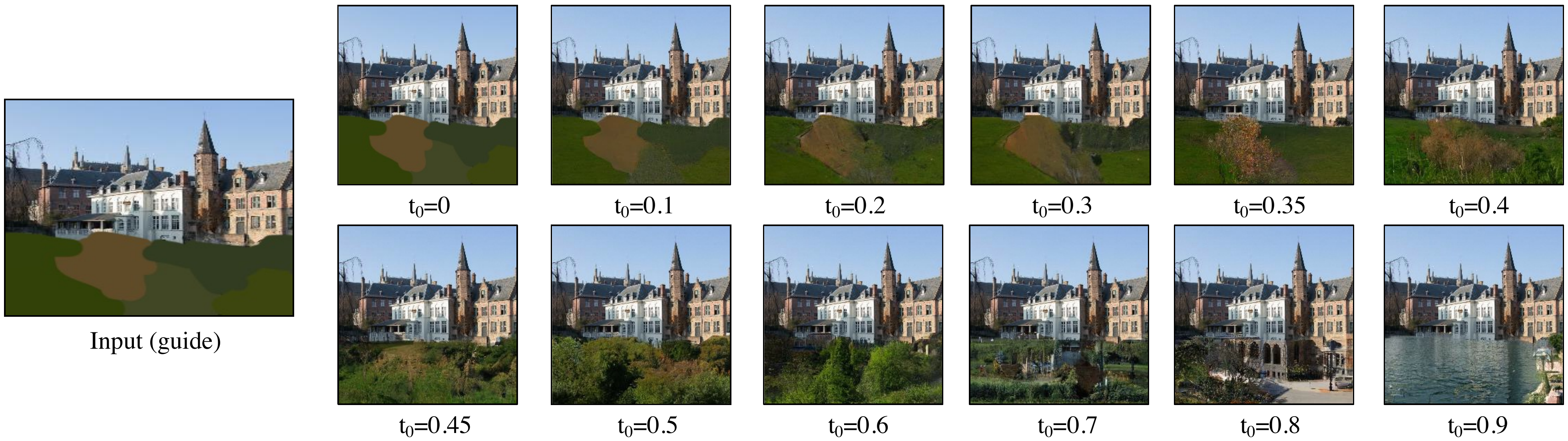}
    \caption{Extra analysis on $t_0$. As $t_0$ increases, the generated images become {{more realistic}}
while {{less faithful}}.}
    \label{fig:new_fig1_analysis}
\end{figure}

\subsection{Extra comparison with other baselines}

We perform extra comparison with SC-FEGAN~\citep{jo2019sc} in \cref{fig:new_baseline}. We observe that \model{} is able to have more realistic results than SC-FEGAN~~\citep{jo2019sc} when using {the same} stroke input guide. We also present results for SC-FEGAN~\citep{jo2019sc} where we use extra sketch together with stroke as the input guide (see \cref{fig:new_baseline1}). We observe that \model{} is still able to outperform SC-FEGAN in terms of realism even when SC-FEGAN is using both sketch and stroke as the input guide. 

\begin{figure}[H]
    \centering
    \includegraphics[width=0.8\textwidth]{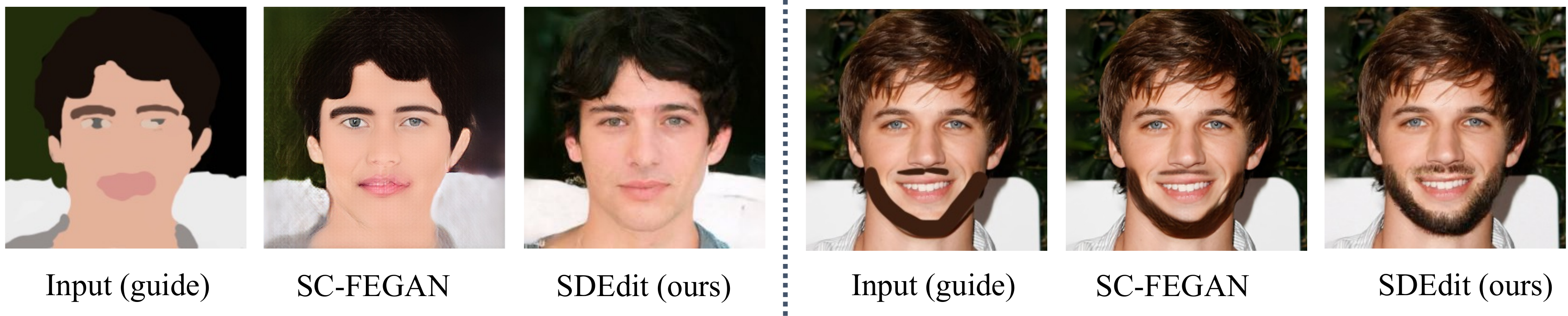}
    \caption{Comparison with SC-FEGAN~\citep{jo2019sc} on stroke-based image synthesis and editing. We observe that \model{} is able to generate more realistic results than SC-FEGAN.}
    \label{fig:new_baseline}
\end{figure}

\begin{figure}[H]
    \centering
    \includegraphics[width=0.47\textwidth]{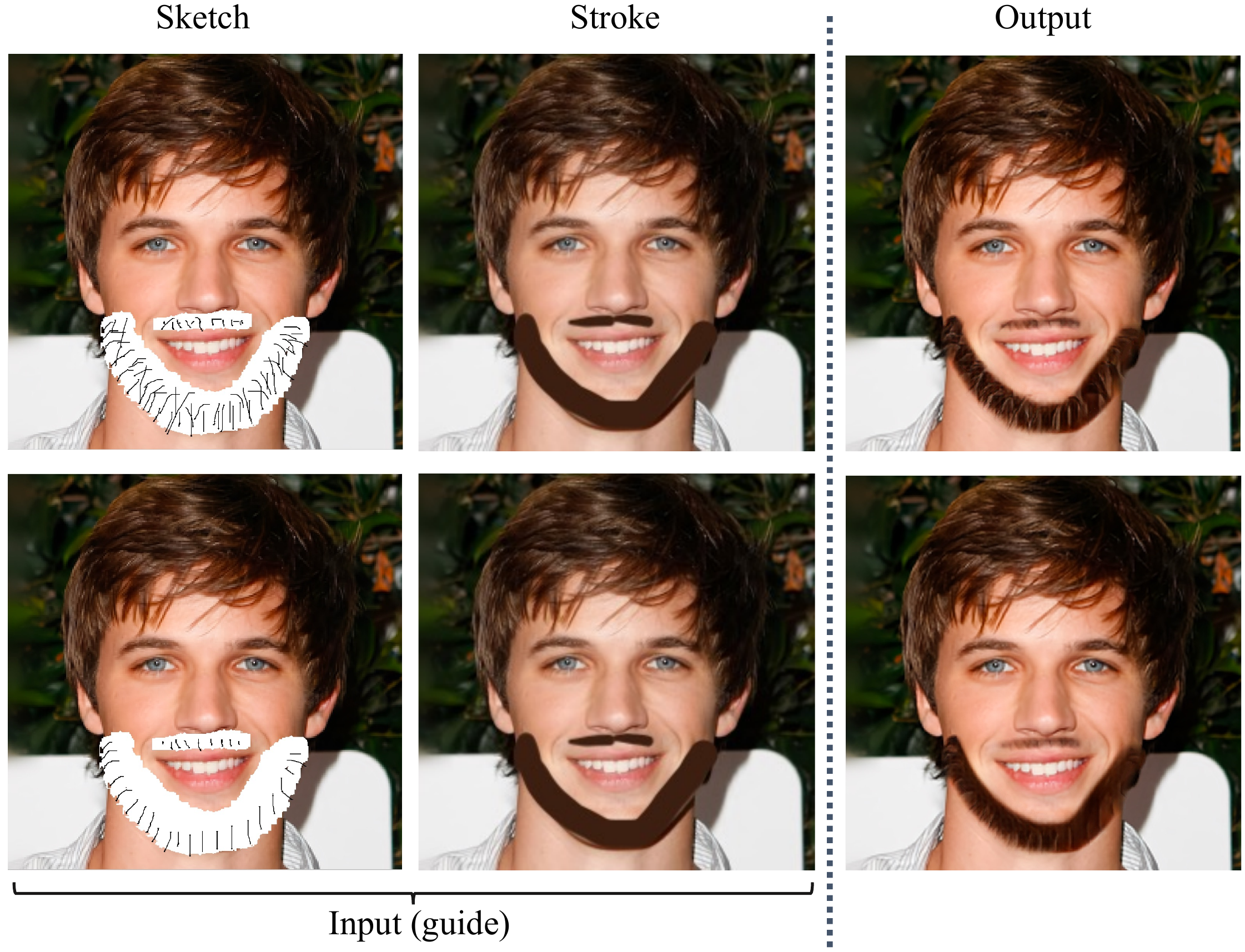}
    \caption{Stroke-based editing for SC-FEGAN~\citep{jo2019sc} using both stroke and extra sketch as the input guide. 
    We observe that \model{} still outperforms SC-FEGAN using only stroke as the input guide.}
    \label{fig:new_baseline1}
\end{figure}

\subsection{Comparison with \cite{song2021scorebased}}
Methods proposed by \cite{song2021scorebased} introduce an extra noise-conditioned classifier for conditional generation and the performance of the classifier is critical to the conditional generation performance. Their settings are more similar to regular inverse problems where the measurement function is known, which is discussed in \cref{sec:method}. Since we do not have a known ``measurement" function for user-generated guides, their approach cannot be directly applied to user-guided image synthesis or editing in the form of manipulating pixel RGB values. To deal with this limitation, \model{} initializes the reverse SDE based on user input and modifies $t_0$ accordingly---an approach different from \cite{song2021scorebased} (which always have the same initialization). This technique allows \model{} to achieve faithful and realistic image editing or generation results without extra task-specific model learning (\eg, an additional classifier in \cite{song2021scorebased}).

For practical applications, we compare with  \cite{song2021scorebased} on stroke-based image synthesis and editing where
we do not learn an extra noise-conditioned classifier (see \cref{fig:new_baseline_ncsn}). In fact, we are also unable to learn the noise-conditioned classifier since we do not have a known ``measurement" function for user-generated guides and we only have one random user input guide instead of a dataset of input guide.
We observe that this application of \cite{song2021scorebased} fails to generate faithful results by performing random inpainting (see \cref{fig:new_baseline_ncsn}). \model{}, on the other hand, generates both realistic and faithful images without learning extra task-specific models (\eg, an additional classifier)
and can be directly applied to pretrained SDE-based generative models, allowing for guided image synthesis and editing using SDE-based models.
We believe this shows the novelty and contribution of \model{}.

\begin{figure}[H]
    \centering
    \includegraphics[width=0.8\textwidth]{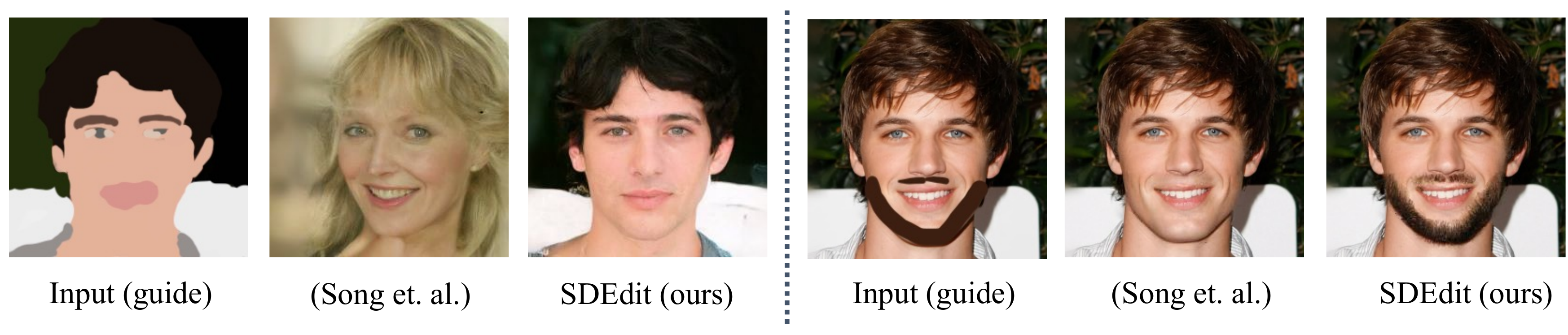}
    \caption{Comparison with \cite{song2021scorebased} on stroke-based image synthesis and editing. We observe that \model{} is able to generate more faithful results than \cite{song2021scorebased} without training an extra task-specific model (\eg, an additional classifier).}
    \label{fig:new_baseline_ncsn}
\end{figure}

\section{Details on \model}
\label{app:sdedit}
\subsection{Details on the VP and VE SDEs}
We follow the definitions of VE and VP SDEs in \cite{song2021scorebased}, and adopt the same settings therein. 
\paragraph{VE-SDE}
In particular, for the VE SDE, we choose
\begin{align*}
    \sigma(t) = \begin{cases}
    0, &\quad t = 0\\
    \sigma_\text{min} \Big(\frac{\sigma_\text{max}}{\sigma_\text{min}} \Big)^t, &\quad t > 0
    \end{cases}
\end{align*}
where $\sigma_\text{min} = 0.01$ and $\sigma_\text{max} = 380$, $378$, $348$, $1348$ for LSUN churches, bedroom, FFHQ/CelebA-HQ $256\times 256$, and FFHQ $1024\times 1024$ datasets respectively.

\paragraph{VP-SDE}
For the VP SDE, it takes the form of
\begin{align}
    \ud \rvx(t) = -\frac{1}{2}\beta(t)\rvx(t) \ud t + \sqrt{\beta(t)} \ud \rvw(t),
\end{align}
where $\beta(t)$ is a positive function. In experiments, we follow \cite{song2021scorebased,ho2020denoising, dhariwal2021diffusion} and set
\begin{align*}
    \beta(t) = \beta_\text{min} + t(\beta_\text{max} - \beta_\text{min}),
\end{align*}
For SDE trained by \cite{song2021scorebased,ho2020denoising} we use 
$\beta_\text{min} = 0.1$ and $\beta_\text{max} = 20$; for SDE trained by \cite{dhariwal2021diffusion}, the model learns to rescale the variance based on the same choices of $\beta_\text{min}$ and $\beta_\text{max}$.
We always have $p_1(\rvx) \approx \mcal{N}(\bm{0}, \bm{I})$ under these settings.

Solving the reverse VP SDE is similar to solving the reverse VE SDE. Specifically, we follow the iteration rule below:
\begin{equation}
\rvx_{n-1} = \frac{1}{\sqrt{1 - \beta(t_n)\Delta t}}(\rvx_n + \beta(t_{n})\Delta t \vs_\vtheta(\rvx(t_n), t_n))\\
+ \sqrt{\beta(t_n)\Delta t}~\rvz_n,
\end{equation}
where $\rvx_N \sim \mcal{N}(\bm{0}, \bm{I})$, $\rvz_n \sim \mcal{N}(\bm{0}, \bm{I})$ and $ n= N, N-1, \cdots, 1$.

\subsection{Details on Stochastic Differential Editing}
\label{app:sec:algorithm}
In generation the process detailed in \cref{alg:synthesis_ve} can also be repeated  for
$K$ number of times as detailed in \cref{alg:app:synthesis_ve}. Note that 
\cref{alg:synthesis_ve} is a special case of \cref{alg:app:synthesis_ve}:
when $K=1$, we recover \cref{alg:synthesis_ve}.
For VE-SDE, \cref{alg:app:synthesis_ve}
converts a stroke painting to a photo-realistic image, which typically modifies all pixels of the input. However, in cases such as image compositing and stroke-based editing, certain regions of the input are already photo-realistic and therefore we hope to leave these regions intact. To represent a specific region, we use a binary mask $\bm{\Omega} \in \{0, 1\}^{C\times H \times W}$ that evaluates to $1$ for editable pixels and $0$ otherwise. We can generalize \cref{alg:app:synthesis_ve} to restrict editing in the region defined by $\bm{\Omega}$.

For editable regions, we perturb the input image with the forward SDE and generate edits by reversing the SDE, using the same procedure in \cref{alg:app:synthesis_ve}. For uneditable regions, we perturb it as usual but design the reverse procedure carefully so that it is guaranteed to recover the input. Specifically, suppose $\x \in \mbb{R}^{C\times H \times W}$ is an input image of height $H$, width $W$, and with $C$ channels. Our algorithm first perturbs $\x(0) = \x$ with an SDE running from $t=0$ till $t = t_0$ to obtain $\x(t_0)$.
Afterwards, we denoise $\x(t_0)$ with separate methods for $\bm{\Omega} \odot \rvx(t)$ and $(\bm{1} - \bm{\Omega}) \odot \rvx(t)$, where $\odot$ denotes the element-wise product and $0 \leq t \leq t_0$. For $\bm{\Omega} \odot \rvx(t)$, we simulate the reverse SDE~\citep{song2021scorebased} and project the results by element-wise multiplication with $\bm{\Omega}$. For $(\bm{1} - \bm{\Omega}) \odot \rvx(t)$, we set it to $(\bm{1} - \bm{\Omega})\odot (\x + \sigma(t) \z)$, where $\z \sim \mcal{N}(\bm{0}, \bm{I})$. Here we gradually reduce the noise magnitude according to $\sigma(t)$ to make sure $\bm{\Omega} \odot \x(t)$ and $(\bm{1} - \bm{\Omega})\odot \x(t)$ have comparable amount of noise. Moreover, since $\sigma(t) \to 0$ as $t\to 0$, this ensures that $(\bm{1} - \bm{\Omega}) \odot \x(t)$ converges to $(\bm{1} -\bm{\Omega})\odot \x$, keeping the uneditable part of $\x$ intact.
The complete SDEdit method (for VE-SDEs) is given in \cref{alg:sdeit}. We provide algorithm for VP-SDEs in \cref{app:alg:vp}
and the corresponding 
masked version in \cref{app:alg:vp_mask}.

With different inputs to \cref{alg:sdeit} or \cref{app:alg:vp_mask}, we can perform multiple image synthesis and editing tasks with a single unified approach, including but not limited to the following:
\begin{itemize}
    \item \textbf{Stroke-based image synthesis:} We can recover \cref{alg:app:synthesis_ve} or \cref{app:alg:vp} by setting all entries in $\bm{\Omega}$ to 1.
    \item \textbf{Stroke-based image editing:} Suppose $\x^{(g)}$ is an image marked by strokes, and $\bm{\Omega}$ masks the part that are not stroke pixels. We can reconcile the two parts of $\x^{(g)}$ with \cref{alg:sdeit} to obtain a photo-realistic image.
    
    \item \textbf{Image compositing:} Suppose $\x^{(g)}$ is an image superimposed by elements from two images, and $\bm{\Omega}$ masks the region that the users do not want to perform editing, we can perform image compositing with \cref{alg:sdeit} or \cref{app:alg:vp_mask}.
\end{itemize}

\begin{algorithm}[H]
  \caption{Guided image synthesis and editing (VE-SDE)}\label{alg:app:synthesis_ve}
  \begin{algorithmic}
  \Require{$\x^{(g)}$ (guide), $t_0$ (SDE hyper-parameter), $N$ (total denoising steps), $K$ (total repeats)
  }
  \State{$\Delta t \gets \frac{t_0}{N}$}
  \For{$k \gets 1$ \textbf{to} $K$}
    \State{$\rvz \sim \mcal{N}(\bm{0}, \bm{I})$}
    \State{$\rvx \gets \rvx + \sigma(t_0) \rvz$}
    \For{$n \gets N$ \textbf{to} $1$}
        \State{$t \gets t_0\frac{n}{N}$}
        \State{$\rvz \sim \mcal{N}(\bm{0}, \bm{I})$} 
        \State{$\epsilon \gets \sqrt{\sigma^2(t) - \sigma^2(t - \Delta t)}$}
        \State{$\rvx \gets \rvx + \epsilon^2 \vs_\vtheta(\rvx, t) + \epsilon \rvz$}
    \EndFor
  \EndFor
  \State{\textbf{Return} $\rvx$}
  \end{algorithmic}
\end{algorithm}

\algrenewcommand\algorithmicindent{0.7em}%
\begin{algorithm}[H]
  \caption{Guided image synthesis and editing with mask (VE-SDE) 
  }\label{alg:sdeit}
  \begin{algorithmic}
  \Require{$\x^{(g)}$ (guide), %
  $\bm{\Omega}$ (mask for edited regions), $t_0$ (SDE hyper-parameter), $N$ (total denoising steps), $K$ (total repeats)
  }
  \State{$\Delta t \gets \frac{t_0}{N}$}
  \State{$\rvx_0 \gets \rvx$}
  \For{$k \gets 1$ \textbf{to} $K$}
    \State{$\rvz \sim \mcal{N}(\bm{0}, \bm{I})$}
    \State{$\rvx \gets (\bm{1}-\bm{\Omega}) \odot \rvx_0 + \bm{\Omega} \odot \rvx + \sigma(t_0) \rvz$}
    \For{$n \gets N$ \textbf{to} $1$}
        \State{$t \gets t_0\frac{n}{N}$}
        \State{$\rvz \sim \mcal{N}(\bm{0}, \bm{I})$} 
        \State{$\epsilon \gets \sqrt{\sigma^2(t) - \sigma^2(t - \Delta t)}$}
        \State{$\rvx \gets (\bm{1}-\bm{\Omega}) \odot (\rvx_0 + \sigma(t) \rvz) + \bm{\Omega} \odot (\rvx + \epsilon^2 \vs_\vtheta(\rvx, t) + \epsilon \rvz)$}
    \EndFor
  \EndFor
  \State{\textbf{Return} $\rvx$}
  \end{algorithmic}
\end{algorithm}

\algrenewcommand\algorithmicindent{0.7em}%
\begin{algorithm}[H]
  \caption{Guided image synthesis and editing (VP-SDE)}\label{alg:synthesis_vp}
  \begin{algorithmic}
   \Require{$\x^{(g)}$ (guide), $t_0$ (SDE hyper-parameter), $N$ (total denoising steps), $K$ (total repeats)
  }
  \State{$\Delta t \gets \frac{t_0}{N}$}
  \State{$\alpha(t_0) \gets \prod_{n=1}^{N} (1-\beta(\frac{nt_0}{N})\Delta t)$}
  \For{$k \gets 1$ \textbf{to} $K$}
    \State{$\rvz \sim \mcal{N}(\bm{0}, \bm{I})$}
    \State{$\rvx \gets \sqrt{\alpha(t_0)}\rvx + \sqrt{1-\alpha(t_0)}\rvz$}
    \For{$n \gets N$ \textbf{to} $1$}
        \State{$t \gets t_0 \frac{n}{N}$}
        \State{$\rvz \sim \mcal{N}(\bm{0}, \bm{I})$}
        \State{$\rvx \gets \frac{1}{\sqrt{1 - \beta(t)\Delta t}}(\rvx + \beta(t)\Delta t \vs_\vtheta(\rvx, t)) + \sqrt{\beta(t)\Delta t}~\rvz$}
    \EndFor
  \EndFor
  \State{\textbf{Return} $\rvx$}
  \end{algorithmic}
  \label{app:alg:vp}
\end{algorithm}
\begin{algorithm}[H]
  \caption{Guided image synthesis and editing with mask (VP-SDE)}\label{alg:editing_vp}
  \begin{algorithmic}
  \Require{$\x^{(g)}$ (guide), 
  $\bm{\Omega}$ (mask for edited regions), $t_0$ (SDE hyper-parameter), $N$ (total denoising steps), $K$ (total repeats)
  }
  \State{$\Delta t \gets \frac{t_0}{N}$}
  \State{$\x_0 \gets \x$}
  \State{$\alpha(t_0) \gets \prod_{i=1}^{N} (1-\beta(\frac{it_0}{N})\Delta t)$}
  \For{$k \gets 1$ \textbf{to} $K$}
    \State{$\rvz \sim \mcal{N}(\bm{0}, \bm{I})$}
    \State{$\rvx \gets  [(\bm{1} - \Omega) \odot \sqrt{\alpha(t_0)} \x_0 + \Omega \odot \sqrt{\alpha(t_0)} \x + \sqrt{1-\alpha(t_0)}\rvz]$}
    \For{$n \gets N$ \textbf{to} $1$}
        \State{$t \gets t_0 \frac{n}{N}$}
        \State{$\rvz \sim \mcal{N}(\bm{0}, \bm{I})$}
        \State{$\alpha(t) \gets \prod_{i=1}^{n} (1-\beta(\frac{it_0}{N})\Delta t)$}
        \State{$\rvx \gets \Big\{(\bm{1} - \Omega) \odot (\sqrt{\alpha(t)}\x_0 + \sqrt{1-\alpha(t)} \rvz) + \Omega \odot \Big[\frac{1}{\sqrt{1 - \beta(t)\Delta t}}( \rvx + \beta(t)\Delta t \vs_\vtheta(\rvx, t)) + \sqrt{\beta(t)\Delta t}~\rvz)\Big]\Big\}$}
    \EndFor
  \EndFor
  \State{\textbf{Return} $\rvx$}
  \end{algorithmic}
  \label{app:alg:vp_mask}
\end{algorithm}

\section{Experimental settings}
\label{app:experiment}
\subsection{Implementation details}
Below, we add additional implementation details for each application. We use publicly available pretrained SDE checkpoints provided by~\citeauthor{song2021scorebased,ho2020denoising,dhariwal2021diffusion}. %
Our code  will be publicly available upon publication. 
\myparagraph{Stroke-based image synthesis.} 
In this experiment, we use $K=1, N=500$, $t_0=0.5$, for SDEdit (VP).
We find that $K=1$ to $3$ work reasonably well, with larger $K$ generating more realistic images but at a higher computational cost.

For StyleGAN2-ADA, in-domain GAN and e4e, we use the official implementation with default parameters to project each input image into the latent space, and subsequently use the obtained latent code to produce stroke-based image samples.

\myparagraph{Stroke-based image editing.}
We use $K=1$ in the experiment for SDEdit (VP). We use $t_0=0.5$, $N=500$ for SDEdit (VP), and $t_0 = 0.45$, $N=1000$ for SDEdit (VE).

\myparagraph{Image compositing.}
We use CelebA-HQ (256$\times$256)~\citep{karras2017progressive} for image compositing experiments.
More specifically, given an image from CelebA-HQ, the user will copy pixel patches from other reference images, and also specify the pixels they want to perform modifications, which will be used as the mask in \cref{alg:sdeit}. In general, the masks are simply the pixels the users have copied pixel patches to.
We focus on editing hairstyles and adding glasses. 
We use an SDEdit model pretrained on FFHQ~\citep{karras2019style}.
We use $t_0=0.35$, $N=700$, $K=1$ for SDEdit (VE).
We present more results in Appendix~\ref{app:image_compositing}.

\subsection{Synthesizing stroke painting}
\label{app:sec:generating_stroke}
\paragraph{Human-stroke-simulation algorithm}
We design a human-stroke-simulation algorithm in order to perform large scale quantitative analysis on stroke-based generation. Given a 256$\times$256 image, we first apply a median filter with kernel size 23 to the image, then reduce the number of colors to 6 using the adaptive palette. We use this algorithm on the validation set of LSUN bedroom and LSUN church outdoor, and subset of randomly selected 6000 images in the CelebA (256$\times$256) test set to produce the stroke painting inputs for \cref{fig:kid_l2}, \cref{tab:stroke_lsun_synthetic} and \cref{tab:att_cls}. Additionally \cref{fig:auto_stroke_bedroom}, \cref{fig:auto_stroke_church} and \cref{fig:auto_stroke_celeba} show examples of the ground truth images, synthetic stroke paintings, and the corresponding generated images by SDEdit. The simulated stroke paintings resemble the ones drawn by humans and SDEdit is able to generate high quality images based on this synthetic input, while the baselines fail to obtain comparable results.

\paragraph{KID evaluation}
 KID is calculated between the real image from the validation set and the generated images using synthetic stroke paintings (based on the validation set), and the squared $L_2$ distance is calculated between the simulated stroke paintings and the generated images.

\paragraph{Realism-faithfulness trade-off} 
\label{app:sec:real-faith-trade-off}
To search for the sweet spot for realism-faithfulness trade-off as presented in Figure \ref{fig:kid_l2}, we select $0.01$ and every $0.1$ interval from $0.1$ to $1$
for $t_0$ and generate images for the LSUN church outdoor dataset. We apply the human-stroke-simulation algorithm on the original LSUN church outdoor validation set and generate one stroke painting per image to produce the same input stroke paintings for all choices of $t_0$. As shown in Figure \ref{fig:auto_stroke_t}, this algorithm is sufficient to simulate human stroke painting and we can also observe the realism-faithfulness trade-off given the same stroke input. KID is calculated between the real image from the validation set and the generated images, and the squared $L_2$ distance is calculated between the simulated stroke paintings and the generated images.

\subsection{Training and inference time}
We use open source pretrained SDE models provided by \citeauthor{song2021scorebased,ho2020denoising,dhariwal2021diffusion}. 
In general, VP and VE have comparable speeds, and can be slower than encoder-based GAN inversion methods. 
For scribble-based generation on 256$\times$256 images, \model takes 29.1s to generate one image 
 on one 2080Ti GPU.
In comparison, StyleGAN2-ADA~\citep{karras2020training} takes around 72.8s and In-domain GAN 2~\citep{zhu2020domain} takes 5.2s using the same device and setting. We note that our speed is in general faster than optimization-based GAN inversions while slower than encoder-based GAN inversions. The speed of \model{} could be improved by recent works on faster SDE sampling.

\section{Extra experimental results}
\label{app:results}
\subsection{Extra results on LSUN datasets}

\myparagraph{Stroke-based image generation.}
We present more SDEdit (VP) results on LSUN bedroom in \cref{fig:lsun_stroke_generation_app}. We use $t_0=0.5$, $N=500$, and $K=1$. We observe that, SDEdit is able to generate realistic images that share the same structure as the input paintings when \emph{no paired data is provided}.

\myparagraph{Stroke-based image editing.}
We present more SDEdit (VP) results on LSUN bedroom in \cref{fig:lsun_stroke_edit_app}. SDEdit generates image edits that are both realistic and faithful to the user edit, while avoids making undesired modifications on pixels not specified by users. See Appendix~\ref{app:experiment} for experimental settings.

\subsection{Extra results on Face datasets}

\myparagraph{Stroke-based image editing.}
We provide intermediate step visualizations for SDEdit in \cref{fig:stroke_visualization_app}.
We present extra SDEdit results on CelebA-HQ in \cref{fig:celeba_stroke_edit_app}.  We also presents results on CelebA-HQ (1024$\times$1024) in \cref{fig:celeba_1024_scribble_app}.
SDEdit generates images that are both realistic and faithful (to the user edit), while avoids introducing undesired modifications on pixels not specified by users.
We provide experiment settings in Appendix~\ref{app:experiment}.

\myparagraph{Image compositing.}
\label{app:image_compositing}
We focus on editing hair styles and adding glasses.
We present more SDEdit (VE) results on CelebA-HQ (256$\times$256) in \cref{fig:brown_hair_app}, \cref{fig:glasses_app}, and \cref{fig:blond_hair_app}. We also presents results on CelebA-HQ (1024$\times$1024) in \cref{fig:celeba_1024_app}. We observe that SDEdit can generate both faithful and realistic edited images.
See Appendix~\ref{app:experiment} for experiment settings.

\myparagraph{Attribute classification with stroke-based generation.}
\label{app:att_cls}
In order to further evaluate how the models convey user intents with high level user guide, we perform attribute classification on stroke-based generation for human faces. We use the human-stroke-simulation algorithm on a subset of randomly selected 6000 images from CelebA (256$\times$256) test set to create the stroke inputs, and apply Microsoft Azure Face API\footnote{\url{https://github.com/Azure-Samples/cognitive-services-quickstart-code/tree/master/python/Face}} to detect fine-grained face attributes from the generated images. We choose gender and glasses to conduct binary classification, and hair color to perform multi-class classification on the images. Images where no face is detected will be counted as providing false and to the classification problems. \cref{tab:att_cls} shows the classification accuracy, and SDEdit (VP) outperforms all other baselines in all attributes of choice.

\subsection{Class-conditional generation with stroke painting}
\label{app:class_dependent}
In addition to user guide, SDEdit is able to also leverage other auxiliary information and models to obtain further control of the generation. Following \cite{song2021scorebased} and \cite{dhariwal2021diffusion}, we present an \textbf{extra} experiment on class-conditional generation with SDEdit. Given a time-dependent classifier $p_t(\rvy\mid\rvx)$, for SDEdit (VE) one can solve the reverse SDE:
\begin{align}
    \mathrm{d} \rvx(t) = \left[ - \frac{\mathrm{d} [\sigma^2(t)]}{\mathrm{d} t} (\nabla_\rvx \log p_t(\rvx) + \nabla_\rvx \log p_t(\rvy\mid\rvx))\right] \mathrm{d}t + \sqrt{\frac{\mathrm{d} [\sigma^2(t)]}{\mathrm{d} t}} \mathrm{d} \bar{\mathbf{w}} \label{eq:class-dependent-sde-ve}
\end{align}
and use the same sampling procedure defined in Section \ref{sec:method}.

For SDEdit (VP), we follow the class guidance setting in \cite{dhariwal2021diffusion} and solve:
\begin{equation}
\rvx_{n-1} = \frac{1}{\sqrt{1 - \beta(t_n)\Delta t}}(\rvx_n + \beta(t_{n})\Delta t \vs_\vtheta(\rvx(t_n), t_n)) \\
+ \beta(t_n)\Delta t\nabla_\rvx \log p_t(\rvy\mid\rvx_n) \\ + \sqrt{\beta(t_n)\Delta t}~\rvz_n,
\end{equation}

\cref{fig:class_conditional} shows the ImageNet (256$\times$256) class-conditional generation results using SDEdit (VP). Given the same stroke inputs, SDEdit is capable of generating diverse results that are consistent with the input class labels.

\subsection{Extra datasets}
\label{app:cat_horse}
We present additional stroke-based image synthesis results on LSUN cat and horse dataset for SDEdit (VP).
\cref{fig:cat_horse} presents the image generation results based on input stroke paintings with various levels of details. We can observe that SDEdit produce images that are both realistic and faithful to the stroke input on both datasets. Notice that for coarser guide (\eg the third row in \cref{fig:cat_horse}), we choose to slightly sacrifice faithfulness in order to obtain more realistic images by selecting a larger $t_0=0.6$, while all the other images in \cref{fig:cat_horse} are generated with $t_0=0.5$.

\subsection{Extra results on baselines}
\model  preserves  the  un-masked regions automatically, while GANs do not. We tried post-processing samples from GANs by masking out undesired changes, yet the artifacts are strong at the boundaries. We further tried blending on GANs (GAN blending) with StyleGAN2-ADA, but the artifacts are still distinguishable (see \cref{fig:app:gan_blending}).

\begin{figure}
 \vspace{-20pt}
     \begin{subfigure}[h]{0.22\linewidth}
         \centering
    \includegraphics[width=\linewidth]{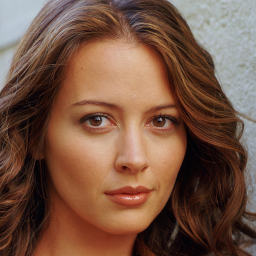}
\caption{Dataset image
}
     \end{subfigure}
      \hfill
    \begin{subfigure}[h]{0.22\linewidth}
         \centering
    \includegraphics[width=\linewidth]{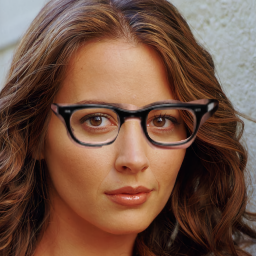}
\caption{User guide
}
     \end{subfigure}
      \hfill
 \begin{subfigure}[h]{0.22\linewidth}
     \centering
     \includegraphics[width=\linewidth]{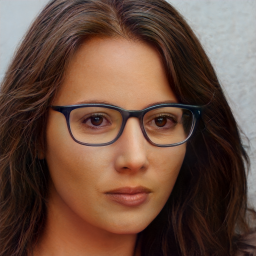}
\caption{GAN output}
     \end{subfigure}
     \hfill
 \begin{subfigure}[h]{0.22\linewidth}
     \centering
     \includegraphics[width=\linewidth]{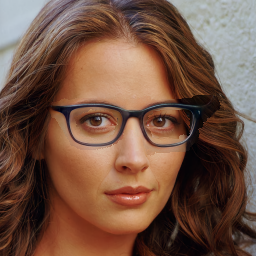}
\caption{
GAN blending
}
\end{subfigure}
\caption{Post-processing samples from GANs by masking out undesired changes, yet the artifacts are strong at the boundaries even with blending.
}
\label{fig:app:gan_blending}
\end{figure}

\begin{table*}[]
 \newcommand{\mround}[1]{\round{#1}{2}}
    \newcommand{\nastar}{\multicolumn{1}{c}{\hspace{15pt}--~*}}
    \newcommand{\na}{\multicolumn{1}{c}{--}}
    \centering \small
    {
        \centering \small
        
        \resizebox{
          \ifdim\width>\textwidth
            \textwidth
          \else
            0.95\width
          \fi
        }{!}{%
\begin{tabular}{@{}lrrrrr@{}}
\toprule
Methods        & \multicolumn{1}{l}{Gender} & \multicolumn{1}{l}{Glasses} & \multicolumn{1}{l}{Hair - Blond} & \multicolumn{1}{l}{Hair - Black} & \multicolumn{1}{l}{Hair - Grey} \\ \midrule
In-domain GAN 1 & 0.5232                     & 0.6355                      & 0.5527                           & 0.5722                           & 0.5398                          \\
In-domain GAN 2 & 0.0202                     & 0.0273                      & 0.1806                           & 0.3158                           & 0.0253                          \\
StyleGAN2-ADA & 0.0127                     & 0.0153                      & 0.1720                           & 0.3105                           & 0.0145                          \\
e4e           & 0.6175                     & 0.6623                      & 0.6731                           & 0.6510                           & 0.7233                          \\
SDEdit (ours) & \textbf{0.8147}            & \textbf{0.9232}             & \textbf{0.8487}                  & \textbf{0.7490}                  & \textbf{0.8928}                 \\ \bottomrule
\end{tabular}}
    }
    \caption{Attribute classification results with simulated stroke  inputs on CelebA. \model (VP) outperforms all baseline methods in all attribute selected in the experiment. Details can be found in Appendix \ref{app:att_cls}.
    }
    \vspace{-10pt}
    \label{tab:att_cls}%
\end{table*}

\section{Human evaluation}
\label{app:sec:mturk}
\subsection{Stroke-based image generation}
Specifically, we synthesize a total of 400 bedroom images from stroke paintings for each method. 
To quantify sample quality, we ask the workers to perform a total of 1500 pairwise comparisons against \model
to determine which image sample looks more realistic. 
Each evaluation HIT contains 15 pairwise comparisons against \model, and we perform 100 such evaluation tasks. The reward per task is kept as 0.2\$. Since each task takes around 1 min, the wage is around 12\$ per hour.
For each question, the workers will be shown two images: one generated image from \model and the other from the baseline model using the same input. The instruction is: ``Which image do you think is \textbf{more realistic}" (see \cref{fig:mturk_realistic1} and \cref{fig:mturk_realistic2}).
\begin{figure}
    \centering
    \includegraphics[width=0.6\textwidth]{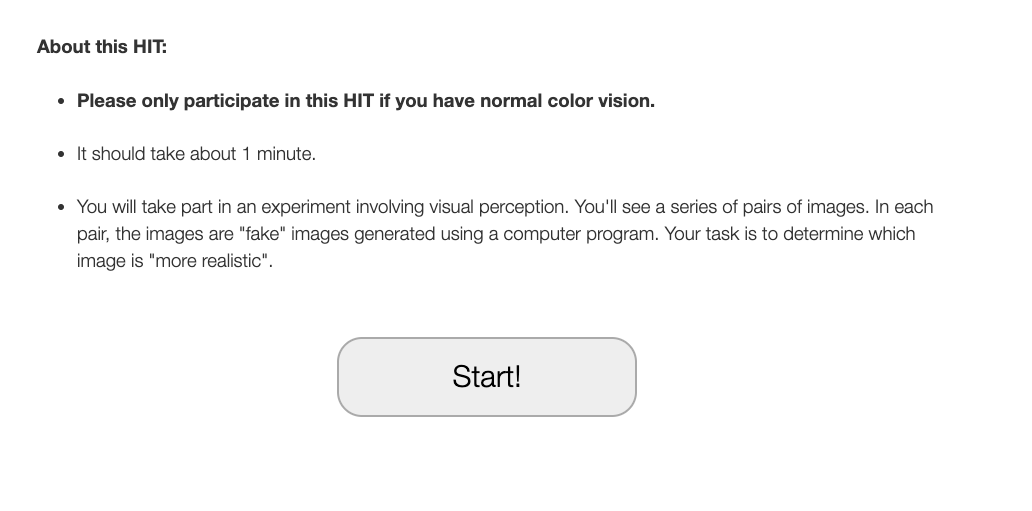}
    \caption{The instruction shown to MTurk workers for pairwise comparison.}
    \label{fig:mturk_realistic1}
\end{figure}

\begin{figure}
    \centering
    \includegraphics[width=0.6\textwidth]{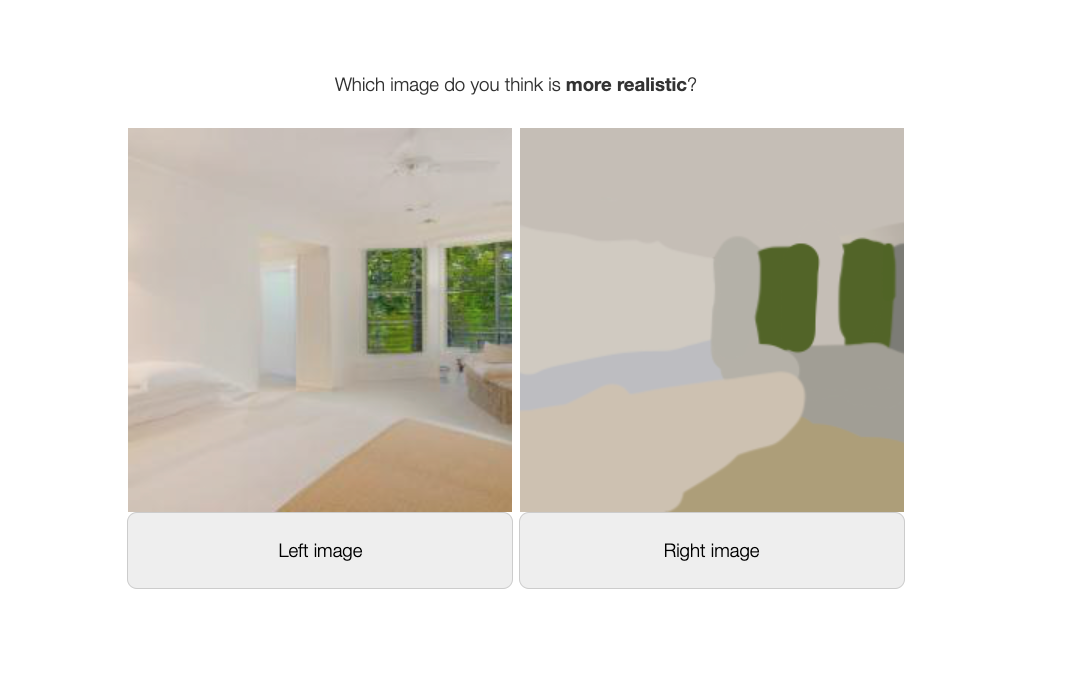}
    \caption{The UI shown to MTurk workers for pairwise comparison.}
    \label{fig:mturk_realistic2}
\end{figure}

\begin{figure}
    \centering
    \includegraphics[width=0.6\textwidth]{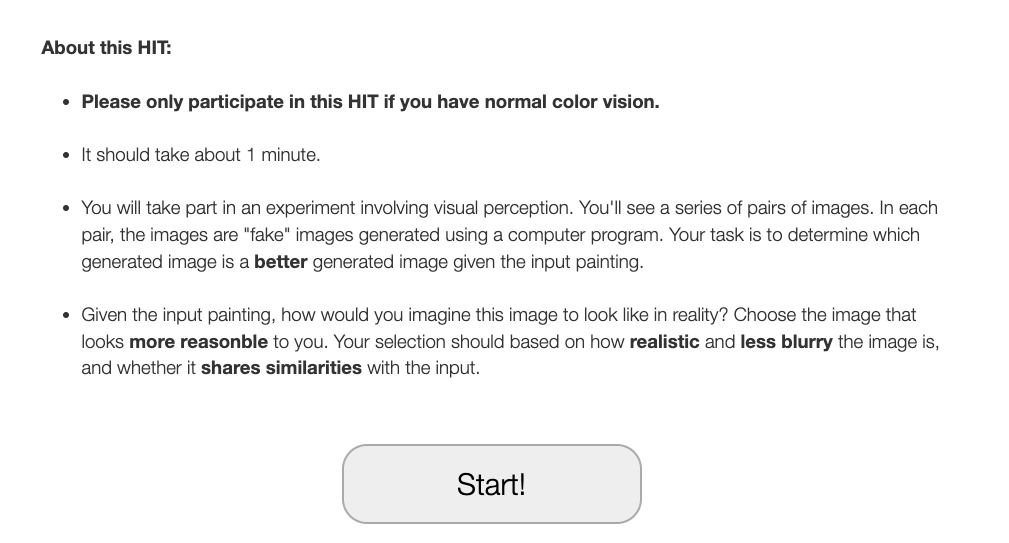}
    \caption{The instruction shown to MTurk workers for pairwise comparison.}
    \label{fig:mturk_reasonable1}
\end{figure}

\begin{figure}
    \centering
    \includegraphics[width=0.6\textwidth]{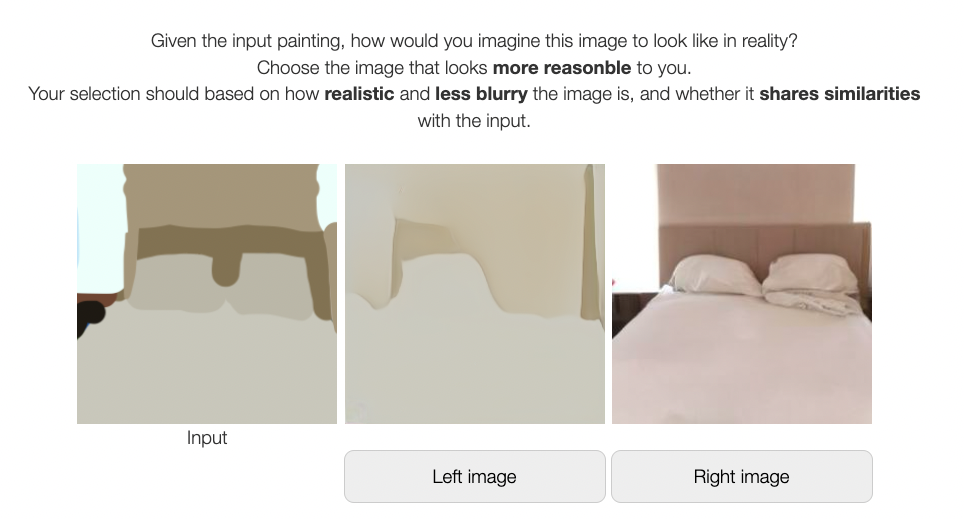}
    \caption{The UI shown to MTurk workers for pairwise comparison.}
    \label{fig:mturk_reasonable2}
\end{figure}

To quantify user satisfactory score (faithfulness+realism),
we ask a different set of workers to perform another 3000 pairwise comparisons against \model. 
For each question, the workers will be shown three images: the input stroke painting (guide), one generated image from \model based on the stroke input, and the other from the baseline model using the same input.
Each evaluation HIT contains 15 pairwise comparisons against \model, and we perform 200 such evaluation tasks. The reward per task is kept as 0.2\$. Since each task takes around 1 min, the wage is around 12\$ per hour.
The instruction is: ``Given the input painting, how would you imagine this image to look like in reality? Choose the image that looks more reasonable to you. Your selection should based on how \textbf{realistic} and \textbf{less blurry} the image is, and whether it \textbf{shares similarities} with the input" (see \cref{fig:mturk_reasonable1} and \cref{fig:mturk_reasonable2}).

\subsection{Image compositing on CelebA-HQ}
To quantitatively evaluate our results, we generate 936 images based on the user inputs. To quantify realism, we ask MTurk workers to perform 1500 pairwise comparisons against \model pre-trained on FFHQ~\citep{karras2019style} to determine which image sample looks more realistic. 
Each evaluation HIT contains 15 pairwise comparisons against \model, and we perform 100 such evaluation tasks. The reward per task is kept as 0.2\$. Since each task takes around 1 min, the wage is around 12\$ per hour.
For each question, the workers will be shown two images: one generated image from \model and the other from the baseline model using the same input. The instruction is: ``Which image do you think was \textbf{more realistic}?".

To quantify user satisfactory score (faithfulness + realism), we ask different workers to perform another 1500 pairwise comparisons against \model pre-trained on FFHQ to decide which generated image matches the content of the inputs more faithfully.
Each evaluation HIT contains 15 pairwise comparisons against \model, and we perform 100 such evaluation tasks. The reward per task is kept as 0.2\$. Since each task takes around 1 min, the wage is around 12\$ per hour.
For each question, the workers will be shown two images: one generated image from \model and the other from the baseline model using the same input. The instruction is: ``Which is a better polished image for the input? An ideal polished image should look \textbf{realistic}, and matches the input in visual appearance (e.g., they look like the same person, with matched hairstyles and similar glasses)".

\FloatBarrier
\newpage
\clearpage

\begin{figure*}
\centering
\includegraphics[width=0.68\linewidth]{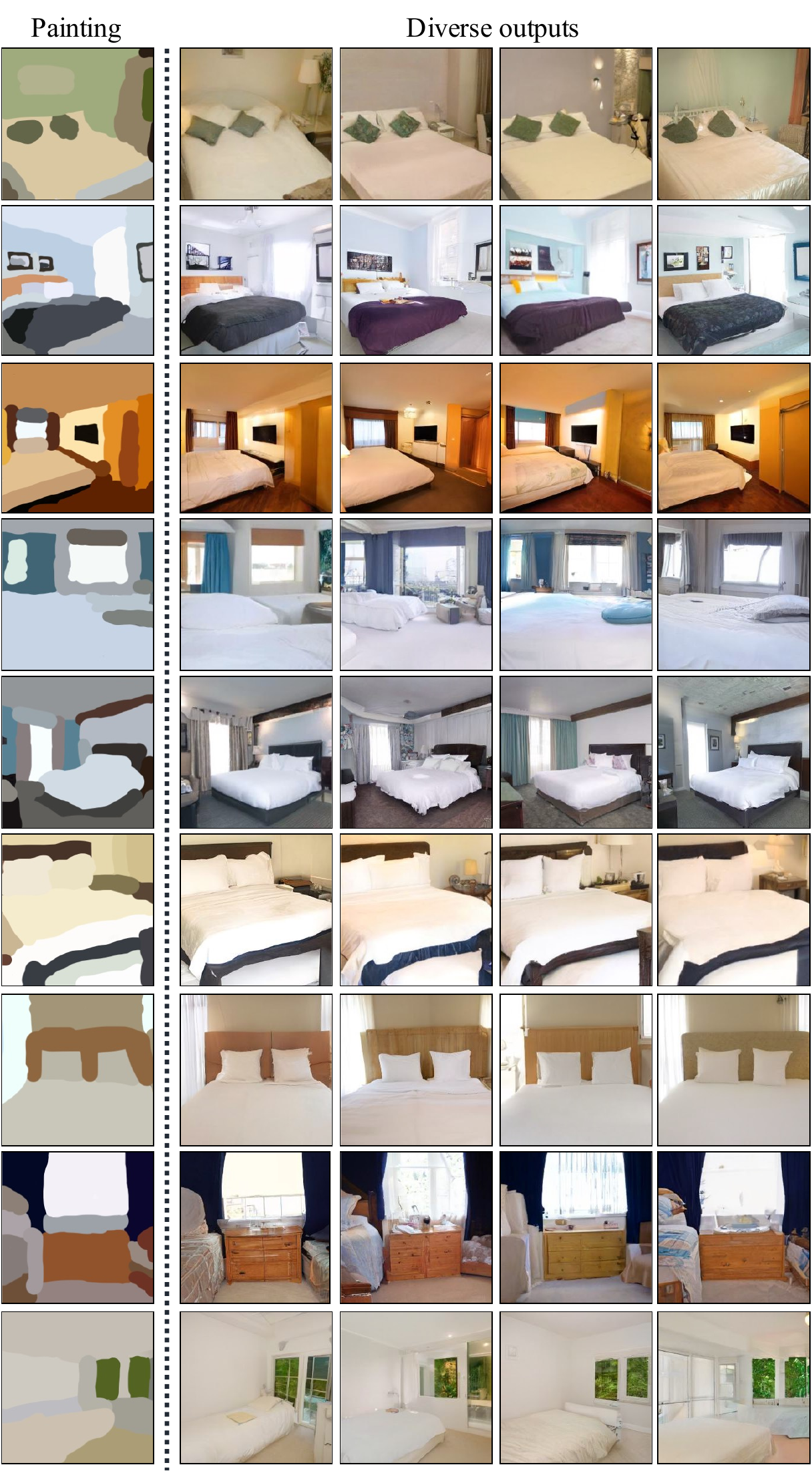}
\caption{Stroke-based image generation on bedroom images with SDEdit (VP) pretrained on LSUN bedroom.}
\label{fig:lsun_stroke_generation_app}
\end{figure*}

\begin{figure*}
\centering
\includegraphics[width=0.9\linewidth]{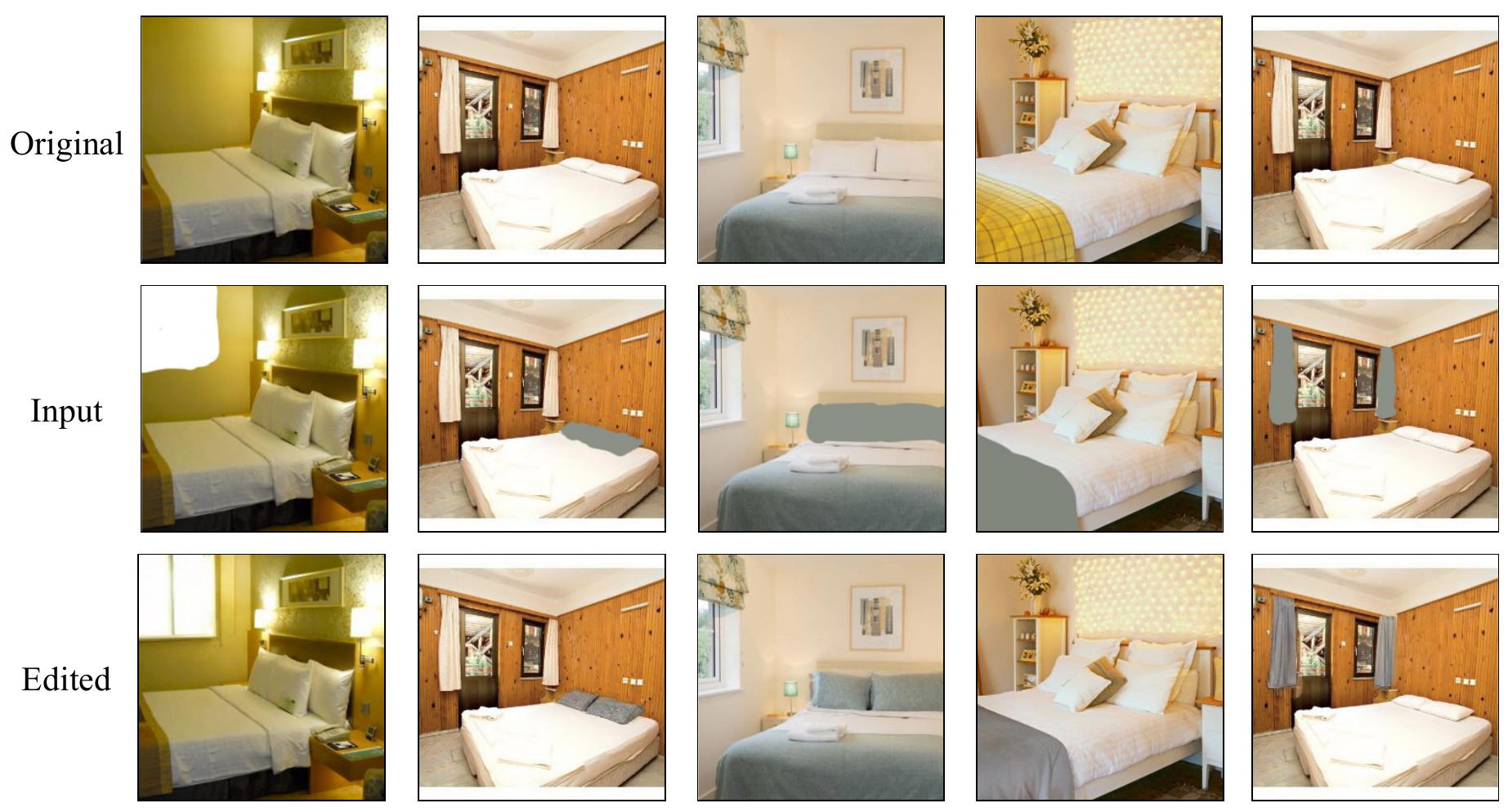}
\caption{Stroke-based image editing on bedroom images with SDEdit (VP) pretrained on LSUN bedroom. SDEdit generates image edits that are both realistic and faithful (to the user edit),  while  avoids  making  undesired  modifications on pixels not specified by users}
\label{fig:lsun_stroke_edit_app}
\end{figure*}

\begin{figure*}
\centering

\includegraphics[width=\linewidth]{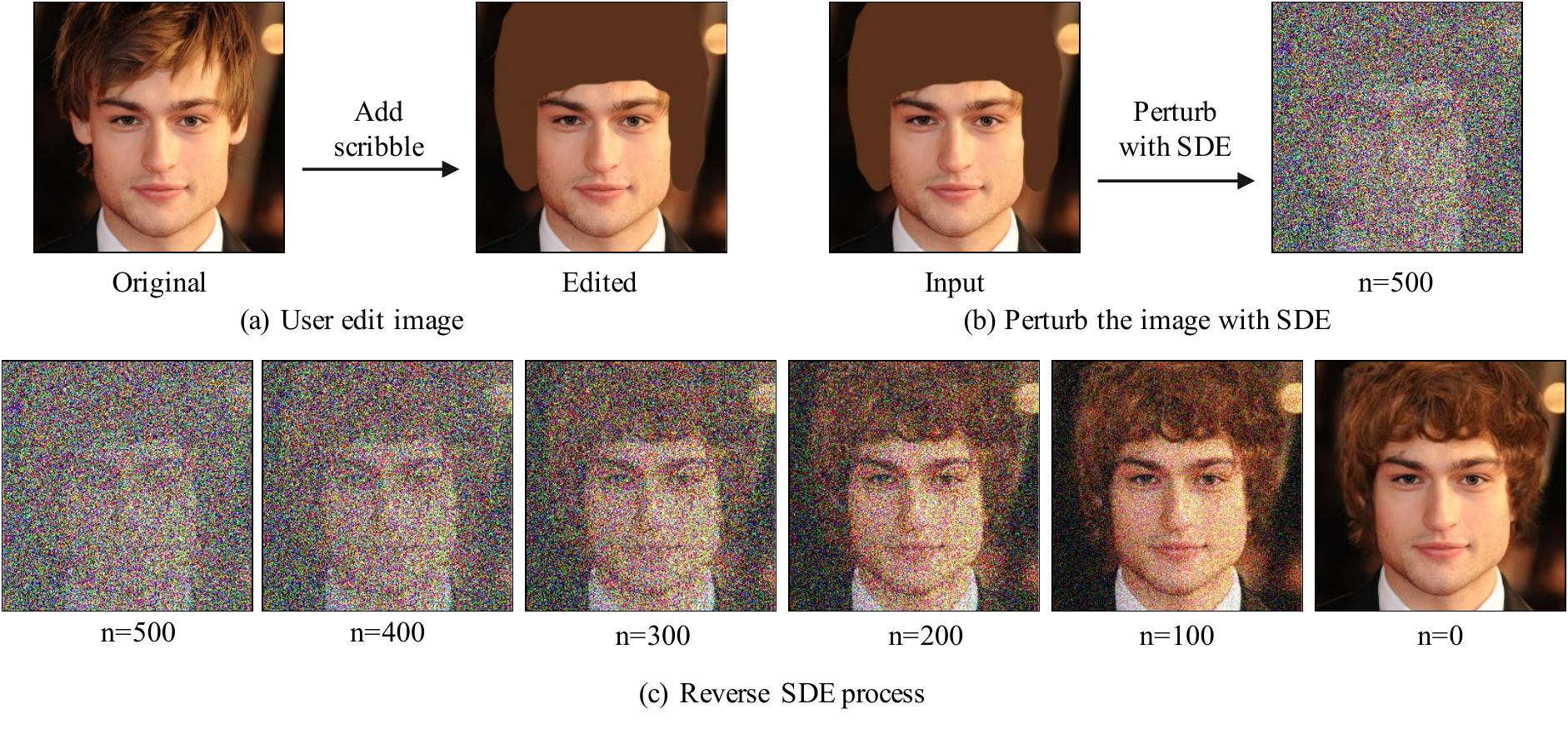}
\caption{Stroke-based image editing. (a) Given an image, users will first modify the image using stroke, and provide a mask which describes the pixels covered by stroke.
(b) The edited image will then be fed into SDEdit. SDEdit will first perturb the image with an SDE,
and then simulate the reverse SDE (see \cref{alg:editing_vp}). (c) We provide visualization of the intermediate steps of reversing SDE used in SDEdit.
}
\label{fig:stroke_visualization_app}
\end{figure*}

\begin{figure*}
\centering
\includegraphics[width=0.86\linewidth]{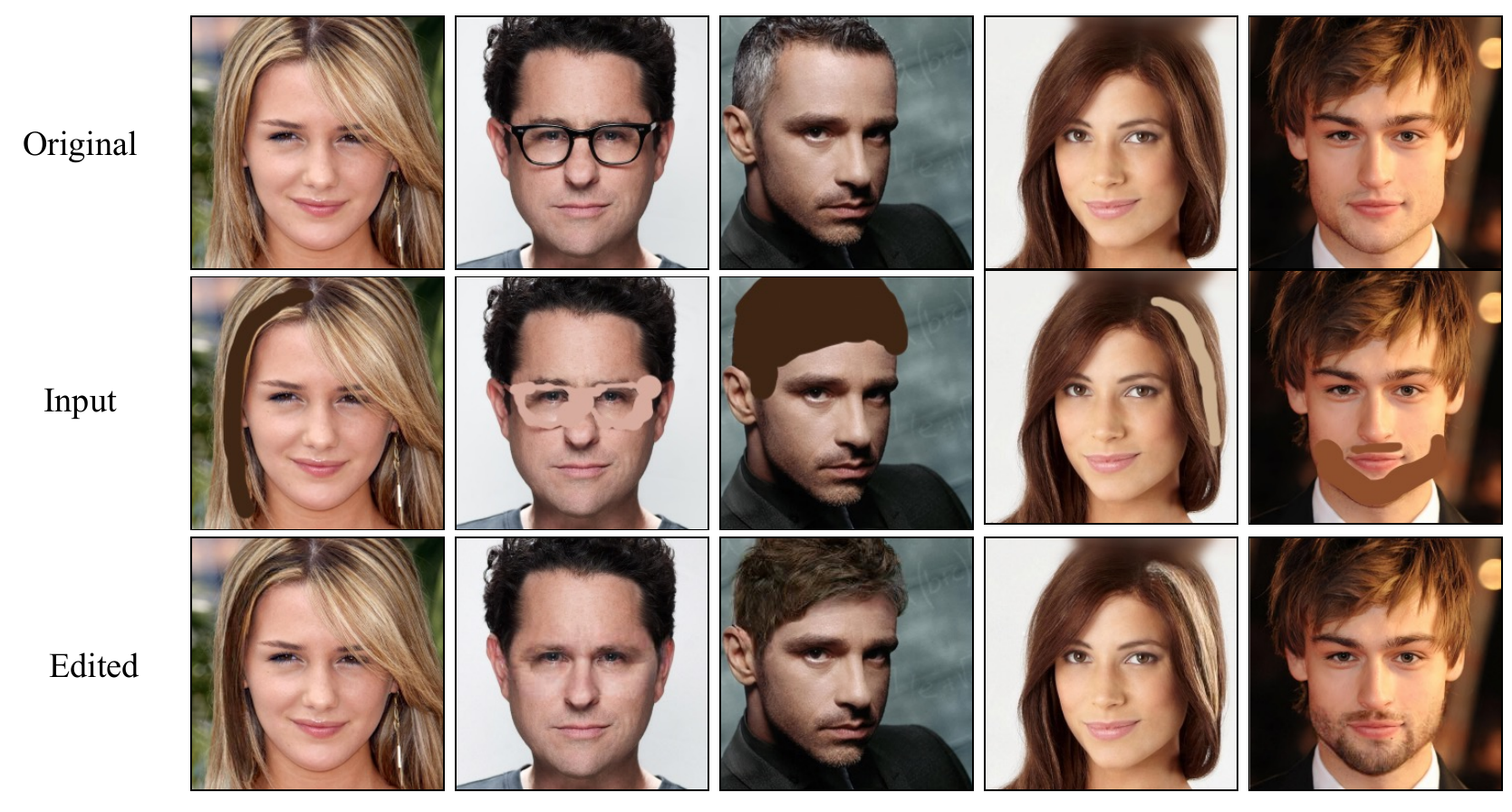}
\caption{Stroke-based image editing on CelebA-HQ images with SDEdit. SDEdit generates image edits that are both realistic and faithful (to the user edit),  while avoids making  undesired  modifications on pixels not specified by users.
}
\label{fig:celeba_stroke_edit_app}
\end{figure*}

\begin{figure*}
\centering
\includegraphics[width=0.8\linewidth]{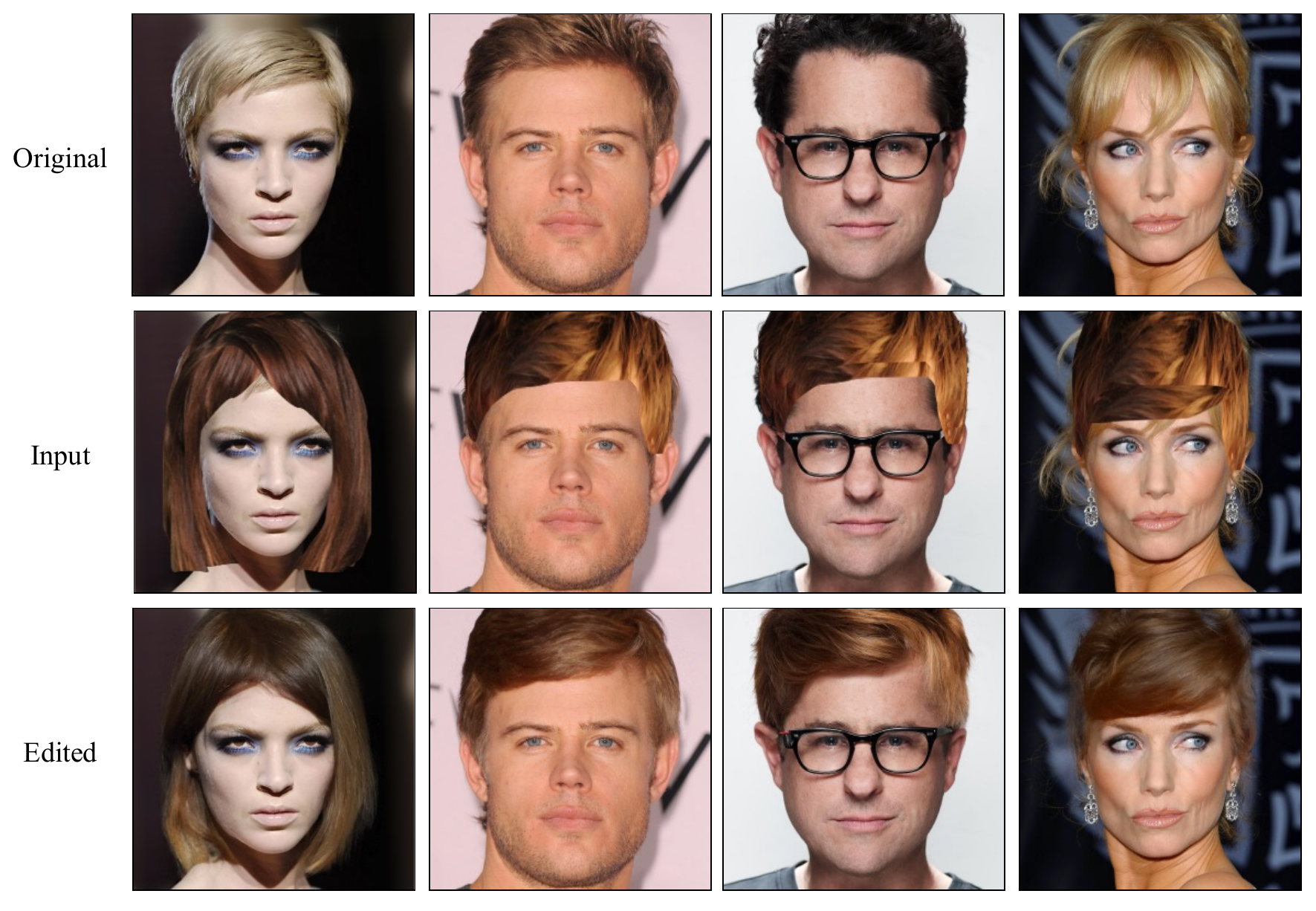}
\caption{Image compositing on CelebA-HQ images with SDEdit. We edit the images to have brown hair. The model is pretrained on FFHQ.}
\label{fig:brown_hair_app}
\end{figure*}

\begin{figure*}
\centering
\includegraphics[width=0.8\linewidth]{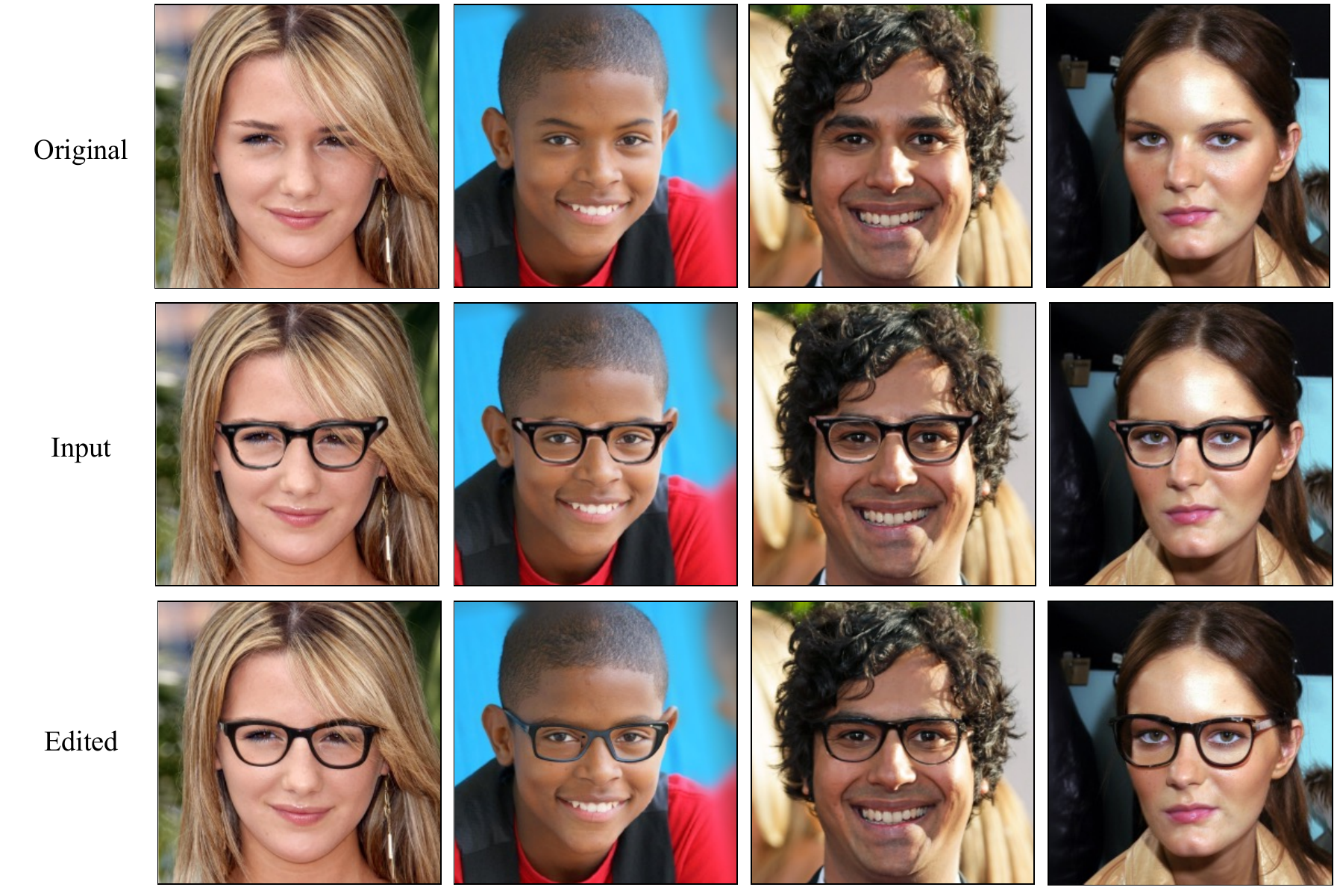}
\caption{Image compositing on CelebA-HQ images with SDEdit. We edit the images to wear glasses. The model is pretrained on FFHQ.}
\label{fig:glasses_app}
\end{figure*}

\begin{figure*}
\centering
\includegraphics[width=0.8\linewidth]{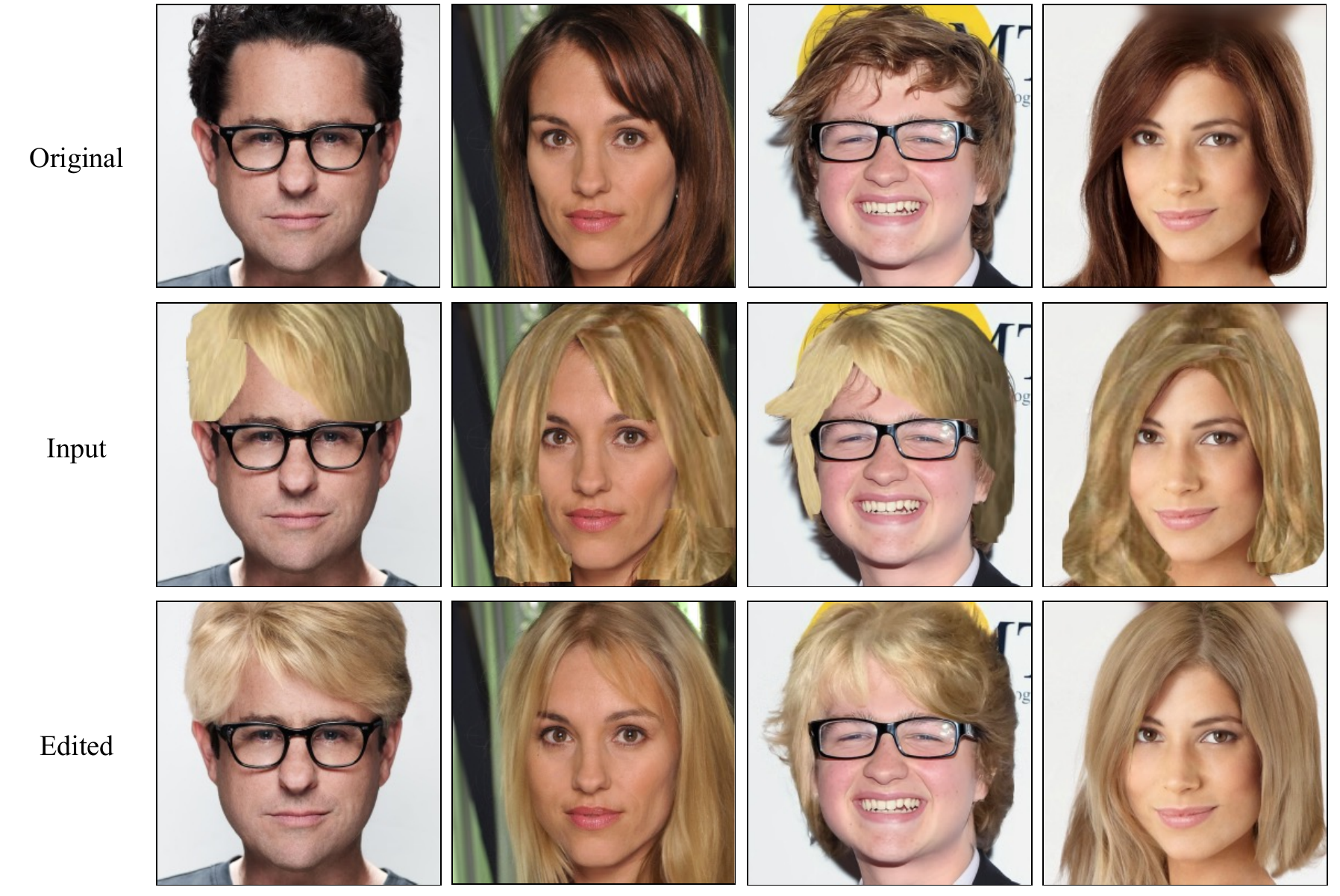}
\caption{Image compositing on CelebA-HQ images with SDEdit. We edit the images to have blond hair. The model is pretrained on FFHQ.}
\label{fig:blond_hair_app}
\end{figure*}

\begin{figure*}
\centering
\begin{subfigure}[b]{0.3\linewidth}
    \centering
    \includegraphics[width=\linewidth]{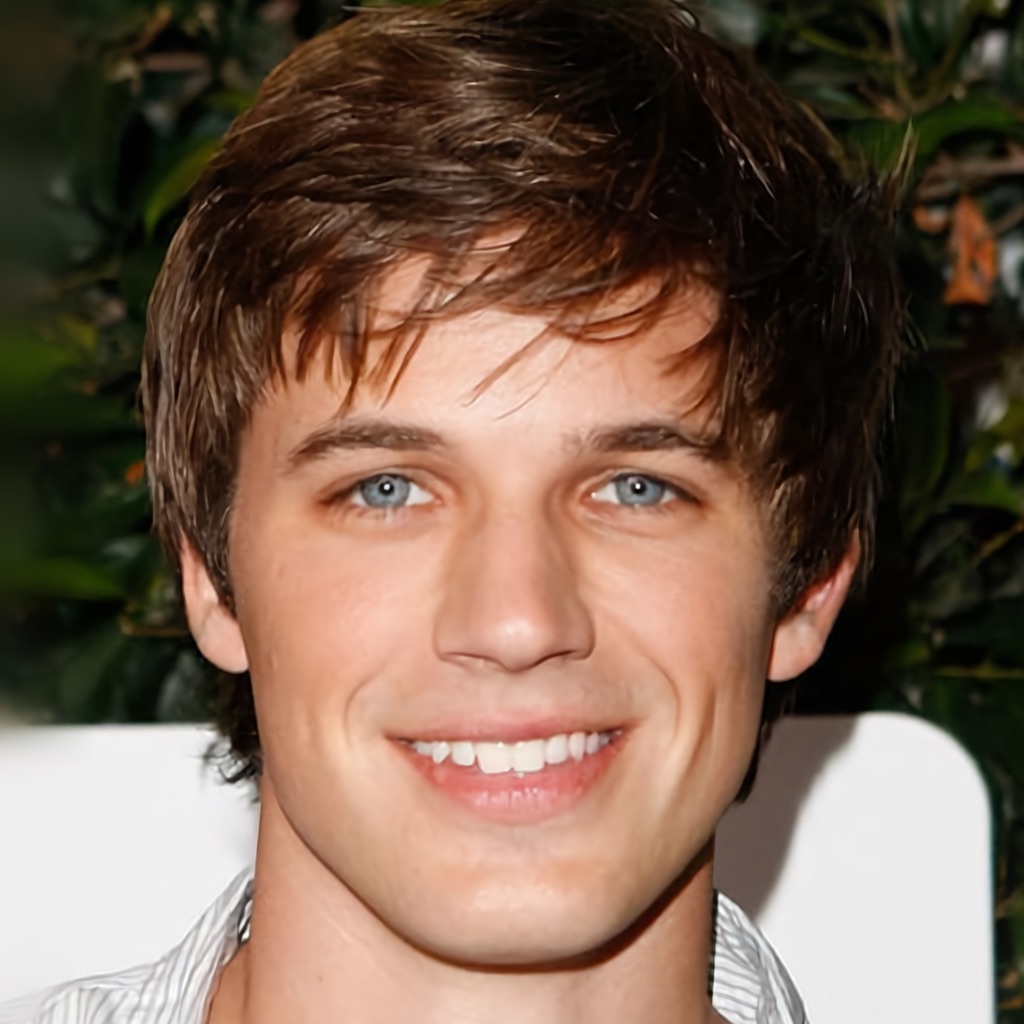}
    \caption{Original dataset image.}
    \label{fig:1024_input}
\end{subfigure}
\begin{subfigure}[b]{0.3\linewidth}
    \centering
    \includegraphics[width=\linewidth]{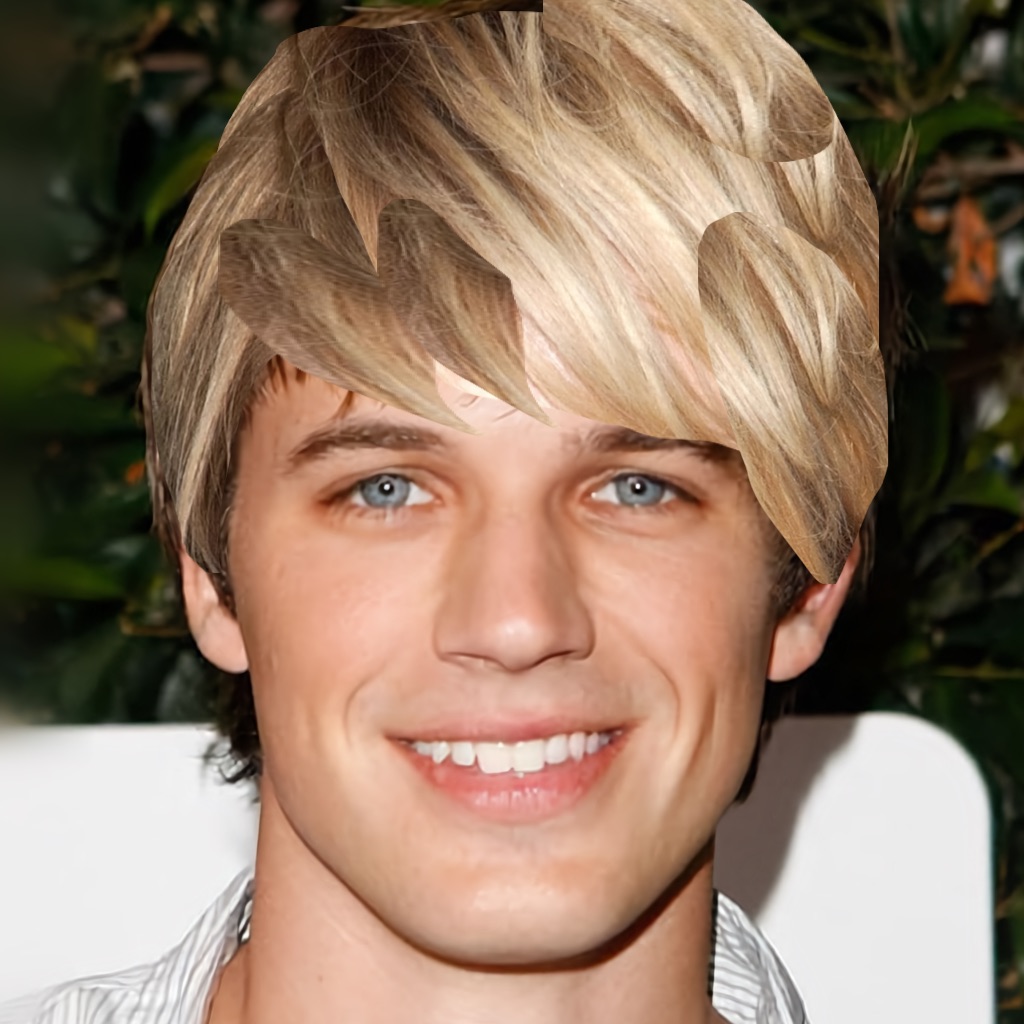}
    \caption{User edited input.}
    \label{fig:1024_colorization}
\end{subfigure}
\begin{subfigure}[b]{0.3\linewidth}
    \centering
    \includegraphics[width=\linewidth]{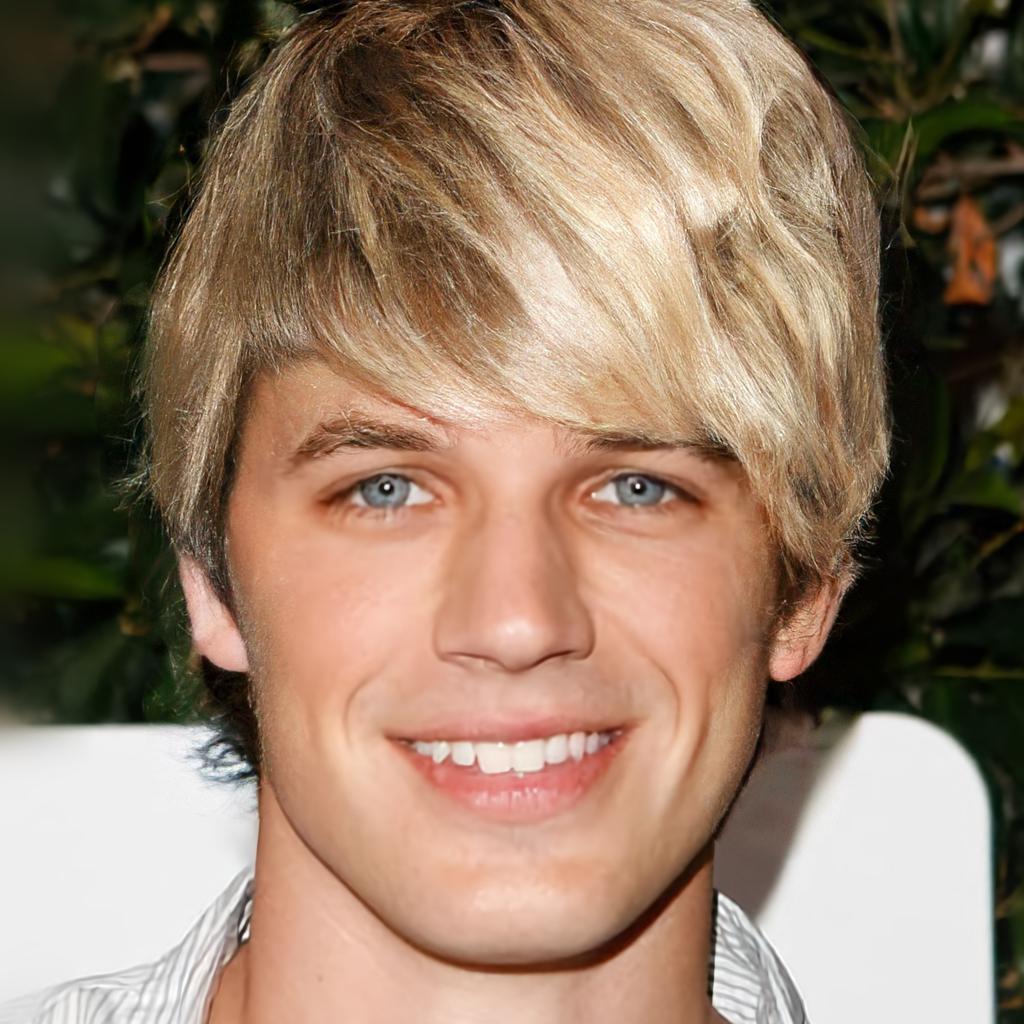}
    \caption{SDEdit results.}
    \label{fig:1024_edited}
\end{subfigure}
\caption{Image compositing results with SDEdit (VE) on CelebA-HQ (resolution 1024$\times$1024). The SDE model is pretrained on FFHQ.}
\label{fig:celeba_1024_app}
\end{figure*}

\begin{figure*}
\centering
\begin{subfigure}[b]{0.3\linewidth}
    \centering
    \includegraphics[width=\linewidth]{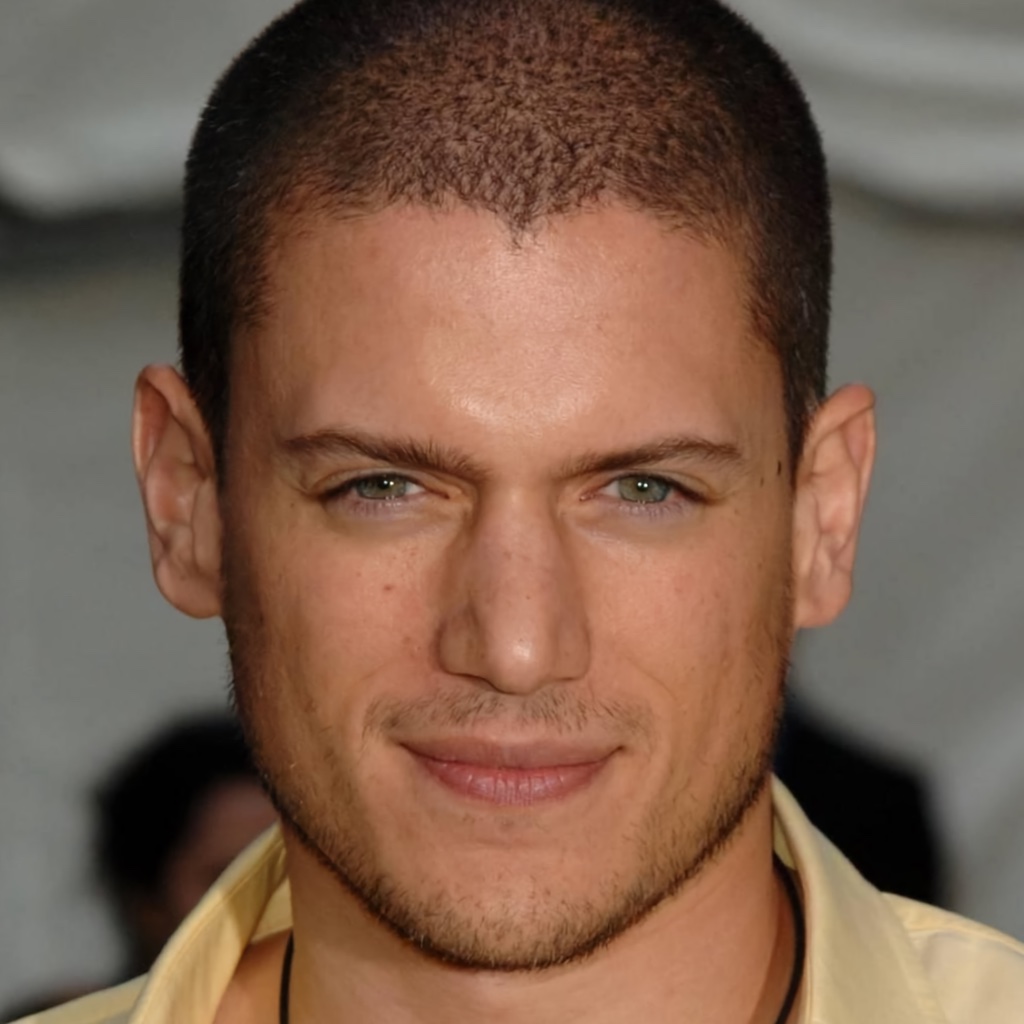}
    \caption{Original dataset image.}
    \label{fig:1024_input_2}
\end{subfigure}
\begin{subfigure}[b]{0.3\linewidth}
    \centering
    \includegraphics[width=\linewidth]{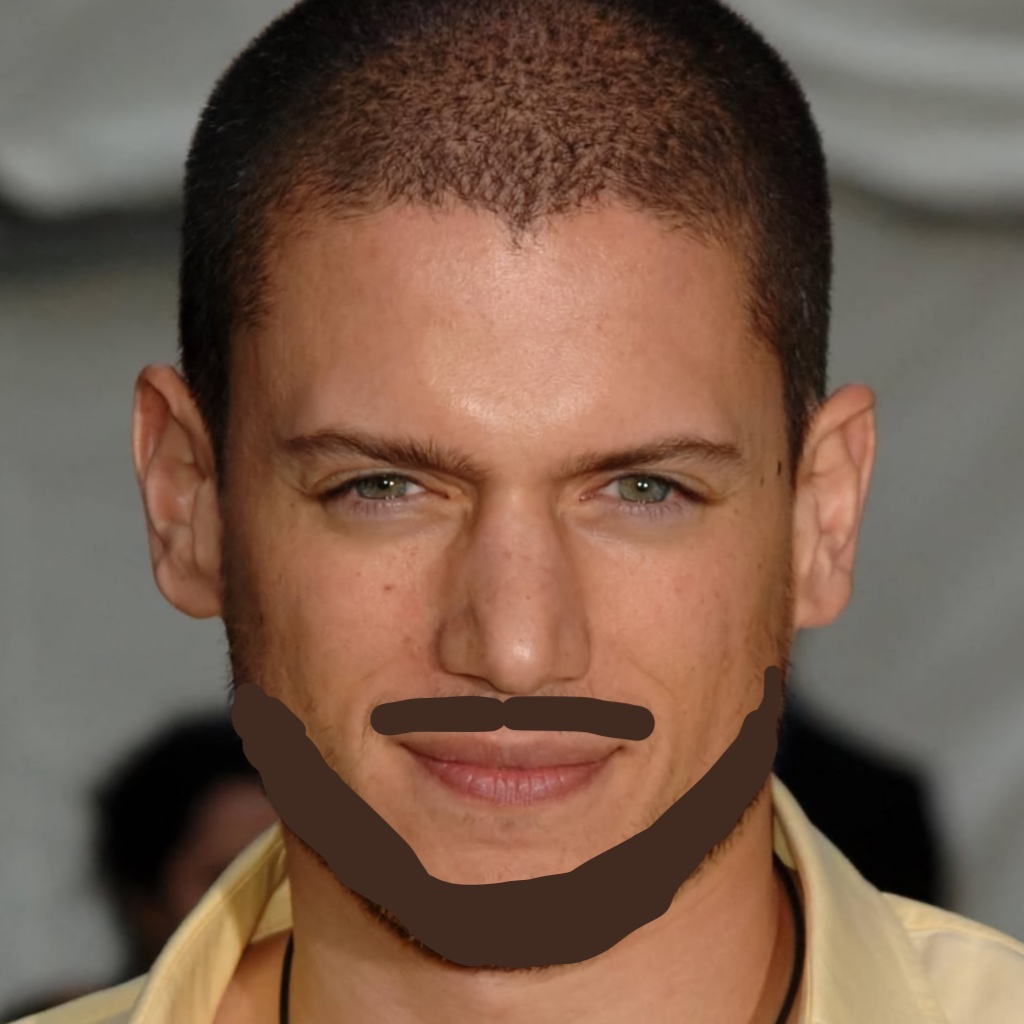}
    \caption{User edited input.}
    \label{fig:1024_colorization_2}
\end{subfigure}
\begin{subfigure}[b]{0.3\linewidth}
    \centering
    \includegraphics[width=\linewidth]{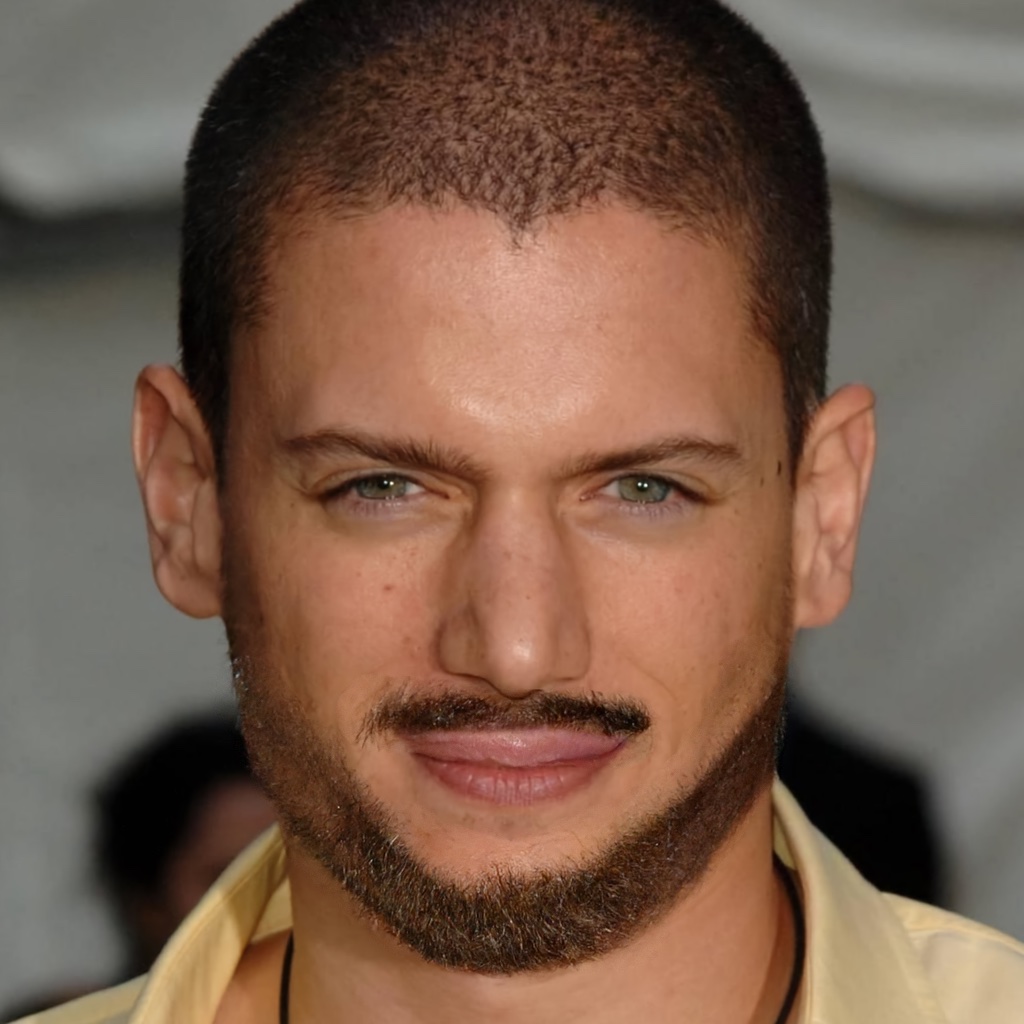}
    \caption{SDEdit results.}
    \label{fig:1024_edited_2}
\end{subfigure}
\caption{Stroke-based image editing results with SDEdit (VE) on CelebA-HQ (resolution 1024$\times$1024). The SDE model is pretrained on FFHQ.}
\label{fig:celeba_1024_scribble_app}
\end{figure*}

\begin{figure*}
\centering
\includegraphics[width=0.9\textwidth]{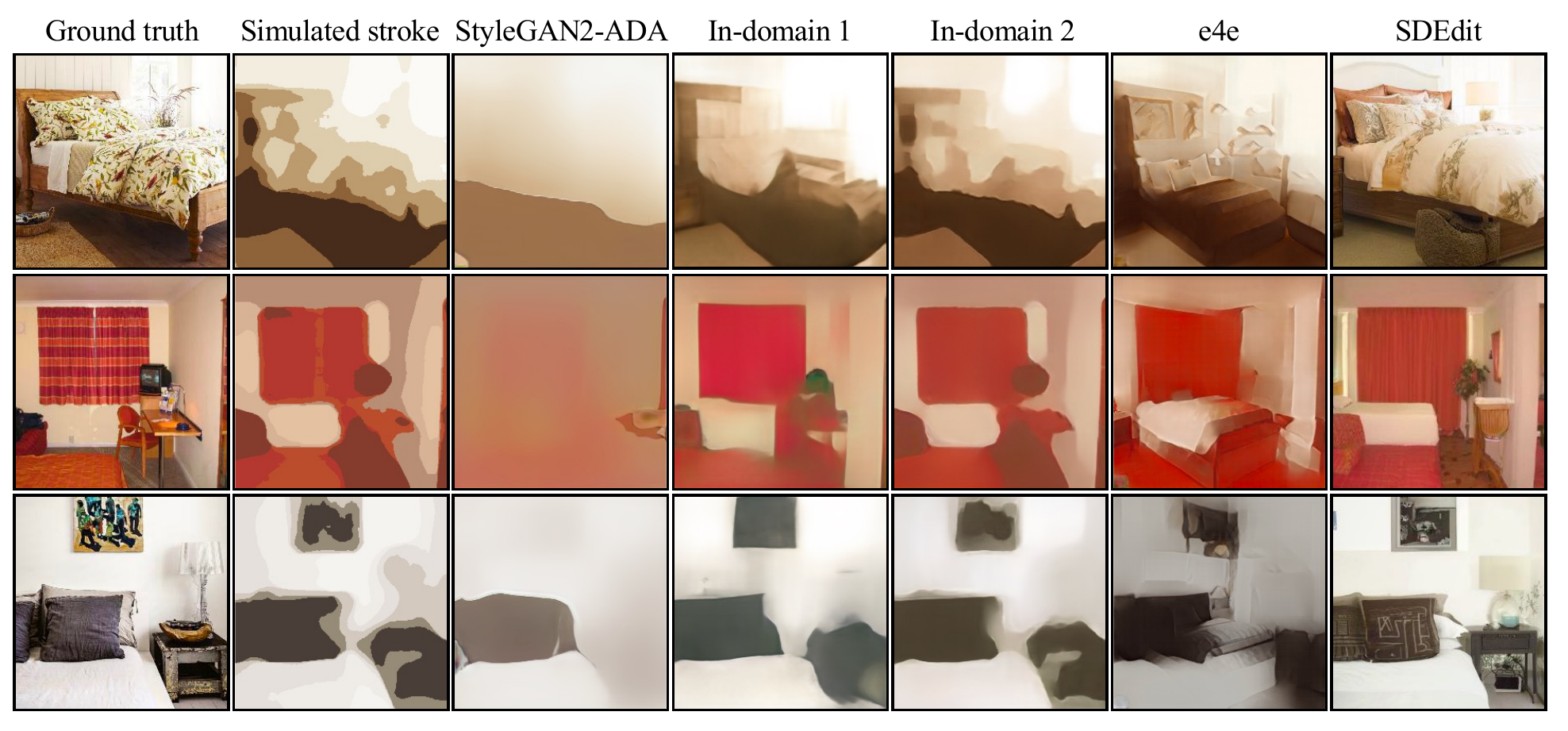}
\caption{Stroke-based image generation with simulated stroke paintings inputs on bedroom images with \model (VP) pretrained on LSUN bedroom dataset.
}
\label{fig:auto_stroke_bedroom}
\end{figure*}

\begin{figure*}
\centering
\includegraphics[width=0.75\textwidth]{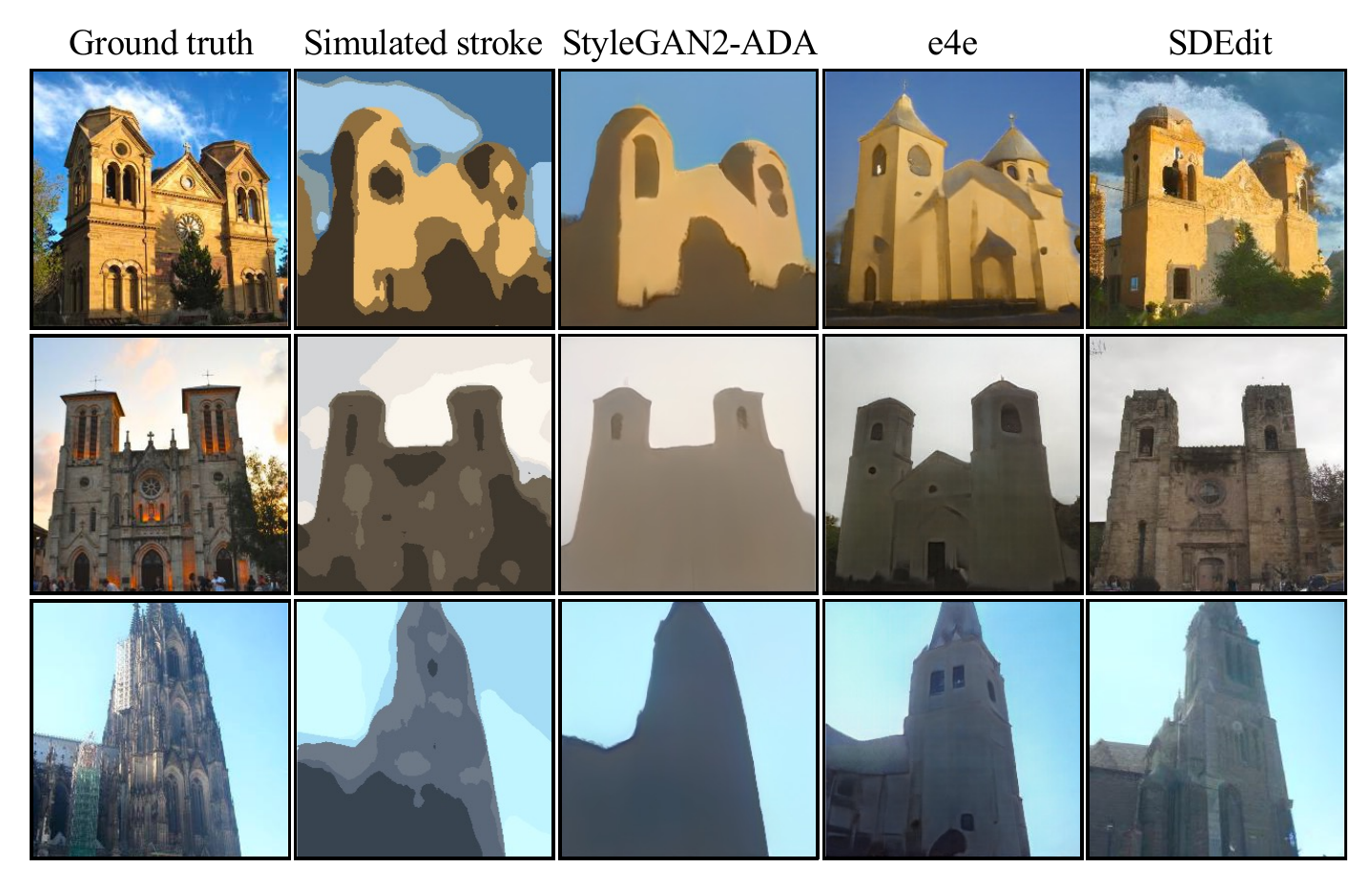}
\caption{Stroke-based image generation with simulated stroke paintings inputs on church images with \model (VP) pretrained on LSUN church outdoor dataset.
}
\label{fig:auto_stroke_church}
\end{figure*}

\begin{figure*}
\centering
\includegraphics[width=0.9\textwidth]{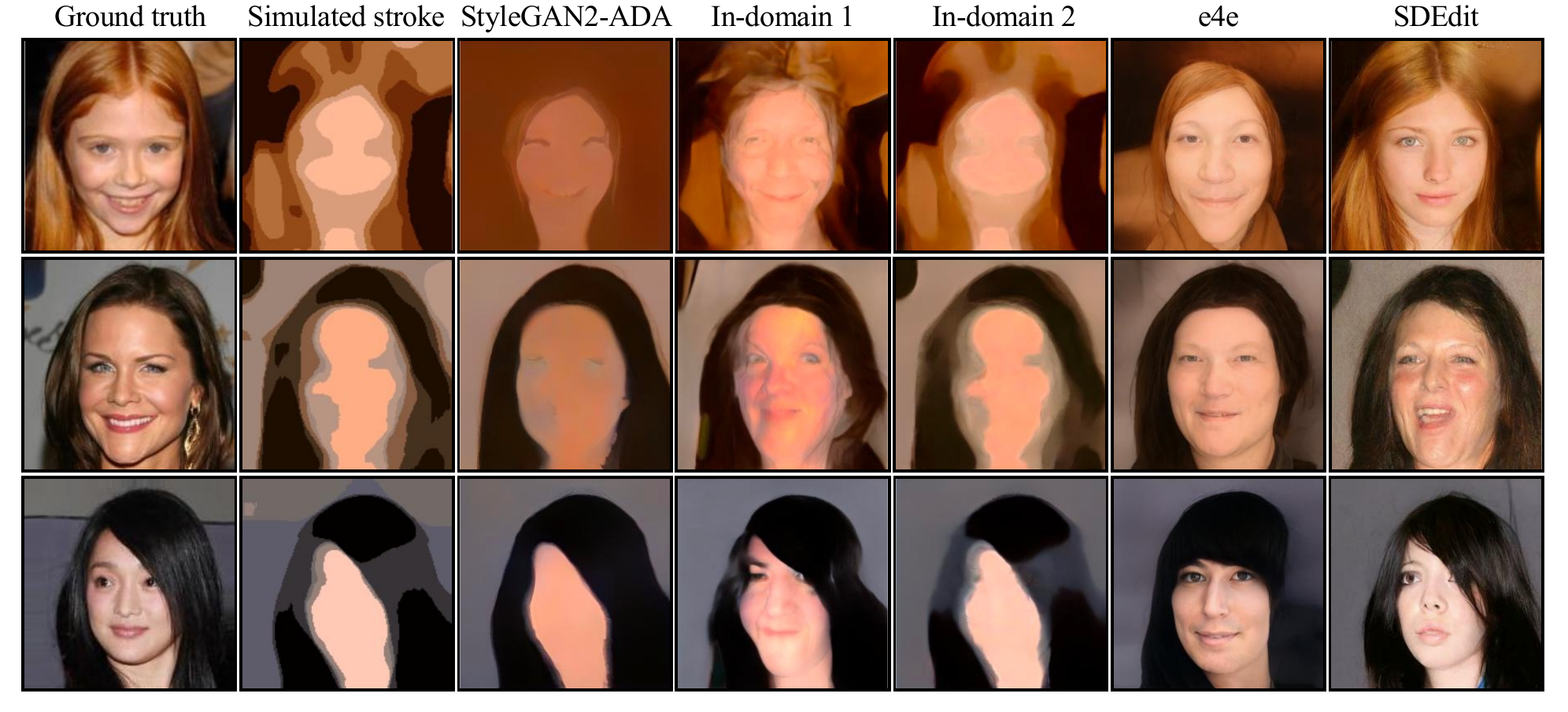}
\caption{Stroke-based image generation with simulated stroke paintings inputs on human face images with \model (VP) pretrained on CelebA dataset.
}
\label{fig:auto_stroke_celeba}
\end{figure*}

\begin{figure*}
\centering
\includegraphics[width=0.9\textwidth]{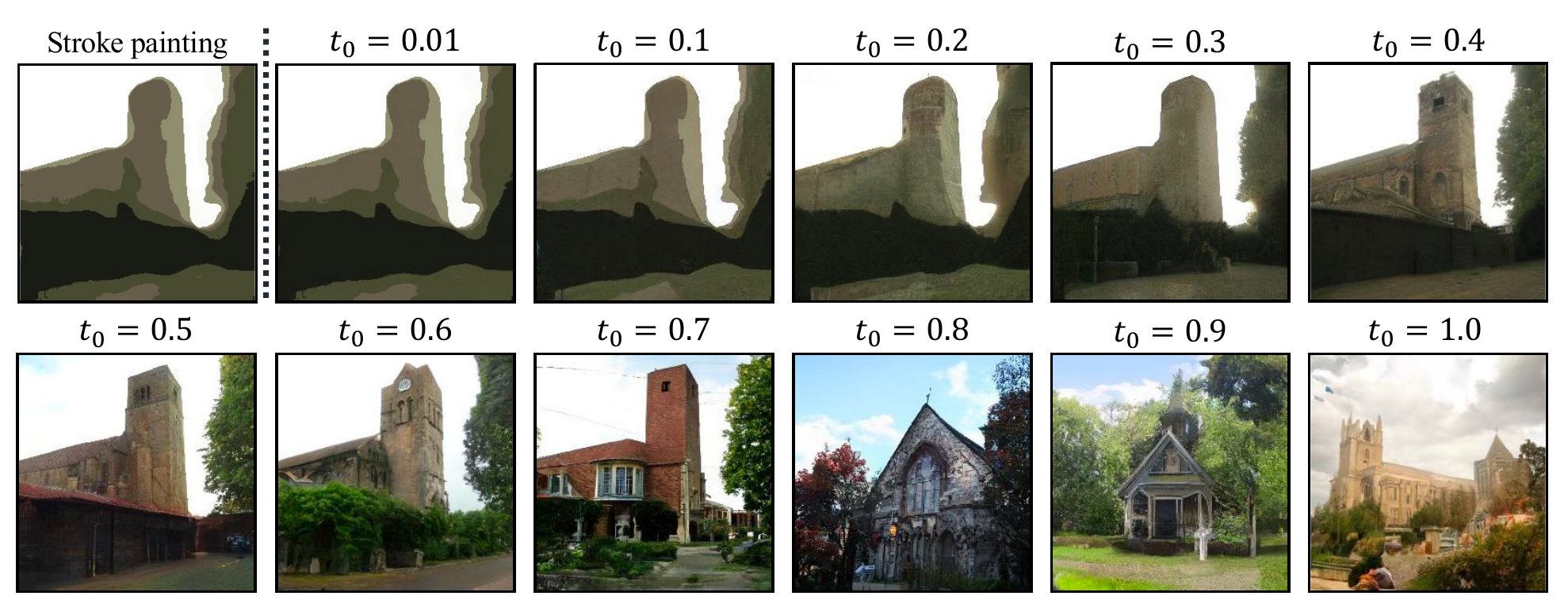}
\caption{Trade-off between faithfulness and realism shown with stroke-based image generation with simulated stroke painting inputs on church images with \model (VP) pretrained on LSUN church outdoor dataset.
}
\label{fig:auto_stroke_t}
\end{figure*}

\begin{figure*}
\centering
\includegraphics[width=\textwidth]{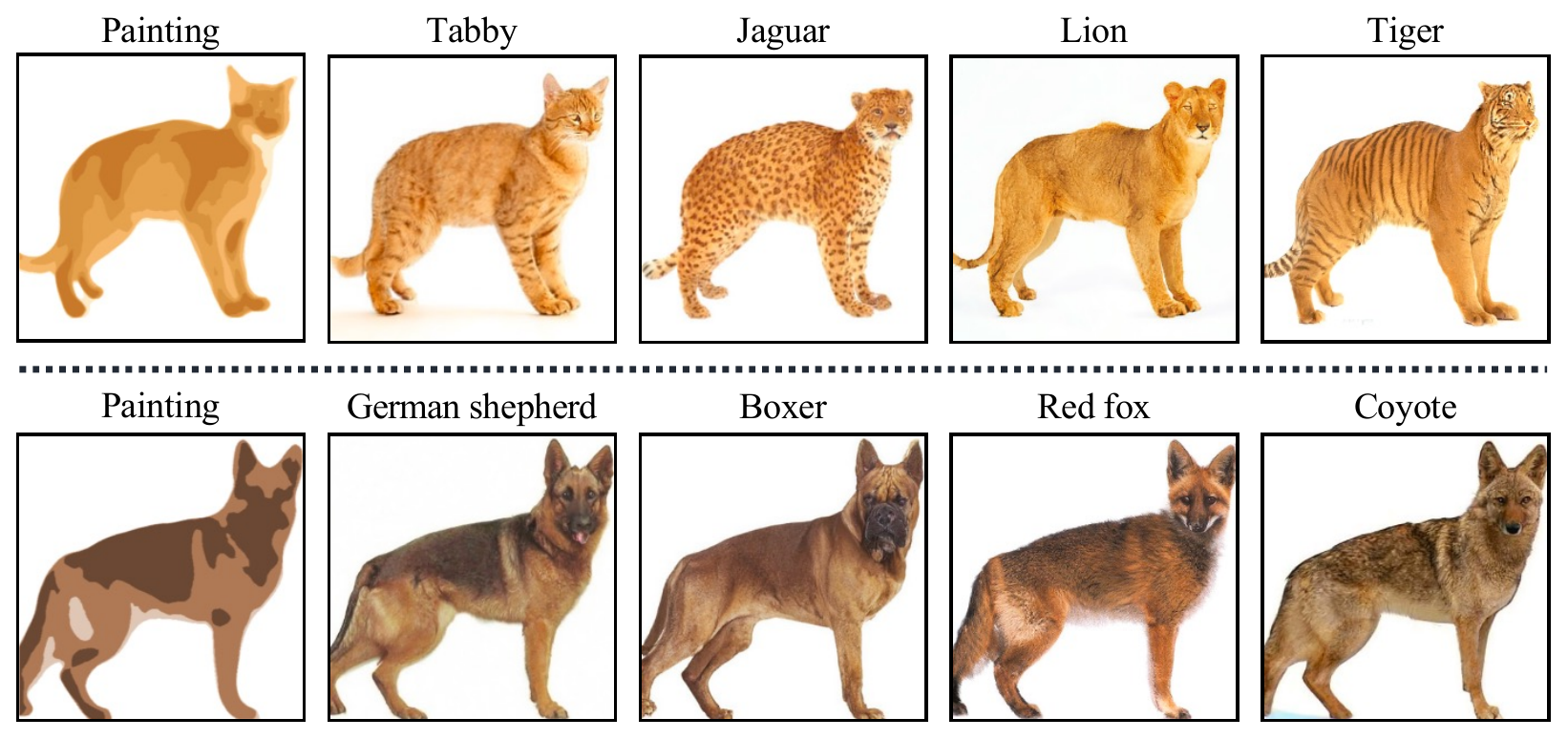}
\caption{Class-conditional image generation from stroke paintings with different class labels by \model (VP) pretrained on ImageNet.
}
\label{fig:class_conditional}
\end{figure*}

\begin{figure*}
\centering
\includegraphics[width=0.75\textwidth]{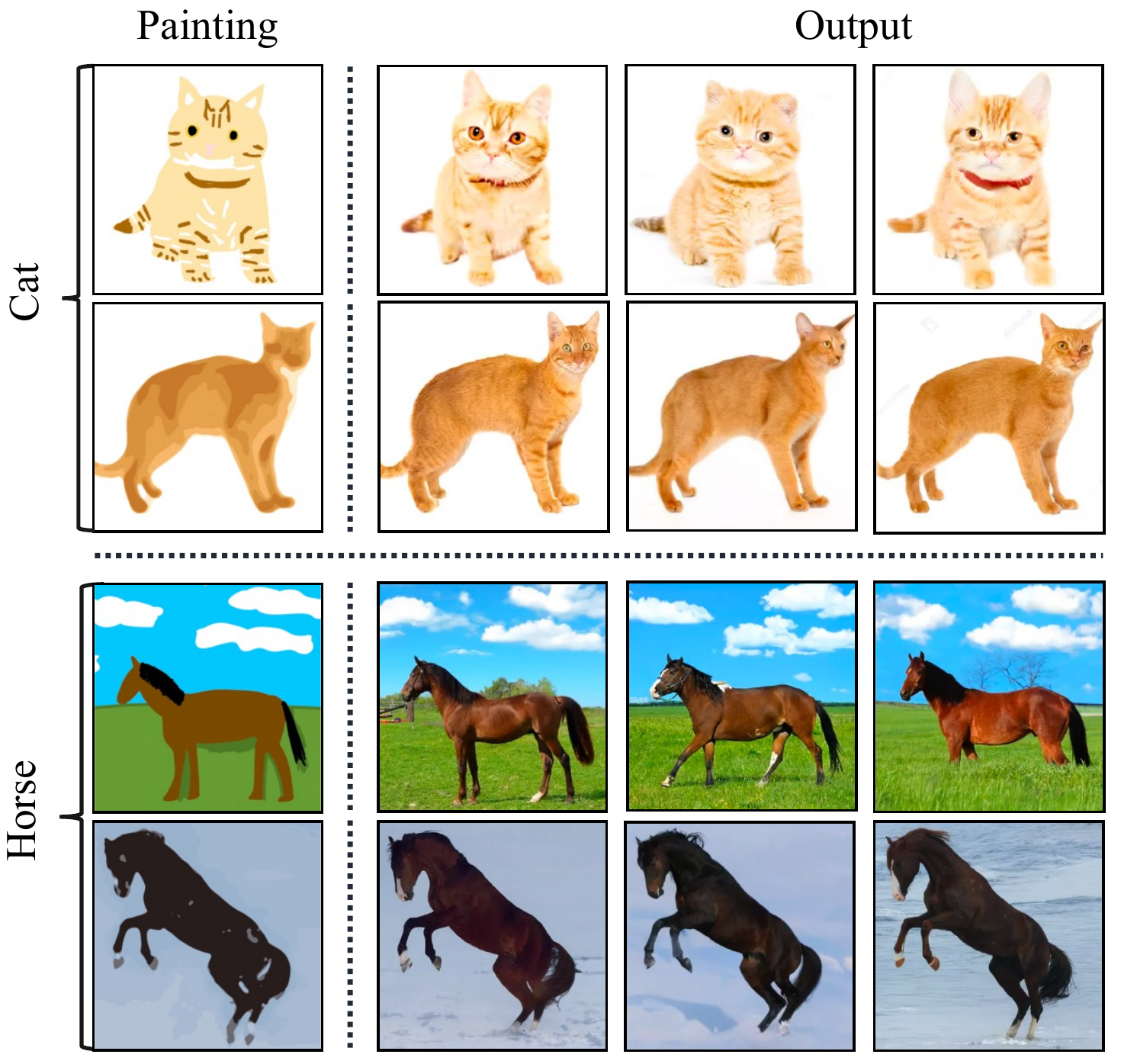}
\caption{Stroke-based image generation with stroke inputs on cat and horse images with \model (VP) pretrained on LSUN cat and horse dataset. Notice that for coarser guide (\eg the third row), we choose to slightly sacrifice faithfulness in order to obtain more realistic images by selecting a larger $t_0=0.6$, while all the other images are generated with $t_0=0.5$.
}
\label{fig:cat_horse}
\end{figure*}

\end{document}